\documentclass{article}

\usepackage{arxiv}
\usepackage{natbib}
\usepackage[utf8]{inputenc} 
\usepackage[utf8]{inputenc} 
\usepackage[T1]{fontenc}    
\usepackage{hyperref}       
\usepackage{url}            
\usepackage{booktabs}       
\usepackage{amsfonts}       
\usepackage{nicefrac}       
\usepackage{microtype}      
\usepackage{xcolor}         
\usepackage{microtype}      
\usepackage{graphicx}
\usepackage{subcaption}
\usepackage{bbm}
\usepackage{amsmath}
\numberwithin{equation}{section}
\usepackage{amssymb}
\usepackage{mathtools}
\usepackage{amsthm}
\usepackage{enumitem}
\usepackage{blindtext}
\usepackage{algorithm}
\usepackage{algpseudocode}

\newcommand{\mc}[1]{\mathcal{#1}}
\newcommand{\mbb}[1]{\mathbb{#1}}
\newcommand{\floor}[1]{\lfloor #1 \rfloor}
\newcommand{\ceil}[1]{\lceil #1 \rceil}
\newcommand{\norm}[1]{\lVert #1 \rVert}
\newcommand{\bignorm}[1]{\left\lVert #1 \right\rVert}

\newcommand{\abs}[1]{\left| #1 \right|}
\definecolor{leoncolour}{rgb}{0.4, 0.0, 0.8}
\definecolor{moritzcolor}{rgb}{0.2,0.6,0.6}

\title{Beyond the Independence Assumption: Finite-Sample Guarantees for Deep Q-Learning under \(\tau\)-Mixing}

\author{%
  Leon Halgryn\\
  Department of Applied Mathematics\\
  University of Twente\\
  \texttt{leon.halgryn@utwente.nl}\\
  \And
  Sophie Langer\\
  Faculty of Mathematics\\
  Ruhr-Universit\"{a}t Bochum\\
  \texttt{s.langer@rub.de}\\
  \AND
  Janusz M. Meylahn\\
  Department of Applied Mathematics\\
  University of Twente\\
  \texttt{j.m.meylahn@utwente.nl}\\
  \And
  E. Moritz Hahn\\
  Department of Applied Mathematics\\
  University of Twente\\
  \texttt{e.m.hahn@utwente.nl}\\
}

\theoremstyle{plain}
\newtheorem{theorem}{Theorem}[section]
\newtheorem{proposition}[theorem]{Proposition}
\newtheorem{lemma}[theorem]{Lemma}
\newtheorem{corollary}[theorem]{Corollary}
\theoremstyle{definition}
\newtheorem{definition}[theorem]{Definition}
\newtheorem{assumption}[theorem]{Assumption}
\theoremstyle{remark}
\newtheorem{remark}[theorem]{Remark}

\begin{document}
\maketitle

\begin{abstract}
Finite-sample analyses of deep Q-learning typically treat replayed data as independent, even though it is sampled from temporally dependent state-action trajectories. We study the Deep Q-networks (DQN) algorithm under explicit dependence by modelling the minibatches used for updating the network as \(\tau\)-mixing. We show that this assumption holds under certain dependence conditions on the underlying trajectories and the mechanism used to sample minibatches. Building on this observation, we extend statistical analyses of DQN with fully connected ReLU architectures to dependent data. We formulate each update as a nonparametric regression problem with $\tau$-mixing observations and derive finite-sample risk bounds under this dependence structure. Our results show that temporal dependence leads to a degradation in the statistical rate by inducing an additional dimensionality penalty in the rate exponent, reflecting the reduced effective sample size of $\tau$-mixing data. Moreover, we derive the sample complexity of DQN under \(\tau\)-mixing from these risk bounds. Finally, we empirically demonstrate on standard Gymnasium environments that the independence assumption is systematically violated and that replay sampling yields approximately exponentially decaying correlations, supporting our theoretical framework.
\end{abstract}

\section{Introduction}
\label{sec:intro}

Reinforcement learning (RL) addresses sequential decision-making problems in which data are generated through interaction with an environment. In value-based methods such as deep Q-learning, this amounts to learning the action-value (Q-)function from data using function approximation. Modern approaches rely on deep neural networks, enabling applications in high-dimensional settings. This paradigm has achieved remarkable empirical success, from Atari games with Deep Q-Networks (DQN) \citep{dqn_og} to large-scale game-playing systems such as AlphaZero \citep{silver_alphazero}, AlphaStar \citep{alphastar}, and MuZero \citep{muzero}. 

From the theoretical perspective, RL is well understood in the tabular setting, where sharp finite-sample guarantees are available (\citep{azar2017minimax, jin2018qlearning, tsitsiklis_q_learning, jaakkola_q_learning}), and extends to structured function approximation such as linear models (\citep{jin2020linear}). However, analyzing nonlinear function approximation, especially deep neural networks, remains challenging. A common approach in recent analyses of deep Q-learning (e.g., \citet{fan_dqn}) is to view each update as a supervised learning problem, where the Q-function is estimated from sample transitions. This reduction enables the use of classical empirical process techniques, but typically relies on the assumption that the training data are independent and identically distributed (i.i.d.). 

In practice, however, data are generated sequentially through agent-environment interaction and stored in a replay buffer, inducing temporal dependence. As a result, the i.i.d. assumption abstracts away the sequential nature of data collection in RL and creates a mismatch between theoretical guarantees and practical implications. This raises the question: 

\begin{center}
\textbf{
Do the same theoretical guarantees hold under temporal dependence?}
\end{center}
To address this, we model the state-action trajectory in the replay buffer as weakly dependent (either \(\beta\)- or \(\tau\)-mixing) and show that, under appropriate sampling schemes, the induced minibatches are \(\tau\)-mixing.
We establish statistical convergence rates for fitted Q-iteration with deep neural networks under temporally dependent data, showing that $\tau$-mixing dependence degrades the rate through a reduced effective sample size. Building on \citet{fan_dqn}, we extend their supervised learning framework to this dependent setting and quantify the resulting loss in statistical efficiency.
We focus on fully connected feedforward neural networks, which include sparse architectures as a special case. We consider Bellman operators for which, for the relevant class of Q-functions, the backup $\mathcal{T}Q$ admits a compositional or low-intrinsic-dimensional representation, or a combination of both. The low-dimensional structure may occur either at the input level, through dependence on a small subset of state-action coordinates, or internally, through low-dimensional intermediate representations in the composition. This is motivated by the structure of many RL problems: rewards often depend on a small number of task-relevant variables, transition dynamics frequently factor into local or modular mechanisms, and the reachable state-action space may concentrate near a low-dimensional latent manifold. Since the Bellman backup combines the reward, transition kernel, and continuation value, such environmental structure can be inherited by \(\mathcal{T}Q\). 
Empirically, we observe across several standard Gymnasium environments \citep{gymnasium} that RL trajectories and replay buffer data violate the independence assumption, while sampled minibatches exhibit approximately exponential decay of correlations, supporting our theoretical framework.
To summarize, our contributions are as follows:

\begin{itemize}
\item We extend finite-sample guarantees for DQN to dependent data and fully connected ReLU networks under structured Bellman operators that are compositional, low-dimensional, or both. We show that temporal dependence degrades statistical rates via an additional dimensionality penalty in the rate exponent, and derive the corresponding sample complexity under \(\tau\)-mixing.
\item We show that if the underlying state-action process is exponentially \(\beta\)- or \(\tau\)-mixing, then the sampled minibatches are exponentially \(\tau\)-mixing provided an order-preserving, duplicate-free sampler is used. 
Moreover, we show why the standard DQN sampling scheme (uniform with replacement) is problematic and introduce a contiguous block sampler that provides some control over the rate of decay of correlations in the minibatches.
\item We design an estimator for analyzing weak dependence and use this in experiments to (i) show that the i.i.d. assumption is violated for standard Gymnasium environments \citep{gymnasium}, and (ii) empirically assess our assumption on the minibatch dependence.
\end{itemize}

\paragraph{Related Work.}
A growing line of work establishes finite-sample guarantees for deep value-based RL.
\citet{fan_dqn} analyze neural fitted $Q$-iteration as an idealized form of DQN and obtain convergence rates to the optimal \(Q\)-function for sparse ReLU networks. 
Subsequent work studies convergence and sample complexity of neural temporal-difference (TD) and $Q$-learning, including overparameterized regimes and mean-field limits \citep{cai_neural_td,sirignano_spiliopoulos_rl_nn,zhang_mean_field_qlearning,xu_dqn,tian_neural_td,ke_improved_neural_td}. 
More directly related to DQN, \citet{zhang_dqn_convergence} provide convergence and sample-complexity guarantees under $\epsilon$-greedy exploration, while \citet{liu_dqn} establish regret bounds under smoothness assumptions. 

Most existing analyses assume i.i.d.\ replay samples or independence across updates.  Some works account for dependence at the trajectory level using Markovian or mixing assumptions \citep{xu_dqn,ke_improved_neural_td,antos_learning_2008}. Relatedly, \citet{shashua_buffer_analysis} analyze replay buffers as stochastic processes and show that specific replay sampling schemes can decorrelate trajectories. Instead, our work studies how the remaining temporal dependence affects the statistical convergence rates of replay-based DQN.

Our analysis builds on viewing DQN updates as supervised regression problems under dependent data. 
While neural network estimators are well understood in the i.i.d.\ setting \citep{schmidt_hieber_bounds,kohler2021rate,jiao_relu_approximation,nagler2026optimalneuralnetworkapproximation}, only limited results exist under dependence, including $\alpha$-mixing \citep{MaSafikani}, $\psi$-mixing \citep{KENGNE2024106163} and $\tau$-mixing \citep{liu_generalisation_tau_mixing_2025}. 
The latter is most closely related to our work, but focuses on sparse networks, more restrictive function classes, and is not embedded in RL. 
Our work combines these learning-theoretic tools with mixing-process techniques \citep{yu_rates,dedecker_new_2005, dedecker_coupling_2004,merlevede_bernstein_mixing} 
to analyze the dependent regression problems arising in DQN.

\paragraph{Notation.}
For \(n \in \mbb{N}\) and \(a,b \in \mbb{R}\), write \([n] := \{1,\dots,n\}\), \(a \lor b := \max\{a,b\}\), and \(a \land b := \min\{a,b\}\). For \(x \in \mbb{R}\), let \(\floor{x}\) and \(\ceil{x}\) denote the floor and ceiling of \(x\), respectively. For positive sequences \(\{f(n),g(n)\}_{n\ge 1}\), we write \(f(n)\lesssim g(n)\) if there exist \(C>0\) and \(n_0\in\mbb{N}\) such that \(f(n)\le C g(n)\) for all \(n\ge n_0\), and \(f(n)\asymp g(n)\) if both \(f(n)\lesssim g(n)\) and \(g(n)\lesssim f(n)\). For a measurable space \(\mc{X}\) we denote by \(\Delta(\mc{X})\) the probability simplex over \(\mc{X}\). Unless necessary, we do not distinguish between scalar- and vector-valued quantities.

\section{Theoretical framework}
\label{sec:setup}

\subsection{Markov decision processes, Bellman updates, and DQN}
\label{subsec:setup_mdp_dqn}

We consider a discounted Markov decision process (MDP) \(\mathfrak{M} = \langle \mc{S}, \mc{A}, P, R, \gamma \rangle,\)
where \(\mc{S}\) and \(\mc{A}\) are the state and action spaces, \(P : \mc{S}\times \mc{A} \to \Delta(\mc{S})\) is the transition kernel, \(R : \mc{S}\times \mc{A} \to \Delta(\mbb{R})\) is the reward distribution, and \(\gamma \in (0,1)\) is the discount factor. At each time step \(t\), the agent observes a reward \(R_t \sim R(\cdot \mid S_t,A_t)\) and transitions to a next state \(S_{t+1} \sim P(\cdot \mid S_t,A_t)\). We assume rewards are uniformly bounded, i.e., \(|R_t| \le R_{\max}\) almost surely for some \(R_{\max} > 0\).

A policy \(\pi : \mc{S} \to \Delta(\mc{A})\) assigns to each state a distribution over actions and induces the action-value function
\begin{align*}
    Q^\pi(s,a)
    &= \mbb{E}\Big[\sum_{t=0}^{\infty} \gamma^t R_t \,\Big|\, S_0=s, A_0=a, A_t \sim \pi(\cdot \mid S_t)\Big] = r(s,a) + \gamma \, \mbb{E}_{S' \sim P(\cdot \mid s,a)}\big[V^\pi(S')\big],
\end{align*}
where 
\begin{align*}
    r(s,a) := \mbb{E}_{R \sim R(\cdot \mid s,a)}[R], \qquad \text{and} \qquad V^\pi(s) = \mbb{E}_{A \sim \pi(\cdot \mid s)}\big[Q^\pi(s,A)\big].
\end{align*}
Since rewards are bounded by \(R_\mathrm{max}\), the \(Q\)-functions satisfy \(| Q^\pi(s, a)| \le \nicefrac{R_\mathrm{max}}{(1 - \gamma)} =: Q_\mathrm{max}\), for any policy \(\pi\).

Our goal is to approximate the optimal \(Q\)-function \(Q^*\), which is the unique fixed point of the Bellman operator \(\mc{T}\) defined as
\begin{align}
\label{eqn:bellman_operator_setup}
    (\mc{T}Q)(s,a)
    := r(s,a) + \gamma \, \mbb{E}_{S' \sim P(\cdot \mid s,a)}
    \Big[\max_{a' \in \mc{A}} Q(S',a')\Big],
\end{align}
i.e., \(Q^* = \mc{T}Q^*\).
Computing \(\mc{T}Q\) exactly is infeasible in large or continuous action spaces, or when the underlying transition dynamics are unknown as in model-free RL. This motivates the use of parametric function classes to approximate action-value functions. In this work, we approximate \(Q^*\) using fully connected feedforward neural networks defined as follows.

\begin{definition}
\label{networks}
Let $M \in \mathbb{N}$ denote the number of hidden layers and let $\mathbf{d}=(d_0,\dots,d_{M+1})$ be the layer widths, where $d_0$ and $d_{M+1}$ are the input and output dimensions, respectively. A fully connected neural network $f:\mathbb{R}^{d_0}\to\mathbb{R}^{d_{M+1}}$ is defined as
\begin{align}
\label{network}
    f(x)
    = W_{M+1}\,\sigma\Big(W_M\,\sigma\big(\cdots \sigma(W_1 x + b_1)\cdots + b_M\big)\Big),
\end{align}
where $W_\ell \in \mathbb{R}^{d_\ell \times d_{\ell-1}}$ and $b_\ell \in \mathbb{R}^{d_\ell}$ are the weights and biases of layer $\ell$, and $\sigma(u)=\max\{u,0\}$ is the ReLU activation applied componentwise. We denote the maximal width of the network by \(d_{\max} = \max_{\ell \in [M]}d_\ell\). 
For $B>0$, define
\begin{equation}
\label{eqn:network_class_setup}
        \begin{aligned}
            \mc{F}(M,d_{\max},B)
            := \Big\{
            f \text{ of the form \eqref{network}} : 
            \max_{\ell \in [L+1]} \|W_\ell\|_\infty \le B,\;
            \max_{\ell \in [L]} \|b_\ell\|_\infty \le B
            \Big\}.
        \end{aligned}
\end{equation}

\end{definition}
We now describe how such function approximators are trained in the DQN algorithm summarized in Algorithm~\ref{algo:dqn}. At environment time \(t\), the agent observes a transition \((S_t,A_t,R_t,S_t')\), which is stored in a replay buffer \(\mc{M} = \{(S_m,A_m,R_m,S_m') : m=1,\dots, |\mc{M}|\}.\)
Here, \(S_m' \sim P(\cdot \mid S_m,A_m)\) denotes the next state and \(R_m \sim R(\cdot \mid S_m,A_m)\) the corresponding reward.

At training time, DQN samples a minibatch $\{(S_i,A_i,R_i,S_i')\}_{i=1}^n$ of size $n$ from the replay buffer to update the network parameters \(\theta\). The regression targets are defined as
\begin{align}
\label{reg_target}
    Y_i
    &:= R_i + \gamma \max_{a \in \mc{A}}
    Q_{\theta^*}(S_i',a).
\end{align}
Here $Q_{\theta^*}$ denotes a target network whose parameters are periodically synchronized with those of the main network $Q_\theta$. The target network mitigates instability due to the moving-target nature of TD learning (see, e.g., \cite{fan_dqn}). The main network is updated by minimizing the squared loss
\begin{equation}
\label{eqn:squared_error_and_targets_setup}
\begin{aligned}
    L(\theta)
    &:= \frac{1}{n}\sum_{i=1}^n
    \Big[Y_i - Q_\theta(S_i,A_i)\Big]^2.
    \qquad
\end{aligned}
\end{equation}
To analyze this update, one can view it as a supervised learning problem. Let $X_i := (S_i,A_i)$ and fix a function $Q:\mc{X}\to\mbb{R}$, which above corresponds to the target network $Q_{\theta^*}$. With the regression targets in~\eqref{reg_target}
we obtain, for $X_i = x$,
\begin{align*}
    \mbb{E}[Y_i \mid X_i = x] = (\mc{T}Q)(x),
\end{align*}
so that the DQN update induces the regression model
\begin{align}
\label{eqn:one_stage_regression_setup}
    Y_i = (\mc{T}Q)(X_i) + \xi_i,
\end{align}
where $\xi_i := Y_i - (\mc{T}Q)(X_i)$ satisfies $\mbb{E}[\xi_i \mid X_i] = 0$. Since $Q$-values are bounded by $Q_{\max}$, the targets satisfy $|Y_i| \le Q_{\max}$. Consequently, \(|\xi_i| \le 2Q_{\max},\)
and in particular $\xi_i$ is sub-Gaussian with parameter proportional to $Q_{\max}$ (e.g., by Hoeffding's lemma \citep[Lemma~2.6]{Massart2007}).

This perspective underlies existing statistical analyses of deep Q-learning such as in \cite{fan_dqn}, which study the generalization error of such regression problems. A common assumption in these analyses is that the samples \(\{(X_i,Y_i)\}_{i=1}^n\) are i.i.d., enabling the use of classical empirical process techniques. In practice, however, the data are generated sequentially through agent-environment interaction and stored in a replay buffer, inducing temporal dependence. This raises the central question of this work: how does such dependence affect the statistical guarantees of deep Q-learning?

\subsection{Weak dependence}
\label{subsec:weak_dependence}

We model the replay-buffer data as a strictly stationary weakly dependent
state-action trajectory \((Z_t)_{t\in\mathbb Z}\) generated under a fixed
behavioral policy. This relaxes the i.i.d. sampling assumption used in prior
work such as \citet{fan_dqn}, while accounting for temporal dependence due to
sequential data collection. Let
\(\mathcal A_t:=\sigma(Z_s:s\le t)\). We say that \((Z_t)\) is exponentially
\(\beta\)-mixing if, for some \(C_\beta,c_\beta>0\),
\begin{align}
\label{eqn:beta_mixing_coefficients}
    \beta_Z(k)
    &:=
    \sup_{t\in\mathbb Z}
    \mathbb E\left[
    \sup_{A\in\sigma(Z_s:s\ge t+k)}
    \big|
        \mathbb P(A\mid \mathcal A_t)-\mathbb P(A)
    \big|
    \right]
    \le
    C_\beta e^{-c_\beta k},
    \qquad k\ge 1 .
\end{align}
We also use the \(\tau\)-mixing coefficient of \citet{dedecker_new_2005}.
For a \(\mathcal Z\)-valued random variable \(V\) on a metric space
\((\mathcal Z,d_{\mathcal Z})\) and a sub-\(\sigma\)-algebra
\(\mathcal A\), define
\begin{align}
\label{eqn:tau_general_setup}
    \tau(\mathcal A,V)
    :=
    \Bigg\|
    \sup_{h\in\Lambda_1(\mathcal Z)}
    \Big|
        \int h\,d\mu_{V\mid\mathcal A}
        -
        \int h\,d\mu_V
    \Big|
    \Bigg\|_1 ,
\end{align}
where \(\Lambda_1(\mathcal Z)\) is the class of real-valued \(1\)-Lipschitz
functions, \(\mu_{V\mid\mathcal A}\) is a conditional law of \(V\) given
\(\mathcal A\), and \(\mu_V\) is its marginal law. This is equivalent to the
coupling formulation of \citet{merlevede_bernstein_mixing}.
For a stationary process, set
\begin{align}
\label{eqn:tau_lag_setup}
    \tau_Z(k)
    :=
    \sup_{t\in\mathbb Z}
    \tau(\mathcal A_t,Z_{t+k}),
    \qquad k\ge 1 .
\end{align}
The process is exponentially \(\tau\)-mixing if
\(\tau_Z(k)\le C_\tau e^{-c_\tau k}\) for some \(C_\tau,c_\tau>0\). The
\(\tau\)-mixing condition is weaker than several classical mixing notions; in particular \(\psi\)-mixing \(\Rightarrow\) \(\phi\)-mixing \(\Rightarrow\) \(\beta\)-mixing \(\Rightarrow\) \(\alpha\)-mixing \(\Rightarrow\) \(\tau\)-mixing
(\citet[p. 112]{bradley_2005_mixing} and
\citet[Remark~1]{liu_generalisation_tau_mixing_2025}). Moreover, on bounded
metric spaces, exponential \(\beta\)-mixing implies exponential
\(\tau\)-mixing with the same exponential rate and a modified prefactor (see
Lemma~\ref{lemma:tau_mixing_from_beta_mixing}).
For a finite sequence \(Y_{1:n}\), define
\[
    \mathcal F_r^Y:=\sigma(Y_1,\dots,Y_r),
    \quad \text{and} \quad
    \tau_Y^n(\ell)
    :=
    \sup_{1\le r\le n-\ell}
    \tau(\mathcal F_r^Y,Y_{r+\ell}),
    \quad \text{for}\quad 1\le \ell\le n-1 .
\]
\begin{proposition}[Replay sampling preserves exponential weak dependence]
\label{prop:mixing_underlying_process_induces_mixing_minibatches}

Let $(X_t)_{t\in\mathbb Z}$ be strictly stationary with values in
$(\mathcal Z,d_{\mathcal Z})$, and denote the replay buffer by
\[
    \mathcal M=(Z_1,\dots,Z_M), \quad \text{with}\quad Z_m := X_{t_0+m-1}.
\]
Let $Y_{1:n}$ be the minibatch obtained from $\mathcal M$ by any order-preserving,
duplicate-free sampler, independent of $(X_t)$.

Then for every $1\le \ell\le n-1$, \(\tau_Y^n(\ell) \le \tau_X(\ell)\).
Consequently:

\begin{enumerate}
    \item[(i)] If $(X_t)$ is exponentially $\tau$-mixing,
\[
    \tau_X(k)\le C_\tau e^{-c_\tau k} \implies \tau_Y^n(\ell)\le C_\tau e^{-c_\tau \ell}.
\]
\item[(ii)] If $(\mathcal Z,d_{\mathcal Z})$ has finite diameter
$D_{\mathcal Z}<\infty$ and $(X_t)$ is exponentially $\beta$-mixing,
\[
    \beta_X(k)\le C_\beta e^{-c_\beta k} \implies \tau_Y^n(\ell)
    \le C_{\beta\tau} D_{\mathcal Z} C_\beta e^{-c_\beta \ell}.
\]
\end{enumerate}
\end{proposition}
\begin{proof}
We defer the proof to Appendix~\ref{app:proofs}.
\end{proof}

Proposition~\ref{prop:mixing_underlying_process_induces_mixing_minibatches} shows that replay sampling preserves the temporal weak-dependence structure of the underlying trajectory: because order-preserving, duplicate-free sampling cannot decrease temporal separation between samples, minibatch dependencies decay at least as fast as those of the original process. Consequently, exponentially mixing trajectories induce exponentially mixing replay minibatches with the same decay rate.

\begin{corollary}[Common replay samplers]
\label{cor:common_replay_samplers}
The conclusions of Proposition~\ref{prop:mixing_underlying_process_induces_mixing_minibatches}
hold for:
(i) uniform sampling without replacement from the replay buffer, provided sampled indices are sorted after sampling, and
(ii) contiguous block sampling. 
\end{corollary}
\begin{proof}
We show in Appendix~\ref{app:proofs} that both sampler fit into the framework of Proposition~\ref{prop:mixing_underlying_process_induces_mixing_minibatches}. The claim then follows immediately.
\end{proof}

\section{Theoretical results}
\label{sec:results}
In the following, we obtain an upper bound for $\|Q^{\pi_K} - Q^*\|_{1,\mu}$, where $Q^{\pi_K}$ is the action-value function corresponding to policy $\pi_K$ and $\mu \in \Delta(\mathcal{S} \times \mathcal{A})$ is the reference distribution. To derive non-trivial results, we need to impose further assumptions on the Bellman operator $\mathcal{T}$ describing our regression function in \eqref{eqn:one_stage_regression_setup}. In particular, we consider the class of regression functions, which are H\"older smooth according to the following definition.

\begin{definition}[H\"older smooth functions]
\label{defn:holder_smooth_functions_setup}
Let \(\mc{D}\subseteq\mbb{R}^d\), where \(d\in\mbb{N}\). A ball of radius \(H\) of \(\beta\)-H\"older smooth functions on \(\mc{D}\) is defined by
\begin{equation*}
        \begin{aligned}
            \mc{C}_d(\mc{D},\beta,H)\!=\!\Bigl\{f:\mc{D}\to\mbb{R}\,\Big|\!&\sum_{\mathbf{\alpha}:\|\mathbf{\alpha}\|_1 < \lfloor \beta \rfloor}
            \|\partial^{\mathbf{\alpha}}f\|_\infty +
            \hspace{-0.25cm}\sum_{\mathbf{\alpha}:\|\mathbf{\alpha}\|_1 = \lfloor \beta \rfloor}
            \sup_{\substack{x,y\in\mc{D} \\ x\neq y}}
            \frac{
                |\partial^{\mathbf{\alpha}}f(x)-\partial^{\mathbf{\alpha}}f(y)|
            }{
                \|x-y\|_\infty^{\beta-\lfloor\beta\rfloor}
            }
            \le H
            \Bigr\}.
        \end{aligned}
\end{equation*}
\end{definition}

We impose an additional structural assumption on the Bellman operator that extends the hierarchical compositional model considered in \cite{fan_dqn}. In particular, the resulting function class also captures compositional functions whose intermediate representations exhibit lower intrinsic dimensionality, quantified through the Minkowski dimension $\dim_M$ by measuring the scaling behavior of their covering numbers. A formal definition of $\dim_M$ is given in the Appendix, see Definition \ref{minkowksi}.

\begin{definition}[see Definition 4.1 in \cite{nagler2026optimalneuralnetworkapproximation}]
    \label{def1}
    A function \( f \colon \mathbb{R}^D \supset \mathcal{M} \to \mathbb{R} \) is called a \emph{compositional model of level}
    \(\ell \in \mathbb{N}\) with dimension vector  $\mathbf{D}=(D_0, \dots, D_\ell)$, where $D_0=D$ and $D_\ell=1$ and \emph{smoothness and intrinsic dimension parameter set}
    \[
        \mathcal{P} = \{(d_{ij}, s_{ij}) \in \mathbb{N} \times \mathbb{R}_{\geq 1} \colon i = 0, \dots, \ell,\ j = 1, \dots, D_{i}\},
    \]
    if there exist functions
        $g_i \colon [-C, C]^{D_i} \to [-C, C]^{D_{i+1}}$, $i = 0, \dots, \ell$,   for some $C>0$,
    such that $g_0 = \text{id}$ (the identity function) and
    \begin{align}
        \label{comp_main}
        f = g_{\ell} \circ g_{\ell-1} \circ \dots \circ g_0,
    \end{align}
    where \( g_i \) is a vector-valued function with components
        $g_i = (g_{i1}, \dots, g_{iD_{i+1}})^{\top}$
    satisfying the following conditions.
    \begin{itemize}
        \item[\textnormal{(C)}] (\emph{Coordinate sparsity}) There exists a subset $S_{ij} \subseteq \{1, \dots, D_i\}$ such that $g_{ij}(x)=g_{ij} \circ \pi_{S_{ij}}(x)$, where $\pi_{S_{ij}}(x) =(x_k)_{k \in S_{ij}}$ denotes the coordinate projection. 
        \item[\textnormal{(S)}] (\emph{Smoothness}) It holds that
        \(g_{ij} \in
        \mc{C}_{|S_{ij}|}([-C, C]^{|S_{ij}|},s_{ij},C).
        \)

        \item[\textnormal{(M)}] (\emph{Low Minkowski dimension}) The set $\mathcal{M}_{ij} := \pi_{S_{ij}} \circ g_{i - 1}(\mathcal{M})$  satisfies $\dim_M({\mathcal{M}_{ij}}) \leq d_{ij}$.
    \end{itemize}
    The corresponding class of compositional models is defined as 
    \begin{align}
    \label{comp_class_main}
        \mathcal{G}(\ell, \mathcal{P},  \mathbf{D}):=\{f \text{ of the form }\eqref{comp_main}\}.
    \end{align}
\end{definition}

Note that this structural assumption strictly generalizes the Bellman-completeness condition of \cite{fan_dqn}. In particular, their framework corresponds to the special case where all intermediate representations are supported on full-dimensional Euclidean domains, although effective dimensionality reduction may still arise through component functions depending only on subsets of their inputs. Our formulation further allows the intermediate representations themselves to be supported on sets of low Minkowski dimension, thereby capturing additional geometric low-dimensional structure.

For our main result we consider the following two function classes: Let
\begin{align}
\label{eqn:functions_classes_0}
    \mc{F}_0 := \mc{F}(M,d_{\max},Q_{\max}), \qquad \text{and }\qquad
    \mc{G}_0 :=
    \mathcal{G}(\ell, \mathcal{P}, \mathbf{D}).
\end{align}
(see Definition \ref{networks} and \ref{comp_main}). Then we set
\begin{align}
    \mc{F}
    &=
    \bigl\{
    f:\mc{S}\times\mc{A}\to\mbb{R}
    \;\big|\;
    f(\cdot,a)\in\mc{F}_0
    \text{ for every } a\in\mc{A}
    \bigr\}, \text{ and}
    \label{eqn:F_class}
    \\
    \mc{G}
    &=
    \bigl\{
    g:\mc{S}\times\mc{A}\to\mbb{R}
    \;\big|\;
    g(\cdot,a)\in\mc{G}_0
    \text{ for every } a\in\mc{A}
    \bigr\}.
    \label{eqn:G_class}
\end{align}

We now state the assumptions required in our analysis.
\begin{assumption}
\label{assumption:main}
We impose the following conditions:
\begin{enumerate}
    \item[(A1)]
    \textbf{Compactness:} the state space is $\mc{S}=[0,1]^d$ and the action space $\mc{A}$ is finite, so that $\mc{X}=\mc{S}\times\mc{A}$ is compact.
    \item[(A2)] \textbf{Exponential mixing:} the sampled minibatch $\{X_i\}_{i=1}^n$ forms a strictly stationary, exponentially $\tau$-mixing process, i.e., there exists constants \(C_\tau, c_1 >0\) such that
    \[
    \tau_X(k) \le C_{\tau} e^{-c_1 k}, \qquad k \ge 1,
    \]
    \item[(A3)] \textbf{Noise independence:} the regression noise in \eqref{eqn:one_stage_regression_setup} satisfies $\mbb{E}[\xi \mid X]=0$ and $\xi \perp X$.
    \item[(A4)] \textbf{Bellman completeness:} the target function $f_0 := \mc{T}Q$ belongs to the compositional H\"older class $\mc{G}$ in ~\eqref{eqn:G_class}, and the function class $\mc{F}$ in ~\eqref{eqn:F_class}  is closed under the Bellman operator, i.e., $\mc{T}f \in \mc{G}$ for all $f \in \mc{F}$.
    \item[(A5)] \textbf{Policy concentrability:} there exists a finite constant $\phi_{\mu,\sigma}<\infty$ controlling the mismatch between the reference distribution \(\mu\) and the state-action distributions induced by the policies encountered during learning.
\end{enumerate}
\end{assumption}

The compositional H\"older structure in (A4) reflects the hierarchical nature of value functions in high-dimensional problems and is widely used to model functions that can be efficiently approximated by deep neural networks (see \cite{schmidt_hieber_bounds, kohler2021rate}). The lower dimension intrinsic dimensionality is further motivated as, in many RL problems, the effective state dynamics and rewards are governed by a small number of underlying degrees of freedom, so that the Bellman target typically varies along a low-dimensional manifold embedded in the high-dimensional state space. The closure condition in (A4) ensures that the Bellman operator preserves this structure and is standard in analyses of deep Q-learning (see\ \citet{fan_dqn}); it also ensures that the regression function \(f_0\) remains in \(\mc G\) over each DQN update.
The concentrability condition in (A5) is common in the offline and batch RL literature; see e.g, \citet{fan_dqn, antos_learning_2008, chen_concentrability, munos_finite_time_bounds}. This assumption can be shown to hold when the underlying MDP satisfies some regularity conditions, see e.g., \citet{chen_concentrability, fan_dqn}.

\begin{theorem}
\label{thm:main_result}
Under Assumption~\ref{assumption:main}, let $\mathcal{G}$ be defined as in \eqref{eqn:G_class} with $\mathcal{G}_0=\mathcal{G}(\ell, \mathcal{P}, \mathbf{D})$ as in \eqref{comp_class_main}. Let \(Q^{\pi_K}\) be the \(Q\)-function after \(K\) iterations of DQN Algorithm~\ref{algo:dqn}, initialized arbitrarily, and let \(Q^*\) satisfy the Bellman optimality equation. Let \(\phi_{\mu,\sigma}<\infty\) be the concentrability constant in Assumption~\ref{assumption:main}(A5), independent of \(n,\mc F,\mc G\), and define
\begin{align}
    (d^*,s^*)\in  \arg\max_{(d,s)\in\mathcal P}\frac{d}{s}, \quad \alpha^*:= \frac{s^*}{s^*+d^*}. 
\end{align} Choose
\begin{align*}
    NL
    \asymp
    \left(\frac{n}{\log^3 n}\right)^{
        \frac{d^*}{4(d^*+s^*)}}
\end{align*}
 and let $\mathcal{F}$ be the class of neural networks defined in \eqref{eqn:F_class} with 
 \begin{align*}   \mathcal{F}_0=\mathcal{F}\!\left(c \big(L + L \sqrt{\log L / \log N}\big), c N, N\right)
\end{align*} for a sufficiently large constant $c=c(\mathbf{D}, \ell)>0$ depending on the dimension vector $\mathbf{D}$ and level $\ell$ of hierarchy in the composition. Then there exists a constant \(C>0\) such that
\begin{align}
\label{eqn:main_dqn_performance_bound}
    \norm{Q^*-Q^{\pi_K}}_{1,\mu}
    \le C\frac{\phi_{\mu,\sigma}}{(1-\gamma)^2} |\mc{A}|^{1/2} \cdot 
    n^{-\alpha^*/2} (\log^3(n))^{\alpha^*/2}+ \frac{4\gamma^{K+1}}{(1-\gamma)^2}R_{\max}.
\end{align}
\end{theorem}

\begin{proof}
The full proof is deferred to Appendix~\ref{app:proofs}.
\end{proof}

The bound in Theorem~\ref{thm:main_result} consists of statistical and algorithmic errors. The latter decays geometrically in \(K\), while the former decays polynomially in \(n\). Thus, for \(K\gtrsim \log n\), the algorithmic term is negligible and the DQN performance is governed by the statistical term.
Up to polylogarithmic factors and treating \(\gamma,\phi_{\mu,\sigma}, |\mc{A}|\) as constants, our rate and the i.i.d.\ rate of \citet{fan_dqn}, adapted to our notation, are
\[
    n^{-s^*/(2s^*+2d^*)}
    \qquad\text{and}\qquad
    n^{-s^*/(2s^*+d^*)},
\]
respectively. Thus temporal dependence deteriorates the exponent from \(\nicefrac{s^*}{2s^*+d^*}\) to \(\nicefrac{s^*}{2s^*+2d^*}\), effectively replacing \(d^*\) by \(2d^*\) in the denominator. The resulting statistical rate is polynomially slower than in the i.i.d.\ setting, especially for larger intrinsic dimension \(d^*\). Intuitively, nearby replay samples contain overlapping information, reducing the effective sample size and slowing concentration relative to the i.i.d.\ case.

\begin{proof}[Proof sketch]
The main difficulty is extending the one-step regression analysis of
\citet{fan_dqn} from i.i.d.\ samples to exponentially \(\tau\)-mixing
trajectories. We first control
\[
    \max_{k\in[K]}
    \|\widehat Q_k-\mathcal T\widehat Q_{k-1}\|_\sigma^2
\]
and then apply the error-propagation result of
\citet[Theorem~6.1]{fan_dqn}.

We denote the ERM by \(\hat{f} = \widehat{Q}_k\) and the regression function by \(f_0 = \mc{T}\widehat{Q}_{k-1}\). The key step is a generic excess-risk decomposition for empirical risk minimization (ERM) under dependent
covariates. For a \(\delta\)-cover \(\mathcal C_\delta\) of \(\mathcal F\),
let the random index \(f_{k^*}\in\mathcal C_\delta\) satisfy
\(\|\hat f-f_{k^*}\|_\infty\le\delta\), and let
\(\{\tilde X_i\}_{i=1}^n\) be an independent exponentially \(\tau\)-mixing ghost sample. Then Proposition~\ref{prop:abstract_excess_risk_decomp_clean} yields
\[
\begin{aligned}
    \mathbb E\|\hat f-f_0\|_\sigma^2
    \le
    \omega(\mc F,\mc G)
    +
    \mathfrak C_n(\mc F,\delta)
    +
    \mathfrak D_n(\mc F,\delta)
    +
    8B\delta ,
\end{aligned}
\]
where \(\omega(\mc F, \mc{G})\) is the approximation error,
\begin{equation*}
        \begin{aligned}
            \mathfrak C_n(\mathcal F,\delta)
            :=
            \sup_{f\in\mathcal F}
            \left|
            \mathbb E\left[
                \frac2n\sum_{i=1}^n
                \xi_i(\hat f(X_i)-f(X_i))
            \right]
            \right|, \;\mathfrak D_n(\mathcal F,\delta)
            :=
            \mathbb E\left[
            \sup_{j\in[\mathcal N_\delta]}
            \left|
                \frac1n\sum_{i=1}^n
                h_j(X_i,\tilde X_i)
            \right|
            \right],
        \end{aligned}
\end{equation*}
with 
\[
h_j(x,y)=(f_j(y)-f_0(y))^2-(f_j(x)-f_0(x))^2.
\]

The covariance term \(\mathfrak{C}_n(\mc{F}, \delta)\) is bounded by adapting the proof in
\citet[Proposition~2]{liu_generalisation_tau_mixing_2025} for sub-Gaussian noise. Since exponential
\(\tau\)-mixing implies \(n_{\mathrm{eff}}\asymp \frac{n}{\log n}\) (see Lemma~\ref{lemma:effective_sample_size_exp_decay_clean})
we obtain, for sufficiently large \(n\) (see Lemma~\ref{lemma:covariance_bound_clean}),
\[
\begin{aligned}
\mathfrak C_n(\mathcal F,\delta)
&\le
C_1\kappa\delta
+
C_2\frac{\log\mathcal N_\delta}{\sqrt n}
\left[
    C_3\sqrt{\frac{\log n}{n}}\log\mathcal N_\delta
    +
    \sqrt\delta
    +
    C_4\sqrt{\mathbb E\|\hat f-f_0\|_n^2}
\right].
\end{aligned}
\]
For the ghost-sample term \(\mathfrak{D}_n(\mc{F}, \delta)\), the paired process
\(\{(X_i,\tilde X_i)\}_{i=1}^n\) is again exponentially \(\tau\)-mixing, and
the same is true for
\(\{h_j(X_i,\tilde X_i)\}_{i=1}^n\) 
(see Lemmas~\ref{lemma:tau_mixing_pairs_clean} and~\ref{lemma:tau_mixing_lipschitz_clean}, respectively, together with Lemma~\ref{lemma:lipschitz_holder_combo_lipschitz_clean}). Applying
\citet[Lemma~1]{liu_generalisation_tau_mixing_2025} with a union bound and
integrating the resulting tail inequality gives (see Lemma~\ref{lemma:squared_differences_bound_clean})
\[
    \mathfrak D_n(\mathcal F,\delta)
    \le
    C B\sqrt{\frac{\log n}{n}}\,
    \log\mathcal N_\delta
    \left(
        \delta+
        \sqrt{\mathbb E\|\hat f-f_0\|_\mu^2}
    \right).
\]
Combining these terms with the previous decomposition, we obtain for any \(\delta \in (0, 1]\) (see Corollary~\ref{cor:dominant_term_simplification_clean})
\[
    \mathbb E\|\hat f-f_0\|_\sigma^2
    \lesssim
    \max\left\{
        \frac{\log n}{n}\log^2\mathcal N_\delta,\,
        \delta
    \right\}
    +
    \omega(\mathcal F,\mathcal G).
\]
It remains to choose the network architecture. 
By Proposition~\ref{prop_approx_reformulated} and~\ref{prop:optimality-bounded-reformulated}, respectively, the approximation error and metric entropy terms are bounded by
\begin{align*}
    \omega(\mathcal F,\mathcal G)
    \lesssim
    (NL)^{-4s^*/d^*}, \qquad \text{and} \qquad \log\mathcal N_\delta
    \lesssim
    (NL)^2\log(NL),
\end{align*}
where we choose \(\delta=(NL)^{-4s^*/d^*}\).
Choosing the architecture to scale as
\[
    NL
    \asymp
    \left(\frac{n}{\log^3 n}\right)^{
        \frac{d^*}{4(d^*+s^*)}
    }.
\]
balances estimation and approximation error terms. For this choice, we obtain (see Theorem~\ref{Bel_reg})
\[
    \max_{k\in[K]}
    \mathbb E
    \|\widehat Q_k-\mathcal T\widehat Q_{k-1}\|_\sigma^2
    \lesssim |\mc{A}|\cdot 
    \left(\frac{\log^3 n}{n}\right)^{\alpha^*},
    \qquad
    \alpha^*=\frac{s^*}{s^*+d^*}.
\]
Applying the error-propagation result of \citet[Theorem~6.1]{fan_dqn} completes the proof.
\end{proof}
\begin{corollary}[Sample complexity of DQN under \(\tau\)-mixing]
\label{cor:dqn_sample_complexity}
Consider the same setting as in Theorem~\ref{thm:main_result}. Assume that the number of iterations satisfy \(K \gtrsim \log n\) so that the statistical error in~\eqref{eqn:main_dqn_performance_bound} dominates. Then, one can achieve any desired accuracy \(\varepsilon > 0\) by taking \( n \gtrsim\varepsilon^{- \nicefrac{2}{\alpha^*}}\).
\end{corollary}
\begin{proof}
We defer the proof to Appendix~\ref{app:proofs}.
\end{proof}

\section{Experiments}
\label{sec:experiments}

We empirically examine whether the dependence assumptions used in our analysis are reflected in data generated by DQN. We study two sources of dependence: state-action trajectories generated by fixed trained policies, and the minibatches supplied to the DQN regression step during training. Experiments are conducted on \texttt{CartPole-v1}, \texttt{FrozenLake-v1}, and \texttt{LunarLander-v3}; implementation details, hyperparameters, and estimator construction are deferred to Appendix~\ref{app:experimental_setting}.

For both trajectories and minibatches, we estimate the finite-history Wasserstein proxy
for \(\tau\)-dependence described in Appendix~\ref{subsec:tau-proxy}. This statistic
compares a local empirical law
of future observations, conditioned on a short observed history, to the corresponding empirical marginal law.
It should therefore be interpreted as a computable finite-sample proxy for temporal dependence, rather than as a direct estimate of the population \(\tau\)-mixing coefficient.
In the trajectory experiments, the lag \(k\)
is ordinary environment-time lag under a fixed greedy policy. In the minibatch experiments, the lag \(k\) is the order in which samples are returned by the replay sampler within a single minibatch.

Figure~\ref{fig:experiments} summarizes the results. The trajectory estimates in the top row show that temporal dependence decays in all three environments, but remains non-negligible at short lags. The decay is relatively fast for \texttt{FrozenLake-v1} and \texttt{LunarLander-v3}, and appears broadly consistent with an exponential \(\tau\)-mixing model. In contrast, \texttt{CartPole-v1} exhibits substantially slower decay, making an exponential trajectory-level mixing model less evident in this environment.

The minibatch estimates in the bottom row show a different but complementary phenomenon. Replay sampling reduces dependence in the data used for DQN updates, but does not eliminate it. Moreover, the estimated dependence decays approximately exponentially as a function of minibatch order. This supports our modeling assumption that the samples entering each DQN regression step retains weak, decaying dependence. The contrast is particularly notable for \texttt{CartPole-v1}: although the underlying trajectory dependence decays slowly, the minibatches produced by the contiguous block sampler exhibit a much faster, approximately exponential decay. This suggests that the sampler can substantially reshape the dependence structure seen by the supervised learning subproblem.

We note that the abrupt cutoff at lag \(k=107\) in the minibatch curves is not a property of the underlying process. It is a finite-sample artifact of the estimator, described in Appendix~\ref{subsec:perm-centering}, with the corresponding cutoff threshold given in~\eqref{eqn:tau_obs_cutoff_loo}. We note a similar effect in the CartPole trajectory at lag \(k = 479\).

\begin{figure}[h]
    \centering
    \begin{subfigure}[b]{0.3\textwidth}
        \centering
        \includegraphics[width=\textwidth]{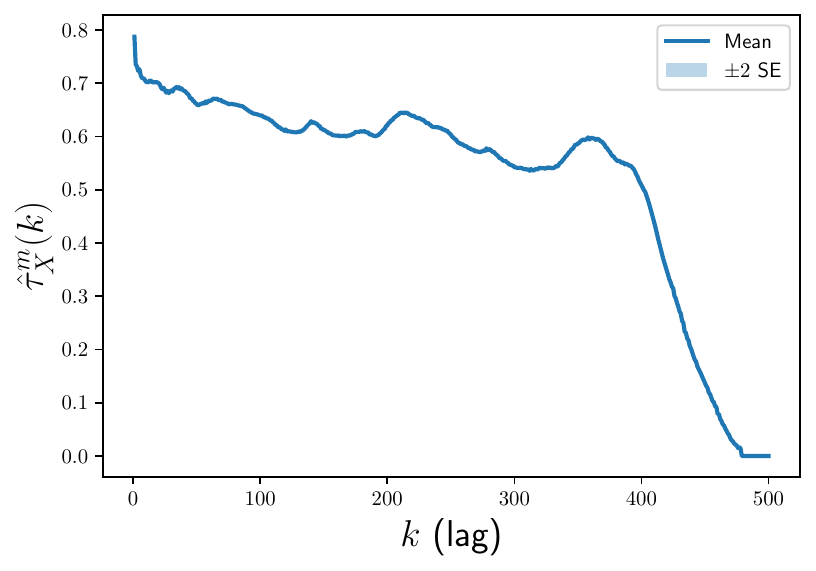}
        \caption{CartPole trajectory}
        \label{fig:sub1}
    \end{subfigure}
    \hfill
    \begin{subfigure}[b]{0.3\textwidth}
        \centering
        \includegraphics[width=\textwidth]{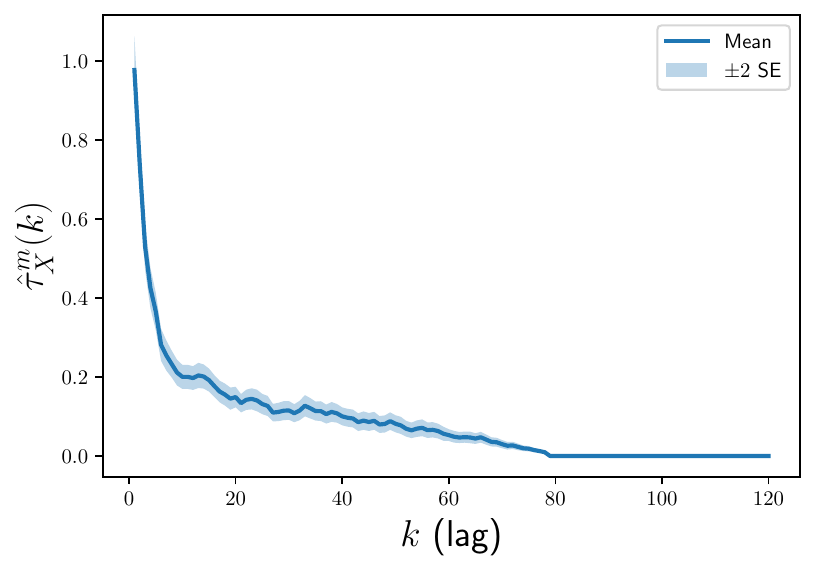}
        \caption{FrozenLake trajectory}
        \label{fig:sub2}
    \end{subfigure}
    \hfill
    \begin{subfigure}[b]{0.3\textwidth}
        \centering
        \includegraphics[width=\textwidth]{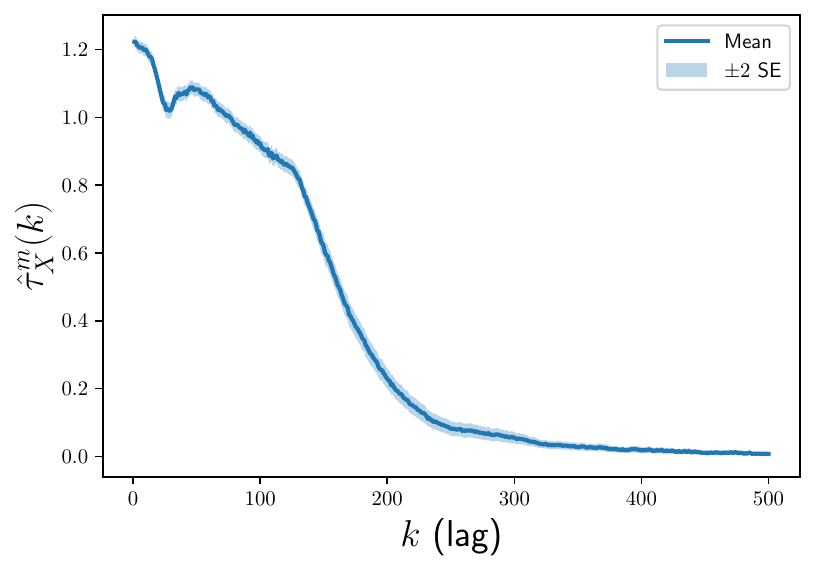}
        \caption{LunarLander trajectory}
        \label{fig:sub3}
    \end{subfigure}

    \vspace{0.5em}

    \begin{subfigure}[b]{0.3\textwidth}
        \centering
        \includegraphics[width=\textwidth]{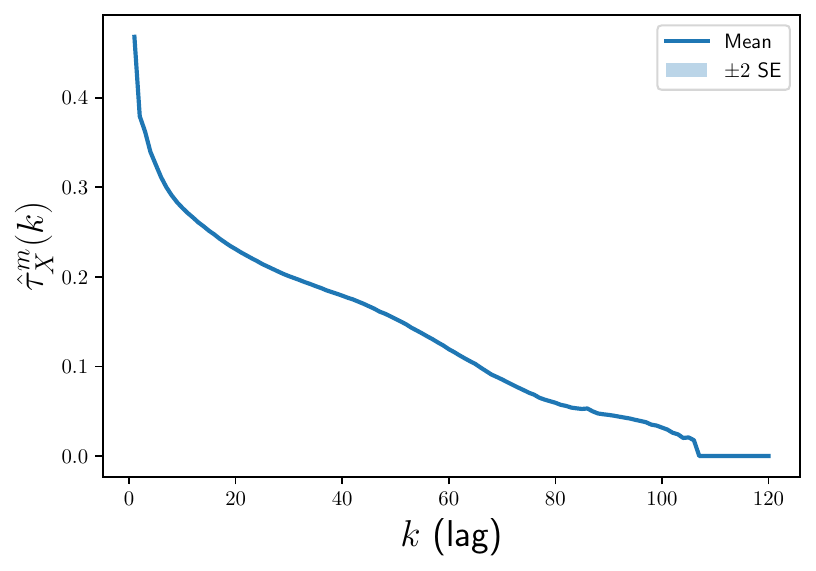}
        \caption{CartPole minibatches}
        \label{fig:sub21}
    \end{subfigure}
    \hfill
    \begin{subfigure}[b]{0.3\textwidth}
        \centering
        \includegraphics[width=\textwidth]{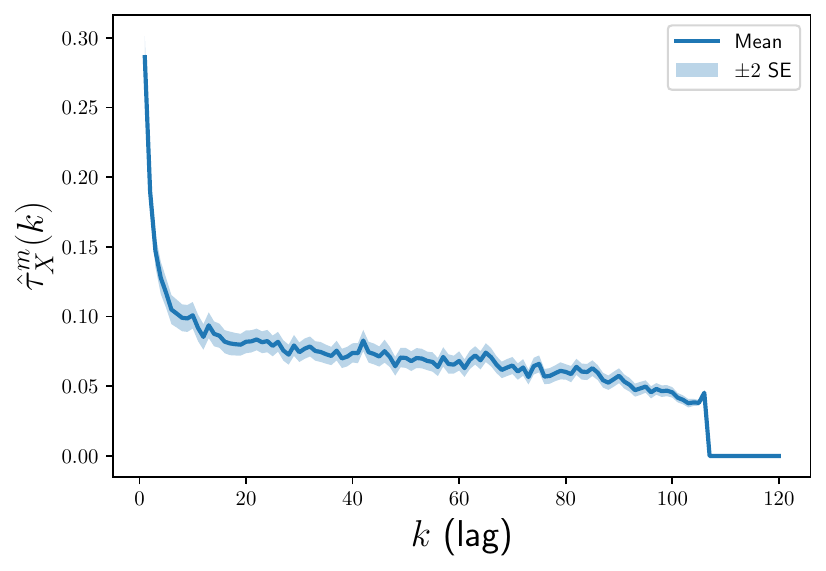}
        \caption{FrozenLake minibatches}
        \label{fig:sub22}
    \end{subfigure}
    \hfill
    \begin{subfigure}[b]{0.3\textwidth}
        \centering
        \includegraphics[width=\textwidth]{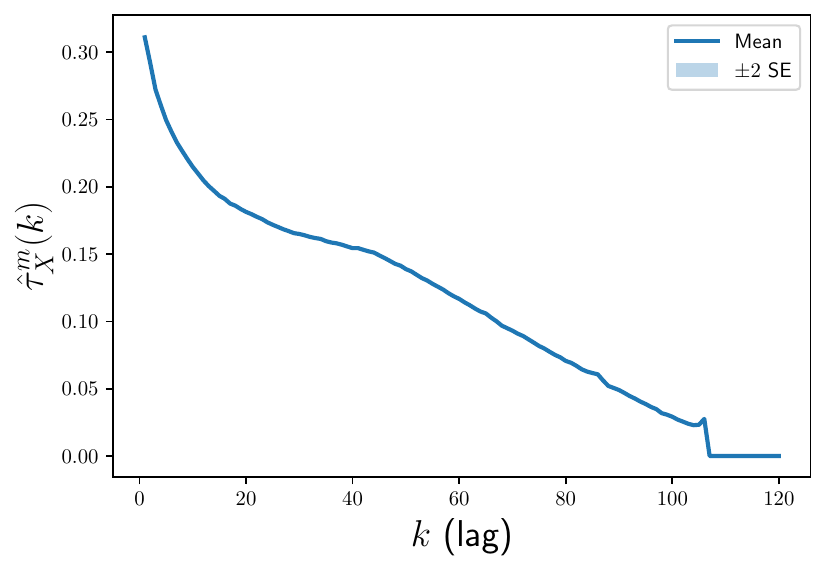}
        \caption{LunarLander minibatches}
        \label{fig:sub23}
    \end{subfigure}

    \caption{Estimated finite-history \(\tau\)-dependence proxies for state--action trajectories generated by fixed trained policies (top row) and for replay minibatches sampled using our contiguous block sampler (see Section~\ref{subsec:contiguous_block_sampling}) used in DQN updates (bottom row). Trajectory estimates show decaying temporal dependence in the underlying RL process, while minibatch estimates show that replay sampling attenuates but does not eliminate dependence. The approximately exponential decay in the minibatch curves is consistent with the dependence model used in our DQN analysis.}
    \label{fig:experiments}
\end{figure}
\vspace{-3mm}

\section{Conclusion \& limitations}
\label{sec:conclusion}

This work extends the analysis of \cite{fan_dqn} to DQN with replay buffer sampling under temporally dependent trajectories. Rather than assuming i.i.d. state-action samples, we model the replay trajectory as a strictly stationary $\beta$-mixing (more generally, $\tau$-mixing) process, thereby capturing the temporal dependence induced by sequential environment interactions. We show that this dependence persists in the replay minibatches and degrades the statistical rate polynomially, effectively doubling the intrinsic dimensionality penalty relative to the i.i.d. setting.

While this model more closely reflects practical DQN training, it still relies on stationarity through the `all-policy' concentrability in Assumption~\ref{assumption:main}(A5). In practice, replay trajectories are generated under evolving behavior policies, inducing distributional drifts in addition to temporal dependence. Extending the analysis to such non-stationary setting would therefore require controlling both dependence and policy-induced distribution shifts. In this direction, \citet{zhan_single_policy_concentrability} show in offline RL that the standard `all-policy' concentrability can be replaced by a weaker `single-policy' notion, suggesting that non-stationarity may be controlled through the local discrepancy induced by individual DQN updates rather than the cumulative drift over the full training trajectory. Extending our analysis to this setting is an interesting direction for future work. Finally, extending this analysis to the multi-agent setting such as the minimax DQN algorithm \citep{fan_dqn} is another interesting direction for future work.

\section*{Acknowledgements}

Sophie Langer was partially supported by the DFG Research Unit 5381 Mathematical Statistics in the Information Age, Projektnumber 460867398.

\bibliography{bibliography}
\bibliographystyle{plainnat}
\newpage

\appendix

\renewcommand{\thealgorithm}{1}

\section{DQN algorithm}
\label{app:dqn_algo}

\begin{algorithm}
\caption{Deep Q-Network (DQN)}
\begin{algorithmic}[1]
\State Initialize replay memory $\mathcal{D}$ with capacity $M$
\State Initialize action-value network $Q_\theta(s,a)$ with random weights $\theta$
\State Initialize target network ${Q}_{\theta^*}(s,a)$ with weights $\theta^* \gets \theta$

\For{episode $= 1, \dots , M$}
    \State Initialize state $s_0$
    
    \For{$t = 0, \dots, T$}
        \State With probability $\epsilon$, select a random action $a_t$
        \State Otherwise select
        \[
        a_t = \arg\max_a Q_\theta(s_t, a)
        \]
        
        \State Execute action $a_t$ in the environment
        \State Observe reward $r_t$ and next state $s_{t}'$
        \State Store transition $(s_t, a_t, r_t, s_{t}')$ in replay memory $\mathcal{D}$
        
        \State Sample a random minibatch of transitions $(s_j, a_j, r_j, s_{j}')$ from $\mathcal{D}$
        
        \For{each sampled transition}
            \If{$S_{j}'$ is terminal}
                \State Set target
                \[
                y_j = r_j
                \]
            \Else
                \State Set target
                \[
                y_j = r_j + \gamma \max_{a'} {Q}_{\theta^*}(s_{j}', a')
                \]
            \EndIf
        \EndFor
        
        \State Perform a gradient descent step on
        \[
        L(\theta) =
        \frac{1}{|\mathcal{B}|}
        \sum_{j \in \mathcal{B}}
        \left(Y_j - Q_\theta(s_j, a_j)\right)^2
        \]
        
        \State Every $T_\mathrm{target}$ steps, update the target network:
        \[
        \theta^* \gets \theta
        \]
    \EndFor
\EndFor
\end{algorithmic}
\label{algo:dqn}
\end{algorithm}

\section{Proofs}
\label{app:proofs}
Throughout the appendix, we use the population and empirical norms
\begin{align*}
    \|f\|_\mu^2 := \int f^2\,d\mu,
    \qquad
    \|f\|_n^2 := \frac1n\sum_{i=1}^n f(X_i)^2.
\end{align*}

In Definition \ref{comp_class_main} we introduce our class of compositional models, which relies on the following formal definition of the Minkowski dimension: 
\begin{definition}
[Minkowski dimension]
\label{minkowksi}
    The Minkowski dimension of a bounded set $\mathcal{M} \subset \mathbb{R}^D$ is defined as
    \begin{align*}
        \dim_M(\mathcal{M}) := \limsup_{\epsilon \to 0} \frac{\log \mathcal{N}(\epsilon, \mathcal{M}, \| \cdot \|_\infty)}{\log(1/\epsilon)}.
    \end{align*}
\end{definition}

\subsection{Proof of Proposition~\ref{prop:mixing_underlying_process_induces_mixing_minibatches}}

\begin{proof}[Proof of Proposition~\ref{prop:mixing_underlying_process_induces_mixing_minibatches}]
    
We first consider assumption \emph{(i)}. Since the replay buffer
\(\mathcal M\) is a contiguous subsequence of the underlying trajectory,
Lemma~\ref{lemma:exp_tau_mixing_buffer} shows that \(\mathcal M\) inherits
the same exponential truncated \(\tau\)-mixing bound:
\[
    \tau_{\mathcal M}^M(k)
    \le
    C_\tau e^{-c_\tau k},
    \qquad 1\le k\le M-1 .
\]
Fix \(1 \le r \le n-\ell\) and condition on the sampled indices \(I_1 = i_1, \dots, I_n = i_n\). Since the sampling is \textit{order-preserving} and \textit{duplicate-free}, we have that \(I_1 < I_2 < \dots < I_n\). Moreover,
\begin{align*}
    i_{(r+\ell)} - i_{(r)} \ge \ell.
\end{align*}
Furthermore,
\begin{align*}
    \mc Y_r
    =
    \sigma(Z_{i_{(1)}},\dots,Z_{i_{(r)}})
    \subseteq
    \sigma(Z_1,\dots,Z_{i_{(r)}}).
\end{align*}
Therefore, by Lemma~\ref{lemma:zeta_monotone_sigma_algebra},
\begin{align*}
    \tau(\mc Y_r, Y_{r+\ell})
    &=
    \tau(\mc Y_r, Z_{i_{(r+\ell)}})
    \\
    &\le
    \tau\bigl(\sigma(Z_1,\dots,Z_{i_{(r)}}), Z_{i_{(r+\ell)}}\bigr)
    \\
    &\le
    \tau_{\mc M}^M(i_{(r+\ell)} - i_{(r)})
    \\
    &\le
    c_0 e^{-c_1 (i_{(r+\ell)} - i_{(r)})}
    \\
    &\le
    c_0 e^{-c_1 \ell}.
\end{align*}
Taking the supremum over \(1 \le r \le n-\ell\) yields
\begin{align*}
    \tau_{\mc Y}^n(\ell) \le c_0 e^{-c_1 \ell},
    \qquad 1 \le \ell \le n-1.
\end{align*}
Since the above argument holds for every realized ordered index vector
\[
    G := (I_{(1)},\dots,I_{(n)}) = (i_{(1)},\dots,i_{(n)}),
\]
the bound holds conditionally on \(G\), uniformly over all admissible
realizations of \(G\). More explicitly, for every \(1\)-Lipschitz function
\(h\),
\[
\begin{aligned}
&\mathbb E\left[
    \sup_{h\in\Lambda_1}
    \big|
        \mathbb E[h(Y_{r+\ell})\mid \mathcal Y_r,G]
        -
        \mathbb E[h(Y_{r+\ell})\mid G]
    \big|
\right]  \\
& =
\mathbb E_G\left[
    \mathbb E\left[
        \sup_{h\in\Lambda_1}
        \left|
            \mathbb E[h(Y_{r+\ell})\mid \mathcal Y_r,G]
            -
            \mathbb E[h(Y_{r+\ell})\mid G]
        \right|
        \,\middle|\, G
    \right]
\right]  \\
&\qquad \le
\mathbb E_G\left[c_0 e^{-c_1\ell}\right]
=
c_0 e^{-c_1\ell}.
\end{aligned}
\]
Thus the same exponential bound holds after averaging over the random
sampling indices. 

Under assumption \emph{(ii)}, Lemma~\ref{lemma:tau_mixing_from_beta_mixing}
implies that the exponentially \(\beta\)-mixing trajectory is exponentially
\(\tau\)-mixing with constants
\[
    C_\tau = C_{\beta\tau}D_{\mathcal Z}C_\beta,
    \qquad
    c_\tau = c_\beta .
\]
Applying the preceding argument with these constants gives
\[
    \tau_Y^n(\ell)
    \le
    C_{\beta\tau}D_{\mathcal Z}C_\beta e^{-c_\beta \ell},
    \qquad 1\le \ell\le n-1 .
\]
\end{proof}

\begin{proof}[Proof of Corollary~\ref{cor:common_replay_samplers}]
We first consider assumption \emph{(i)}. 
For the case of uniform sampling without replacement, the sampled indices satisfy \(I_1 < I_2 < \dots < I_n\) and hence the claim follows directly from Proposition~\ref{prop:mixing_underlying_process_induces_mixing_minibatches}. For the contiguous block sampler, we show in Corollary~\ref{cor:contiguous_block_sampling_buffer} that the sampled indices satisfy the condition of Proposition~\ref{prop:mixing_underlying_process_induces_mixing_minibatches}, and hence the claim follows immediately.

Under assumption \emph{(ii)}, Lemma~\ref{lemma:tau_mixing_from_beta_mixing}
implies that the exponentially \(\beta\)-mixing trajectory is exponentially
\(\tau\)-mixing with constants
\[
    C_\tau = C_{\beta\tau}D_{\mathcal Z}C_\beta,
    \qquad
    c_\tau = c_\beta .
\]
Applying the preceding argument with these constants proves the claim.
\end{proof}

\subsection{Error propagation in DQN}
To prove our main result we make use of the following lemma that translates the maximum one-step regression error into a performance guarantee for DQN.

\begin{lemma}[Error propagation]
\label{lemma:error_propagation}
    Let \(\{\widehat{Q}_k\}_{k \in [K]}\) be the iterative \(Q\)-functions of the DQN algorithm, and let \(Q^*\) denote the optimal \(Q\)-function satisfying the Bellman optimality equation. Furthermore, let \(\phi_{\mu, \sigma} < \infty\) be a constant independent of \(n, \mc{F}, \) and \(\mc{G}\). Then
    \begin{align*}
         \norm{Q^* - Q^{\pi_K}}_{1, \mu} \le \frac{2\phi_{\mu, \sigma}}{(1- \gamma)^2}  \underset{k \in [K]}{\max} \norm{\widehat{Q}_k - \mc{T}\widehat{Q}_{k-1}}_\sigma + \frac{4 \gamma^{K+1}}{(1-\gamma)^2}  R_\mathrm{max}.
    \end{align*}
\end{lemma}

\begin{proof}
    See Theorem~6.1 of \citet{fan_dqn}.
\end{proof}

We begin by controlling the one-step statistical error $\underset{k \in [K]}{\max} \norm{\widehat{Q}_k - \mc{T}\widehat{Q}_{k-1}}_\sigma$ of the fitted
Bellman update. This step reduces to a nonparametric regression problem
with dependent data:

Fix an iteration \(k\), and define the regression targets
\[
Y_i^{(k)} := R_i + \gamma \max_{a'\in\mathcal A} Q_k(S_i',a').
\]
Let \(X_i := (S_i,A_i)\). Then the fitted Q-iteration update satisfies
\[
Q_{k+1} \in \arg\min_{f\in\mathcal F}
\frac{1}{n}\sum_{i=1}^n (f(X_i) - Y_i^{(k)})^2.
\]

We denote by
\[
m_k(x) := \mathbb E[Y^{(k)} \mid X=x]
\]
the regression function, which coincides with the Bellman operator:
\(m_k = \mathcal{T}Q_k\).

In the following, we first establish a general bound for least-squares regression
under exponentially $\tau$-mixing observations. This result will later be
applied to the fitted Bellman updates above.
\subsection{Nonparametric regression under exponentially $\tau$-mixing data}
Let \(\mc X\subseteq \mbb R^d\) be compact, and let \(\mu\) denote the stationary marginal law of the covariates. Assume without loss of generality that \(\mc X=[0,1]^d\). We observe a strictly stationary sample
\begin{align*}
    \mc D_n=\{(X_i,Y_i)\}_{i=1}^n,
\end{align*}
where \(\{X_i\}_{i\ge1}\) is an exponentially \(\tau\)-mixing process with marginal law \(\mu\), and
\begin{align}
\label{eqn:general_regression_model}
    Y_i=f_0(X_i)+\xi_i,\qquad i=1,\dots,n.
\end{align}
Here \(f_0:\mc X\to\mbb R\) is the unknown regression function and \(\xi_i\) is the regression noise. We assume \(f_0\) belongs to a function class \(\mc G\).

Let \(\mc F\) be a candidate function class. The estimator is the empirical risk minimizer
\begin{align}
\label{eqn:general_erm}
\hat f \in \arg\min_{f\in\mc F}\frac1n\sum_{i=1}^n \bigl(Y_i-f(X_i)\bigr)^2.
\end{align}
We will work under the following standing assumptions.
\begin{assumption}[Dependent-design regression model]
\label{ass:general_dep_reg_clean}
The data-generating process satisfies the following conditions.
\begin{enumerate}
    \item[(i)] \textbf{Stationary exponentially \(\tau\)-mixing covariates:}
    the process \(\{X_i\}_{i\ge1}\) is strictly stationary with marginal law \(\mu\), and there exist constants \(c_0,c_1>0\) such that
    \begin{align*}
        \tau_X(k)\le c_0 \exp(-c_1 k),\qquad k\ge1.
    \end{align*}
    \item[(ii)] \textbf{Independence:}
    \(X_i\perp \xi_i\) for each \(i\ge1\).
    \item[(iii)] \textbf{Centered noise:}
    for each \(i\),
    \begin{align*}
        \mbb E[\xi_i]=\mbb E[\xi_i\mid X_i]=0.
    \end{align*}
    \item[(iv)] \textbf{Sub-Gaussian noise:}
    there exists \(\kappa>0\) such that almost surely,
    \begin{align*}
        \mbb E\!\left[\exp(\lambda \xi_i)\right]
        \le \exp\!\left(\frac{\kappa^2\lambda^2}{2}\right),
        \quad \forall \lambda\in\mbb R.
    \end{align*}
    \item[(v)] \textbf{Uniform boundedness:}
    there exists \(B>0\) such that
    \begin{align*}
        \|f_0\|_\infty\le B,
        \qquad
        \sup_{f\in\mc F}\|f\|_\infty\le B.
    \end{align*}
    \item[(vi)] \textbf{Function classes:} \(\mc G = \mc G(\ell, \mathcal{P}, \mathbf{D})\) defined as in \eqref{comp_class_main} and \(\mathcal{F}=\mathcal{F}(M, d_{\max}, B)\) defined as in \eqref{eqn:network_class_setup}.
\end{enumerate}
\end{assumption}
We note that these assumptions coincide with Assumption \ref{assumption:main} of our main result, such that the corresponding results can be directly applied.

For \(\delta>0\), we denote by \(\mathcal N(\delta, \mathcal F, \|\cdot\|_\infty)\)
the covering number of \(\mathcal F\) with respect to the sup-norm,
i.e., the smallest integer \(N\) such that there exist
\(f_1,\dots,f_N\in\mathcal F\) satisfying
\[
\forall f\in\mathcal F,\quad
\min_{1\le i\le N}\|f-f_i\|_\infty \le \delta.
\]
If clear from the context, we use the shorthand notation $\mathcal{N}_{\delta} := \mathcal N(\delta, \mathcal F, \|\cdot\|_\infty)$. Let \( \mathcal C_\delta := \{f_1,\dots,f_{\mathcal N_\delta}\} \subset \mathcal F\)\ be the corresponding minimal
\(\delta\)-cover of \(\mathcal F\) in \(\|\cdot\|_\infty\), and let
\(f_{k^*}\in\mathcal C_\delta\) be a measurable projection of \(\hat f\)
onto this cover, so that
\[
    \|\hat f-f_{k^*}\|_\infty\le \delta .
\]
 Let \(\tilde X_1,\dots,\tilde X_n\) be an independent exponentially \(\tau\)-mixing sample with marginal law \(\mu\) and independent of \(\mathcal A_n:=\sigma(X_1,\dots,X_n)\).
Define
\begin{align}
\mathfrak C_n(\mathcal F,\delta)
&:=
\sup_{f\in\mathcal F} \left|\mathbb E\left[
    \frac2n\sum_{i=1}^n
    \xi_i\bigl(\hat f(X_i)-f(X_i)\bigr)
\right]\right|,
\\
\mathfrak D_n(\mathcal F,\delta)
&:=
\mathbb E\left[\sup_{j \in [N_\delta]}
    \left|
    \frac1n\sum_{i=1}^n
    h_{j}(X_i,\tilde X_i)
    \right|
\right],
\end{align}
where
\begin{align}
\label{hj}
     h_{j}(x,y)
    :=
    (f_{j}(y)-f_0(y))^2
    -
    (f_{j}(x)-f_0(x))^2 
\end{align}

and define the approximation error
\begin{align}
\label{eqn:general_uniform_approx_error}
\omega(\mc F,\mc G):=\sup_{g\in\mc G}\inf_{f\in\mc F}\|f-g\|_\mu^2.
\end{align}

\begin{proposition}[Expected excess-risk decomposition]
\label{prop:abstract_excess_risk_decomp_clean}
Under Assumption~\ref{ass:general_dep_reg_clean}, for any \(\delta>0\),
\begin{align}
    \mathbb{E}\norm{\hat f-f_0}_\mu^2
    \le
    \omega(\mc F,\{f_0\})
    + \mathfrak C_n(\mc F,\delta)
    + \mathfrak D_n(\mc F,\delta)
    + 8B\delta .
\end{align}
More generally, if \(f_0\in\mc G\), then
\begin{align}
    \sup_{f_0\in\mc G} \mathbb{E}
    \norm{\hat f-f_0}_\mu^2
    \le
    \omega(\mc F,\mc G)
    + \mathfrak C_n(\mc F,\delta)
    + \mathfrak D_n(\mc F,\delta)
    + 8B\delta .
\end{align}
\end{proposition}

\begin{proof}
Fix \(\delta>0\). Since \(\hat f\in\mathcal F\), there exists an
\(\mathcal A_n\)-measurable random index \(k^*\in\mathcal [N_\delta]\) such
that
\[
    \|\hat f-f_{k^*}\|_\infty\le\delta .
\]

For every \(x\in\mathcal X\),
\[
\begin{aligned}
    |(\hat{f}(x)-f_0(x))^2 - (f_{k^*}(x)-f_0(x))^2|
    &=
    |\hat f(x)-f_{k^*}(x)|
    |\hat f(x)+f_{k^*}(x)-2f_0(x)|  \\
    &\le 4B\delta,
\end{aligned}
\]
where the last inequality follows as \(|\hat f(x)|,|f_{k^*}(x)|,|f_0(x)|\le B\).
Because \(\tilde X_i\sim\mu\) and is independent of \(\mathcal A_n\),
    \[
    \|\hat f -f_0\|_\mu^2
    =
    \mathbb E\left[
        \frac1n\sum_{i=1}^n
        (\hat f(\tilde X_i)-f_0(\tilde X_i))^2
        \,\middle|\,\mathcal A_n
    \right].
\]

Using the preceding \(4B\delta\) bound twice, once for \(\tilde X_i\) and
once for \(X_i\), gives
\[
\begin{aligned}
    \|\hat f-f_0\|_\mu^2
    \le
    \|\hat f-f_0\|_n^2
    +
    \mathbb E\left[
        \frac1n\sum_{i=1}^n h_{k^*}(X_i,\tilde X_i)
        \,\middle|\,\mathcal A_n
    \right]
    +8B\delta .
\end{aligned}
\]
Taking expectations and bounding the random index by the supremum over the
cover yields
\begin{align}
\label{eq:exp_fhat_minus_fzero_mu}
    \mathbb E\|\hat f-f_0\|_\mu^2
    \le
    \mathbb E\|\hat f-f_0\|_n^2
    +
    \mathbb E\left[
        \sup_{j \in [N_\delta]}
        \left|
        \frac1n\sum_{i=1}^n h_j(X_i,\tilde X_i)
        \right|
    \right]
    +8B\delta .
\end{align}

It remains to control the empirical norm term. By the ERM property, for any
\(f\in\mathcal F\),
\[
    \frac1n\sum_{i=1}^n (Y_i-\hat f(X_i))^2
    \le
    \frac1n\sum_{i=1}^n (Y_i-f(X_i))^2 .
\]
Substituting \(Y_i=f_0(X_i)+\xi_i\) and expanding gives
\begin{align}
\label{erm_form}
        \|\hat f-f_0\|_n^2
    \le
    \|f-f_0\|_n^2
    +
    \frac2n\sum_{i=1}^n \xi_i(\hat f(X_i)-f(X_i)).
\end{align}
Taking expectations, then taking the infimum over \(f\in\mathcal F\) in the first term and the supremum in the second, yields
    \begin{align}
    \label{eq:exp_fhat_minus_fzero_n}
        \mathbb E\|\hat f-f_0\|_n^2
        \le
        \inf_{f\in\mathcal F} \mathbb E\|f-f_0\|_n^2
        + \sup_{f\in\mathcal F} \left| \mathbb E\left[
            \frac2n\sum_{i=1}^n
            \xi_i(\hat f(X_i)-f(X_i))
        \right]\right|.
    \end{align}
Since \(X_i\sim\mu\) for every \(i\),
\begin{align}
\label{eq:relation_n_mu}
    \mathbb E\|f-f_0\|_n^2
    =
    \mathbb \|f-f_0\|_\mu^2 .
\end{align}
Using this with the definition of \(\mathfrak{C}_n(\mc F, \delta)\) yields
\begin{align}
    \label{eqn:erm_intermediate}
    \norm{\hat{f} - f_0}_\mu^2 \le \inf_{f\in\mathcal F} \mathbb E\|f-f_0\|_n^2 + \mathfrak{C}_n(\mc F, \delta).
\end{align}
Combining \eqref{eq:exp_fhat_minus_fzero_mu} with \eqref{eq:exp_fhat_minus_fzero_n}, \eqref{eq:relation_n_mu}, and \eqref{eqn:erm_intermediate} proves the first claim. The uniform version
follows by taking the supremum over \(f_0\in\mathcal G\) and using the
definition of \(\omega(\mathcal F,\mathcal G)\).
\end{proof}

\subsubsection{Stochastic error bound}
To control \(\mathfrak C_n(\mc F,\delta)\) and \(\mathfrak D_n(\mc F,\delta)\), we use the following results.

\begin{lemma}[Stochastic covariance bound]
\label{lemma:covariance_bound_clean}
Suppose Assumption~\ref{ass:general_dep_reg_clean} holds. Then there exist
constants \(C_1,C_2,C_3, C_4 >0\) such that, for all sufficiently large \(n\),
\[
\begin{aligned}
\mathfrak C_n(\mathcal F,\delta)
&\le C_1\kappa\delta
+
C_2 \frac{\log \mathcal N_\delta}{\sqrt n}
\left[
C_3\sqrt{\frac{\log n}{n}}\log\mathcal N_\delta
+
\sqrt\delta
+
C_4\sqrt{\mathbb E\|\hat f-f_0\|_n^2}
\right].
\end{aligned}
\]
\end{lemma}

\begin{proof}
Since \(\xi_i\) are centered random variables and \(\mbb{E}[\xi_i \mid X_i]= 0\) it follows that 
\begin{align*}
    \mbb{E}[\xi_i \cdot g(X_i) \mid X_i] = g(X_i)\cdot \mbb{E}[\xi_i \mid X_i] = 0    
\end{align*}
for any bounded, measurable, and deterministic function \(g\). For a fixed \(f \in \mc F\) this then implies that
\begin{align}
\label{eqn:mb_deterministic_expectations}
    \mbb{E}\left[\frac2n\sum_{i=1}^n\xi_i \cdot f(X_i)\right] = \mbb{E}\left[\frac2n\sum_{i=1}^n\xi_i \cdot f_0(X_i)\right] = 0,
\end{align}
since both, \(f\) and \(f_0\), are measurable, bounded, and deterministic.
Consequently,
\begin{align}
    \label{eqn:f_tq_intermediate_clean}
    \mbb{E}\left\{\frac2n\sum_{i=1}^n \xi_i \cdot \left[\hat{f}(X_i) - f(X_i) \right] \right\} = \mbb{E}\left \{\frac2n\sum_{i=1}^n \xi_i \cdot \left[\hat{f}(X_i) - f_0(X_i) \right] \right\}.
\end{align}
Note that, since \(\xi_i\) is sub-Gaussian, it is also sub-exponential and its moments are bounded by 
\begin{align}
\label{eqn:sub_gauss_moments_bound_clean}
    \mbb{E}\left[\abs{\xi_i}^m\right] \le \left(c \sqrt{m}\right)^m,
\end{align}
for some constant \(c > 0\) (see Proposition~2.6.1.(ii) of \citet{Vershynin_2026}).
Furthermore, \(\norm{f}_\infty , \norm{f_0}_\infty \le B\;\forall f \in \mc{F}\), \(\norm{X_i}_\infty \le 1\), and since \(X_i\) forms an exponentially \(\tau\)-mixing sequence, we can adapt Proposition~2 of \citet{liu_generalisation_tau_mixing_2025} to bound the RHS of~\eqref{eqn:f_tq_intermediate_clean}. Let
\(C_\delta=\{f_1,\dots,f_{\mathcal N_\delta}\}\) be a fixed
\(\delta\)-cover of \(\mathcal F\) in \(\|\cdot\|_\infty\). Since
\(\hat f\in\mathcal F\), there exists a random index
\(k^*=k^*(\hat f)\in[\mc{N}_\delta]\) such that
\[
\|\hat f-f_{k^*}\|_\infty\le \delta .
\]
The only adaptation of Proposition~2 of \citet{liu_generalisation_tau_mixing_2025} needed is to account for the sub-Gaussian parameter \(\kappa\) of the noise \(\xi\). Notice that, in their notation, we have an output dimension \(\mu_{q+1}=1\), and so we do not have the double summation as in their proof; moreover, they denote the noise term by \(\epsilon\), correpsonding to our \(\xi\). 

First, we apply the triangle inequality to the RHS of~\eqref{eqn:f_tq_intermediate_clean}, which gives
\begin{align}
\label{eqn:covariance_triangle_inequality_clean}
    \left| \mbb{E}\left \{\frac2n \sum_{i=1}^n \xi_i \cdot \left[\hat{f}(X_i) - f_0(X_i) \right]\right\} \right|
    &\le \left| \mbb{E} \left\{\frac2n\sum_{i=1}^n \xi_i \cdot \left[\hat{f}(X_i) - f_{k^*}(X_i) \right] \right\}\right| \notag\\
    &\quad+ \left|\mbb{E} \left\{\frac2n\sum_{i=1}^n \xi_i \cdot \Big[f_{k^*}(X_i) - f_0(X_i) \Big] \right\}\right|,
\end{align}
where \(f_{k^*}\) satisfies \(\norm{\hat{f} - f_{k^*}}_\infty \le \delta\). 
For the first term on the RHS of~\eqref{eqn:covariance_triangle_inequality_clean} we use the triangle inequality together with the bound \(\norm{\hat{f} - f_{k^*}}_\infty \le \delta\) and Cauchy-Schwarz inequality to obtain
\begin{align*}
    \left|\mbb{E} \left\{\frac2n\sum_{i=1}^n \xi_i \cdot \left[\hat{f}(X_i) - f_{k^*}(X_i) \right] \right\} \right|
    &\le \frac2n \sum_{i=1}^n \mbb{E}\left[\abs{\xi_i}\cdot \abs{\hat{f}(X_i) - f_{k^*}(X_i)}\right] \\
    &\le \frac2n \sum_{i=1}^n \delta\, \mbb{E}\left[\abs{\xi_i}\right] \\
    &\le \frac2n \cdot n \cdot \delta \sqrt{\mbb{E}[\xi_i^2]}.
\end{align*}
Since \(\xi_i \sim \mc{S}\mc{G}(\kappa)\), Lemma~\ref{lemma:sub_gauss_bounded_variance_clean} implies \(\mathrm{Var}(\xi_i) = \mbb{E}\left[\xi_i^2\right]\le \kappa^2\). It then follows that
\begin{align}
    \label{eqn:cauchy_schwartz_twice_clean}
    \left|\mbb{E} \left\{\frac2n\sum_{i=1}^n \xi_i \cdot \left[\hat{f}(X_i) - f_{k^*}(X_i) \right] \right\} \right| \le 2
    \,\kappa\cdot \delta.
\end{align}
This is the sub-Gaussian analogue of (23) in the proof of Proposition~2 of \citet{liu_generalisation_tau_mixing_2025}. 

(24) in \citet{liu_generalisation_tau_mixing_2025} (in our notation and setting) shows that:
\begin{align*}
    \abs{\mbb{E}\left\{\frac2n\sum_{i=1}^n\left[ \xi_i \hat{f}(X_i)\right]\right\}} 
    &=  \left| \mbb{E}\left \{\frac2n \sum_{i=1}^n \xi_i \cdot \left[\hat{f}(X_i) - f_0(X_i) \right]\right\} \right|
\end{align*}
Proposition~2 of \citet{liu_generalisation_tau_mixing_2025} is used to bound the LHS of the equality above, and hence it also holds for the RHS.
Therefore, we can now directly apply Proposition~2 of \citet{liu_generalisation_tau_mixing_2025} to bound the RHS of~\eqref{eqn:f_tq_intermediate_clean}, after accounting for the sub-Gaussian parameter \(\kappa\) in the second term of~\eqref{eqn:f_tq_intermediate_clean}. This then yields:
\begin{align}
\label{eqn:liu_proposition2_clean}
    \left|
\mathbb E\left\{
\frac2n \sum_{i=1}^n
\xi_i\bigl[\hat{f}(X_i)-f_0(X_i)\bigr]
\right\}
\right|
&\le
C\,\frac{\log \mathcal N_\delta}{\sqrt n}
\left\{
\frac{\log \mathcal N_\delta}{\sqrt{n_{\rm eff}}}
+\sqrt{\delta}
+\sqrt{2\,\mathbb E\|\hat{f}-f_0\|_n^2}
\right\} + 2\kappa \cdot \delta.
\end{align}

where \(n_\mathrm{eff}\) is the effective sample size induced by \(\tau\)-mixing and \(C > 0\) is a constant depending on \(B\), the Lipschitz constant \(L_f\), and the constant \(c > 0 \) in \eqref{eqn:sub_gauss_moments_bound_clean} (see Appendix~\ref{app:effective_sample_size} for a general discussion on effective sample sizes for \(\tau\)-mixing sequences).

Under exponential \(\tau\)-mixing, Lemma~\ref{lemma:effective_sample_size_exp_decay_clean} implies \(\exists\, C', C'' > 0\) such that, for sufficiently large \(n\),
\begin{align*}
    n_\mathrm{eff} \ge C' \cdot \frac{n}{\log n}
    \implies
    \frac{1}{\sqrt{n_\mathrm{eff}}} \le \sqrt{C''} \cdot \sqrt{\frac{\log n}{n}}.
\end{align*}
It then follows from~\eqref{eqn:f_tq_intermediate_clean} and~\eqref{eqn:liu_proposition2_clean} that
\begin{align*}
    &\left| \mbb{E}\left\{\frac2n \sum_{i=1}^n \xi_i \left[\hat{f}(X_i) - f(X_i)\right]\right\}\right|\\
    &\le C \,\frac{\log \mc{N}_\delta}{\sqrt{n}} \left(\sqrt{C''} \sqrt{\frac{\log n}{n}} \log{\mc{N}_\delta} + \sqrt{\delta} + \sqrt{2} \sqrt{\mbb{E}\left[\norm{\hat{f} - f_0}\right]}\right) + 2\kappa \cdot 
    \delta.
\end{align*}
Absorbing all absolute constants into \(C_1, C_2, C_3, C_4 > 0\) we conclude that
\begin{align*}
    &\left|\mbb{E} \left\{ \frac2n \sum_{i=1}^n \xi_i \cdot \left[\hat{f}(X_i) - f(X_i) \right]\right\}\right| \\
    &\quad \le C_1\,\kappa\cdot \delta + C_2 \frac{\log \mc{N}_\delta}{\sqrt{n}} \left[ C_3 \sqrt{\frac{\log n}{n}} \log \mc{N}_\delta + \sqrt{\delta} + C_4 \sqrt{\mbb{E}\left[\norm{\hat{f} - f_0}_n^2\right]} \right],
\end{align*}
where we take
\begin{align*}
    C_1 = 2, \quad C_2 = 2C, \quad C_3 = \sqrt{C''}, \quad C_4 = \sqrt{2}.
\end{align*}
Finally, we take a supremum over all \(f \in \mc{F}\) on both sides of the last inequality to complete the proof.
\end{proof}

\begin{lemma}[Squared-differences bound]
\label{lemma:squared_differences_bound_clean}
Suppose Assumption~\ref{ass:general_dep_reg_clean} holds. Then there exists
\(C>0\) such that, for all sufficiently large \(n\),
\[
\mathfrak D_n(\mathcal F,\delta)
\le
C B\sqrt{\frac{\log n}{n}}\,
\log \mathcal N_\delta
\left(
\delta+
\sqrt{\mathbb E\|\hat f-f_0\|_\mu^2}
\right).
\]
\end{lemma}

\begin{proof}
Recall the definition of \(h_j\) as in~\eqref{hj}
\begin{align*}
    h_j(x, y) = \left[f_j(y) - f_0(y)\right]^2 - \left[f_j(x) - f_0(x)\right]^2,
\end{align*}
where 
\(f_j \in C_{\delta}\) is in the minimal $\delta$-cover of $\mathcal{F}$ with all functions being \(L_f\)-Lipschitz and \(f_0 \in \mc G\) is \(H_f\)-H\"{o}lder with order \(\beta_f\) by Assumption~\ref{ass:general_dep_reg_clean}.
Lemma~\ref{lemma:lipschitz_holder_combo_lipschitz_clean} then implies that \(h_j\) is \(L_h\)-Lipschitz for any \(j \in [\mc{N}_\delta]\), where we define
\begin{align}
    L_h &:= B\cdot \left(L_{f} + H_f \cdot \max\left\{1, D_{\mc{X}}^{\beta_f - 1}\right\}\right).
\end{align}
Here \(\mc{X}\) is compact and hence its diameter satisfies \(D_\mc{X} < \infty\).
By construction of the ghost sample, we have that \(\{X_i\}_{i \in [n]}\) and \(\{\tilde{X}_i\}_{i \in [n]}\) are independent \(\tau\)-mixing sequences, with 
\begin{align*}
    \tau_X(k) = \tau_{\tilde{X}}(k) \le c_0 \exp(-c_1 k) \quad \forall k \ge 1.
\end{align*}
Lemma~\ref{lemma:tau_mixing_pairs_clean} implies that the paired process \(\{Z_i := (X_i, \tilde{X}_i)\}_{i \in [n]}\) forms a \(\tau\)-mixing sequence with 
\begin{align*}
    \tau_Z(k) &\le \tau_{X}(k) + \tau_{\tilde{X}}(k) \\
    &\le 2 c_0 \exp(-c_1 k) \quad \forall k \ge 1.
\end{align*}
Lemma~\ref{lemma:tau_mixing_lipschitz_clean} then implies that the process 
\begin{align*}
   \{h_j(Z_i)\} _{i \in [n]} = \{h_j(X_i, \tilde{X}_i)\}_{i \in [n]}
\end{align*}
is \(\tau\)-mixing for any \(j \in [\mc{N}_\delta]\) with coefficients given by
\begin{align}
    \label{eqn:tau_coeffs_h_clean}
    \tau_{h}(k) &\le L_h \tau_Z(k) \notag\\
    &= L_h \cdot 2 c_0 \exp(-c_1k) \notag\\
    &= \tilde{c}_0 \exp(-c_1 k),
\end{align}
where we write \(\tilde{c}_0 = 2 L_h c_0\).

We will now use the concentration inequality in \citet[Lemma~1]{liu_generalisation_tau_mixing_2025}.
First, we note that \(h_j\) is centered and absolutely bounded (satisfying the first two conditions in Lemma~1 of \citep{liu_generalisation_tau_mixing_2025}). Since \(\norm{f_j}_\infty\le B\) and \(\norm{f_0}_\infty\le B\), we have
\begin{align*}
    |h_j(x,y)|
    &= \left|\left(f_j(y)-f_0(y)\right)^2-\left(f_j(x)-f_0(x)\right)^2\right| \\
    &\le \left|f_j(y)-f_0(y)\right|^2 + \left|f_j(x)-f_0(x)\right|^2 \\
    &\le 4B^2 + 4B^2 = 8B^2.
\end{align*}
Hoeffding's Lemma then implies that \(h_j\) is sub-Gaussian (see e.g., \citet{Massart2007}[Lemma~2.6]),
\begin{align*}
    h_j \sim \mc{S}\mc{G}(8B^2).
\end{align*}
Lemma~\ref{lemma:sub_gauss_bounded_variance_clean} then ensures that the variance of \(h_j\) is bounded (satisfying the third condition of Lemma~1 of \citep{liu_generalisation_tau_mixing_2025}) by
\begin{align*}
    \mathrm{Var}(h_j) = \mbb{E}_{\sigma \times \sigma}\left[h_j^2 \right] \le 64 B^4.
\end{align*}
This verifies that our sequence \(\{h_j(X_i, \tilde{X}_i)\}_{i \in [n]}\) satisfies the conditions in Lemma~1 of \citet{liu_generalisation_tau_mixing_2025}.
We now refine the variance bound following~(C.50) of \citet{fan_dqn}. Since \(X_i\) and \(\tilde{X}_i\) are independent with the same marginal law \(\mu\), we have that
\begin{align}
    \mathrm{Var}[h_j(X_i, \tilde{X}_i)] &= 2 \mathrm{Var}\left[\left(f_j(X_i) - f_0(X_i)\right)^2\right] \notag \\
    \label{eqn:h_j_moment_variance_bound_clean}
    &\quad\le 2 \,\mbb{E}\left[\left(f_j(X_i) - f_0(X_i)\right)^4\right] \\
    \label{eqn:h_j_variance_bound_upsilon}
    &\quad \le \Upsilon^2, 
\end{align}
where
\begin{align}
\label{eqn:upsilon_variance_proxy_clean}
\Upsilon^2:=\max\left(16B^4 \cdot \frac{\log \mc N_\delta}{n_{\mathrm{eff}}^{(h)}},\;8B^2 \cdot \max_{j\in[\mc N_\delta]}\norm{f_j- f_0}_\mu^2
\right),
\end{align}
and
\begin{align}
    n_\mathrm{eff}^{(h)} = \max\left\{1 \le m \le n: \tau_h\left(\left\lfloor \frac{n}{m} \right\rfloor\right) \le \frac{4B^2}{\sqrt{m}}\right\}
\end{align}
is the effective sample size of \(\{h_j(X_i, \tilde{X}_i)\}_{i \in [n]}\).
Recall that 
\begin{align*}
\tau_X(k) \le c_0 \exp(-c_1k), \quad \tau_h(k)\le \tilde{c}_0 \exp(-c_1 k).
\end{align*}
Therefore, we have by Corollary~\ref{cor:effective_sample_size_comparison} that \(n_\mathrm{eff} \asymp n_\mathrm{eff}^{(h)}\) and  hence we do not distinguish between the two effective sample sizes \(n_\mathrm{eff}\) and \(n_\mathrm{eff}^{(h)}\) and simply write \(n_\mathrm{eff}\) throughout the remainder of this proof.
By the definition of \(\Upsilon\),
\begin{align*}
    \Upsilon^2 \ge 16B^4 \frac{\log \mc N_\delta}{n_{\mathrm{eff}}},
\end{align*}
and hence
\begin{align*}
    \Upsilon \ge 4B^2 \sqrt{\frac{\log \mc N_\delta}{n_{\mathrm{eff}}}}.
\end{align*}
Therefore, we have that
\begin{align}
\label{eqn:normalized_uniform_bound}
    \left|\frac{h_j(X_i,\tilde X_i)}{\Upsilon}\right|
    \le
    \frac{8B^2}{\Upsilon}
    \le
    2\sqrt{\frac{n_{\mathrm{eff}}}{\log \mc N_\delta}},
    \qquad \forall i\in[n],\; \forall j\in[\mc N_\delta].
\end{align}
Using again \(|f_j(X_i)-f_0(X_i)|\le 2B\) we obtain
\begin{align*}
    \left(f_j(X_i)-f_0(X_i)\right)^4
    \le 4B^2 \left(f_j(X_i)-f_0(X_i)\right)^2,
\end{align*}
and therefore by the variance bound in~\eqref{eqn:h_j_moment_variance_bound_clean},
\begin{align*}
    \mathrm{Var}\!\left(h_j(X_i,\tilde X_i)\right)
    &\le 2 \cdot 4B^2 \, \mbb E\!\left[\left(f_j(X_i)-f_0(X_i)\right)^2\right] \\
    &= 8B^2 \norm{f_j-f_0}_\mu^2.
\end{align*}
Since, by definition,
\begin{align*}
    \Upsilon^2 \ge 8B^2 \max_{j\in[\mc N_\delta]} \norm{f_j-f_0}_\mu^2
    \ge 8B^2 \norm{f_j-f_0}_\mu^2,
\end{align*}
it follows that
\begin{align}
\label{eqn:normalized_variance_bound}
    \mathrm{Var}\!\left(\frac{h_j(X_i,\tilde X_i)}{\Upsilon}\right)
    =
    \frac{\mathrm{Var}\!\left(h_j(X_i,\tilde X_i)\right)}{\Upsilon^2}
    \le 1,
    \qquad \forall j\in[\mc N_\delta].
\end{align}

Furthermore, we note that the normalized random variables \(h_j(X_i, \tilde{X}_i) / \Upsilon\) are still sub-Gaussian. Therefore all conditions of Lemma~1 of \citet{liu_generalisation_tau_mixing_2025} are satisfied by the normalized process.
Next, we define
\begin{align}
\label{eqn:T_concentration_parameter}
    T := \underset{j \in [\mc{N}_\delta]}{\sup} \abs{\sum_{i=1}^n \frac{h_j(X_i, \tilde{X}_i)}{\Upsilon}}.
\end{align}
We now apply Lemma~1 of \citet{liu_generalisation_tau_mixing_2025} together with a union bound to obtain, for all \(t \ge 0\)
\begin{align*}
    \mbb{P}(T \ge t) \le 2 \mc{N}_\delta \exp\left(-\frac{t}{\phi(n)}\right),
\end{align*}
where we define
\begin{align}
    \label{eqn:phi_n_gamma_n_clean}
    \phi(n) := \frac{n \cdot \left[13 \sqrt{n_\mathrm{eff}} + 21 \,\Gamma_n\right]}{n_\mathrm{eff}}, \qquad \Gamma_n := C_1 \sqrt{\frac{n_\mathrm{eff}}{\log \mc{N}_\delta}},
\end{align}
where \(C_1 := 2\).
Since \(T\) is non-negative, we can integrate over this bound in probability to obtain a bound in expectation
\begin{align}
\label{eqn:expectation_integral_bound}
    \mbb{E}(T) &= \int_0^{\infty} \mbb{P}(T \ge t)dt \le \theta + \int_{\theta}^\infty \mbb{P}(T \ge t)dt \notag\\
    &= \theta + 2\, \mc{N}_\delta \cdot \phi(n) \cdot \exp\left(-\,\frac{\theta}{\phi(n)}\right) =: g(\theta)
\end{align}
for any \(\theta \ge 0\).
Noticing that \(g(\theta)\) is convex in \(\theta\), we can minimize this upper bound by solving for \(\theta\) in 
\begin{align*}
    g'(\theta) = 0,
\end{align*}
which yields
\begin{align*}
    \theta = \phi(n) \cdot \log(2\,\mc{N}_\delta).
\end{align*}
Therefore, we have that
\begin{align}
\label{eqn:intermediate_expectation_bound_clean}
    \mbb{E}[T]
    &\le \phi(n) \log(2\mc{N}_\delta) + 2\, \mc{N}_\delta \phi(n)\exp \left(- \log(2 \mc{N}_\delta) \right) \notag \\
    &= \phi(n) \log(2\mc{N}_\delta) + \phi(n).
\end{align}
Since \(\mc{N}_\delta \ge 3 \ge 2\), we have \(\log \mc{N}_\delta \ge \log 2\), and hence
\begin{align*}
\phi(n)
\le \frac{1}{\log 2}\,\phi(n)\log \mc{N}_\delta.
\end{align*}
Also, since \(\mc{N}_\delta \ge 2\),
\begin{align*}
\log(2\mc{N}_\delta)=\log 2+\log \mc{N}_\delta \le 2\log \mc{N}_\delta.
\end{align*}
Substituting these two bounds into~\eqref{eqn:intermediate_expectation_bound_clean}, we obtain
\begin{align}
\label{eqn:intermediate_expectation_bound_no_plus_one}
    \mbb{E}[T]
    \le \left(2+\frac{1}{\log 2}\right)\phi(n)\log \mc{N}_\delta.
\end{align}
Therefore,
\begin{align}
\label{eqn:expectation_bound_intermediate_no_plus_one_clean}
     \mbb{E} \left[\underset{j \in [\mc{N}_\delta]}{\sup} \left|\frac{\Upsilon}{n} \sum_{i=1}^n \frac{h_j(X_i, \tilde{X}_i)}{\Upsilon} \right|\right]
     &\le \Upsilon\left(2+\frac{1}{\log 2}\right)\frac{\phi(n)}{n}\log \mc{N}_\delta.
\end{align}
Since \(\{h_j(X_i, \tilde{X}_i)\}_{i \in [n]}\) is exponentially \(\tau\)-mixing with coefficients given in~\eqref{eqn:tau_coeffs_h_clean}, and since \(\Upsilon\) is a constant, the normalized process \(\{h_j(X_i, \tilde{X}_i) / \Upsilon\}_{i \in [n]}\) is also exponentially \(\tau\)-mixing. Therefore, Lemma~\ref{lemma:effective_sample_size_exp_decay_clean} implies that there exists a constant \(C_2>0\) such that, for all sufficiently large \(n\),
\begin{align*}
    n_\mathrm{eff} \ge C_2 \frac{n}{\log n}.
\end{align*}
Equivalently, there exists a constant \(C_3>0\) such that, for all sufficiently large \(n\),
\begin{align*}
    \frac{1}{n_\mathrm{eff}} \le C_3 \frac{\log n}{n}.
\end{align*}
Substituting this bound into the definition of \(\phi(n)\) in~\eqref{eqn:phi_n_gamma_n_clean}, we obtain
\begin{align*}
    \phi(n) &= n\cdot \left[13\sqrt{C_3}\sqrt{\frac{\log n}{n}} + 21\,C_1 \sqrt{C_3}\sqrt{\frac{\log n}{n}} \frac{1}{\sqrt{\log \mc{N}_\delta}}\right]\\
    &= n \cdot \left[C_4 \sqrt{\frac{\log n}{n}} + C_5 \sqrt{\frac{\log n}{n}}\frac{1}{\sqrt{\log \mc{N}_\delta}}\right],
\end{align*}
where we set \(C_4 = 13 \sqrt{C_3}, C_5 = 21 \,C_1\sqrt{C_3}\). Substituting this into~\eqref{eqn:expectation_bound_intermediate_no_plus_one_clean} we obtain
\begin{align*}
    \mbb{E} \left[\underset{j \in [\mc{N}_\delta]}{\sup} \left|\frac{\Upsilon}{n} \sum_{i=1}^n \frac{h_j(X_i, \tilde{X}_i)}{\Upsilon} \right|\right]
     &\le \Upsilon C_6 \left[C_4 \sqrt{\frac{\log n}{n}} + C_5 \sqrt{\frac{\log n}{n}}\frac{1}{\sqrt{\log \mc{N}_\delta}}\right] \log \mc{N}_\delta\\
     &= \Upsilon C_6 \left[C_4 \sqrt{\frac{\log n}{n}}\log \mc{N}_\delta +C_5\sqrt{\frac{\log n}{n}}\sqrt{\log \mc{N}_\delta} \right].
\end{align*}
where \(C_6 = 2 + \frac{1}{\log 2}\). 
Now, since \(\mc{N}_\delta \ge 3\) we have \(\log \mc{N}_\delta \ge 1\) and therefore there exists a constant \(C_7 > 0 \) such that
\begin{align}
\label{eqn:expectation_bound_upsilon_clean}
    \mbb{E} \left[\underset{j \in [\mc{N}_\delta]}{\sup} \left|\frac{\Upsilon}{n} \sum_{i=1}^n \frac{h_j(X_i, \tilde{X}_i)}{\Upsilon} \right|\right] &\le \Upsilon C_6 C_4 C_7 \sqrt{\frac{\log n}{n}} \log \mc{N}_\delta\notag\\
    &=C_8 \cdot \Upsilon \sqrt{\frac{\log n}{n}}\log \mc{N}_\delta,
\end{align}
where we take \(C_8 := C_4 C_6 C_7\). 

Following \citet{fan_dqn} we note that, by the definition of \(\Upsilon\) in~\eqref{eqn:upsilon_variance_proxy_clean} we have that 
\begin{align*}
    \Upsilon &\le \max \left(4B^2 \sqrt{\frac{\log \mc{N}_\delta}{n_\mathrm{eff}}}, \sqrt{8}B \underset{j \in [\mc{N}_\delta]}{\max} \norm{f_j - f_0}_\mu\right) \\
    &\quad\le \max \left(4B^2 \sqrt{\frac{\log \mc{N}_\delta}{n_\mathrm{eff}}}, \sqrt{8}B \left(\delta + \norm{\hat{f} - f_0}_\mu \right)\right).
\end{align*}
We need only consider the case where \(\Upsilon \le \sqrt{8}B\left(\delta + \norm{\hat{f} - f_0}_\mu\right)\). This is because, when using the case \(\Upsilon \le 4\,B^2 \sqrt{\nicefrac{\log \mc{N}_\delta}{n_\mathrm{eff}}}\) we obtain a smaller upper bound in the subsequent theorem and hence the bound obtained from \(\Upsilon \le \sqrt{8}B\left(\delta + \norm{\hat{f} - f_0}_\mu\right)\) holds in both cases. Substituting this into~\eqref{eqn:expectation_bound_upsilon_clean} we obtain
\begin{align*}
    \mbb{E} \left[\underset{j \in [\mc{N}_\delta]}{\sup} \left|\frac{\Upsilon}{n} \sum_{i=1}^n \frac{h_j(X_i, \tilde{X}_i)}{\Upsilon} \right|\right] \le C_9 B \sqrt{\frac{\log n}{n}}\log \mc{N}_\delta \cdot \left(\delta + \norm{\hat{f} - f_0}_\mu\right),
\end{align*}
where \(C_9 := \sqrt{8}C_8\).
Setting \(C := C_9\) completes the proof.
\end{proof}

Combining Proposition~\ref{prop:abstract_excess_risk_decomp_clean} with Lemmas~\ref{lemma:covariance_bound_clean} and~\ref{lemma:squared_differences_bound_clean} yields the following bound.
\begin{proposition}[Generic risk bound under dependent covariates]
\label{prop:general_risk_bound_clean}
Suppose Assumption~\ref{ass:general_dep_reg_clean} holds and
\(\mathcal N_\delta\ge 3\). Then there exist constants
\(C_1,C_2,C_3,C_4>0\), independent of \(n\), \(\delta\), and
\(\mathcal N_\delta\), such that, for every \(\delta>0\) and all sufficiently
large \(n\),
\[
\begin{aligned}
\mathbb E\|\hat f-f_0\|_\mu^2
&\le
C_1\,\omega(\mathcal F,\{f_0\})
+
C_2 B^2\frac{\log n}{n}\log^2\mathcal N_\delta  \\
&\quad+
C_3
\max\left\{
B\delta\sqrt{\frac{\log n}{n}},
\frac{\sqrt\delta}{\sqrt n}
\right\}
\log\mathcal N_\delta
+
C_4(\kappa+B)\delta .
\end{aligned}
\]
Consequently, if \(f_0\in\mathcal G\), then
\[
\begin{aligned}
\sup_{f_0\in\mathcal G}
\mathbb E\|\hat f-f_0\|_\mu^2
&\le
C_1\,\omega(\mathcal F,\mathcal G)
+
C_2 B^2\frac{\log n}{n}\log^2\mathcal N_\delta  \\
&\quad+
C_3
\max\left\{
B\delta\sqrt{\frac{\log n}{n}},
\frac{\sqrt\delta}{\sqrt n}
\right\}
\log\mathcal N_\delta
+
C_4(\kappa+B)\delta .
\end{aligned}
\]
\end{proposition}

\begin{proof}
Fix \(\delta>0\). Since
\(\mathcal N_\delta\ge 3\), we have \(\log \mc{N}_\delta\ge 1\). We first derive an
auxiliary bound for the expected empirical error.

By the ERM inequality obtained in the proof of
Proposition~\ref{prop:abstract_excess_risk_decomp_clean},
\[
\mathbb E\|\hat f-f_0\|_n^2
\le
\inf_{f\in\mathcal F}\mathbb E\|f-f_0\|_n^2
+
\mathfrak C_n(\mathcal F,\delta).
\]
Since \(X_i\sim\mu\) for every \(i\),
\[
\inf_{f\in\mathcal F}\mathbb E\|f-f_0\|_n^2
=
\inf_{f\in\mathcal F}\|f-f_0\|_\mu^2
=
\omega(\mathcal F,\{f_0\}).
\]
Applying Lemma~\ref{lemma:covariance_bound_clean} gives
\[
\begin{aligned}
\mathbb E\|\hat f-f_0\|_n^2
&\le
\omega(\mathcal F,\{f_0\})
+
C_1\kappa\delta  \\
&\quad+
C_2\frac{\log \mc{N}_\delta}{\sqrt n}
\left[
C_3\sqrt{\frac{\log n}{n}}\log \mc{N}_\delta
+
\sqrt\delta
+
C_4\sqrt{\mathbb E\|\hat f-f_0\|_n^2}
\right].
\end{aligned}
\]
Set
\[
a:=\sqrt{\mathbb E\|\hat f-f_0\|_n^2},
\qquad
b:=\frac{C_2C_4}{2}\frac{\log \mc{N}_\delta}{\sqrt n},
\]
and let
\[
c:=
\omega(\mathcal F,\{f_0\})
+
C_1\kappa\delta
+
C_2C_3\frac{\sqrt{\log n}}{n}\log^2 \mc{N}_\delta
+
C_2\frac{\log \mc{N}_\delta}{\sqrt n}\sqrt\delta .
\]
Then the preceding inequality is of the form \(a^2\le 2ab+c\). Using the elementary inequality
\begin{align}
\label{eqn:elementary_inequality}
    a^2\le \frac{(1+\epsilon)^2}{\epsilon}b^2+(1+\epsilon)c,
    \qquad \epsilon\in(0,1],
\end{align}
and setting \(\epsilon=1\), we obtain
\begin{align}
\label{eqn:empirical_error_bound_clean}
\mathbb E\|\hat f-f_0\|_n^2
&\le
2\omega(\mathcal F,\{f_0\})
+
2C_1\kappa\delta
+
C_2^2C_4^2\frac{\log^2 \mc{N}_\delta}{n}
\notag\\
&\quad+
2C_2C_3\frac{\sqrt{\log n}}{n}\log^2 \mc{N}_\delta
+
2C_2\frac{\log \mc{N}_\delta}{\sqrt n}\sqrt\delta .
\end{align}

Next, we use the population--empirical comparison from the proof of
Proposition~\ref{prop:abstract_excess_risk_decomp_clean}, namely
\[
\mathbb E\|\hat f-f_0\|_\mu^2
\le
\mathbb E\|\hat f-f_0\|_n^2
+
\mathfrak D_n(\mathcal F,\delta)
+
8B\delta .
\]
By Lemma~\ref{lemma:squared_differences_bound_clean}, there exists \(C_5>0\)
such that
\[
\mathfrak D_n(\mathcal F,\delta)
\le
C_5B\sqrt{\frac{\log n}{n}}\log \mc{N}_\delta
\left(
\delta+
\sqrt{\mathbb E\|\hat f-f_0\|_\mu^2}
\right).
\]
Combining this with \eqref{eqn:empirical_error_bound_clean} yields
\[
\begin{aligned}
\mathbb E\|\hat f-f_0\|_\mu^2
&\le
2\omega(\mathcal F,\{f_0\})
+
2C_1\kappa\delta
+
C_2^2C_4^2\frac{\log^2 \mc{N}_\delta}{n}
+
2C_2C_3\frac{\sqrt{\log n}}{n}\log^2 \mc{N}_\delta + 8B\delta
  \\
&\quad +
2C_2\frac{\log \mc{N}_\delta}{\sqrt n}\sqrt\delta+
C_5B\sqrt{\frac{\log n}{n}}\log \mc{N}_\delta\,\delta
+
C_5B\sqrt{\frac{\log n}{n}}\log \mc{N}_\delta
\sqrt{\mathbb E\|\hat f-f_0\|_\mu^2}.
\end{aligned}
\]
Now set
\[
a:=\sqrt{\mathbb E\|\hat f-f_0\|_\mu^2},
\qquad
b:=\frac{C_5}{2}B\sqrt{\frac{\log n}{n}}\log \mc{N}_\delta,
\]
and let
\[
\begin{aligned}
c
&:=
2\omega(\mathcal F,\{f_0\})
+
2C_1\kappa\delta
+
C_2^2C_4^2\frac{\log^2 \mc{N}_\delta}{n}
+
2C_2C_3\frac{\sqrt{\log n}}{n}\log^2 \mc{N}_\delta  \\
&\quad+
2C_2\frac{\log \mc{N}_\delta}{\sqrt n}\sqrt\delta
+
C_5B\delta\sqrt{\frac{\log n}{n}}\log \mc{N}_\delta
+
8B\delta .
\end{aligned}
\]
Then the preceding inequality is again of the form \(a^2\le 2ab+c\). Applying
\eqref{eqn:elementary_inequality} with \(\epsilon=1\), we get
\[
\begin{aligned}
\mathbb E\|\hat f-f_0\|_\mu^2
&\le
4\omega(\mathcal F,\{f_0\})
+
C_5^2B^2\frac{\log n}{n}\log^2 \mc{N}_\delta
+
2C_2^2C_4^2\frac{\log^2 \mc{N}_\delta}{n}
+
4C_2C_3\frac{\sqrt{\log n}}{n}\log^2 \mc{N}_\delta \\
&\quad+
4C_2\frac{\log \mc{N}_\delta}{\sqrt n}\sqrt\delta
+
2C_5B\delta\sqrt{\frac{\log n}{n}}\log \mc{N}_\delta
+
4C_1\kappa\delta
+
16B\delta .
\end{aligned}
\]
For all sufficiently large \(n\),
\[
2C_2^2C_4^2\frac{\log^2 \mc{N}_\delta}{n}
+
4C_2C_3\frac{\sqrt{\log n}}{n}\log^2 \mc{N}_\delta
\le
\left(2C_2^2C_4^2+4C_2C_3\right)
B^2\frac{\log n}{n}\log^2 \mc{N}_\delta .
\]
Furthermore,
\[
4C_2\frac{\log \mc{N}_\delta}{\sqrt n}\sqrt\delta
+
2C_5B\delta\sqrt{\frac{\log n}{n}}\log \mc{N}_\delta
\le
(4C_2+2C_5)
\max\left\{
B\delta\sqrt{\frac{\log n}{n}},
\frac{\sqrt\delta}{\sqrt n}
\right\}
\log \mc{N}_\delta .
\]
Finally,
\[
4C_1\kappa\delta+16B\delta
\le
\max\{4C_1,16\}(\kappa+B)\delta .
\]
Substituting these three bounds gives
\[
\begin{aligned}
\mathbb E\|\hat f-f_0\|_\mu^2
&\le
4\,\omega(\mathcal F,\{f_0\})
+
\left(C_5^2+2C_2^2C_4^2+4C_2C_3\right)
B^2\frac{\log n}{n}\log^2\mathcal N_\delta  \\
&\quad+
(4C_2+2C_5)
\max\left\{
B\delta\sqrt{\frac{\log n}{n}},
\frac{\sqrt\delta}{\sqrt n}
\right\}
\log\mathcal N_\delta
+
\max\{4C_1,16\}(\kappa+B)\delta .
\end{aligned}
\]
Finally, we relabel constants as:
\begin{align*}
    C_1' = 4, \quad C_2' = C_5^2+2C_2^2C_4^2+4C_2C_3, \quad C_3' = 4C_2+2C_5, \quad C_4' = \max\{4C_1,16\}.
\end{align*}
This then yields
\[
\begin{aligned}
\mathbb E\|\hat f-f_0\|_\mu^2
&\le
C_1'\omega(\mathcal F,\{f_0\})
+
C_2'B^2\frac{\log n}{n}\log^2\mathcal N_\delta  \\
&\quad+
C_3'
\max\left\{
B\delta\sqrt{\frac{\log n}{n}},
\frac{\sqrt\delta}{\sqrt n}
\right\}
\log\mathcal N_\delta
+
C_4'(\kappa+B)\delta ,
\end{aligned}
\]
which proves the first claim.

The uniform statement follows by taking the supremum over
\(f_0\in\mathcal G\). The constants in the stochastic bounds are uniform over
\(\mathcal G\) under Assumption~\ref{ass:general_dep_reg_clean}, and
\[
\sup_{f_0\in\mathcal G}\omega(\mathcal F,\{f_0\})
=
\omega(\mathcal F,\mathcal G).
\]
This completes the proof.
\end{proof}

\begin{corollary}[Dominant-term simplification]
\label{cor:dominant_term_simplification_clean}
Under the assumptions of Proposition~\ref{prop:general_risk_bound_clean},
suppose in addition that \(0<\delta\le 1\). Then there exist constants
\(C_1,C_2>0\) such that, for all sufficiently large \(n\),
\begin{align}
\label{eqn:dominant_term_bound}
    \mathbb E\|\hat f-f_0\|_\mu^2
    \le
    C_1
    \max\left\{
        \frac{\log n}{n}\log^2\mathcal N_\delta,\,
        \delta
    \right\}
    +
    C_2\,\omega(\mathcal F,\mathcal G).
\end{align}
\end{corollary}

\begin{proof}
Let
\begin{align*}
    T_1:=\frac{\log n}{n}\log^2\mathcal N_\delta,\qquad
    T_2:=\max\left\{\delta\sqrt{\frac{\log n}{n}},\frac{\sqrt\delta}{\sqrt n}\right\}\log\mathcal N_\delta,\qquad
    T_3:=\delta.
\end{align*}
By Proposition~\ref{prop:abstract_excess_risk_decomp_clean}, there exist constants \(C_1,\dots,C_4>0\) such that for all sufficiently large \(n\)
\begin{align*}
    \|\hat f-f_0\|_\mu^2
    \le
    C_1T_1+C_2T_2+C_3T_3+C_4\,\omega(\mathcal F,\mathcal G).
\end{align*}

We claim that \(T_2\le \max\{T_1,T_3\}\) for all \(n\ge 3\) whenever \(0<\delta\le 1\).

If \(0<\delta\le 1/\log n\), then
\begin{align*}
    T_2=\frac{\sqrt\delta}{\sqrt n}\log\mathcal N_\delta,
\end{align*}
so
\begin{align*}
    T_2^2 = \frac{\delta}{n}\log^2\mathcal N_\delta = \frac{1}{\log n}\,T_1T_3 \le \frac{1}{\log n}\max\{T_1,T_3\}^2.
\end{align*}
Hence
\begin{align*}
    T_2\le \frac{1}{\sqrt{\log n}}\max\{T_1,T_3\}\le \max\{T_1,T_3\},
\end{align*}
since \(\log n\ge 1\) for \(n\ge 3\).

If \(1/\log n\le \delta\le 1\), then
\begin{align*}
    T_2=\delta\sqrt{\frac{\log n}{n}}\log\mathcal N_\delta,
\end{align*}
so
\begin{align*}
    T_2^2 = \delta^2\frac{\log n}{n}\log^2\mathcal N_\delta = \delta\,T_1T_3 \le T_1T_3 \le \max\{T_1,T_3\}^2.
\end{align*}
Thus again \(T_2\le \max\{T_1,T_3\}\).

Therefore, for all sufficiently large \(n\),
\begin{align*}
    C_1T_1+C_2T_2+C_3T_3
    \le
    (C_1+C_2+C_3)\max\{T_1,T_3\}.
\end{align*}
It follows that
\begin{align*}
    \|\hat f-f_0\|_\mu^2 \le (C_1+C_2+C_3)\max\{T_1,T_3\} + C_4\,\omega(\mathcal F,\mathcal G).
\end{align*}
Hence \eqref{eqn:dominant_term_bound} holds with
\begin{align*}
    C_1':=C_1+C_2+C_3,\qquad C_2':=C_4.
\end{align*}

Finally,
\begin{align*}
    T_1\ge T_3
    \quad\Longleftrightarrow\quad
    \frac{\log n}{n}\log^2\mathcal N_\delta \ge \delta
    \quad\Longleftrightarrow\quad
    \log\mathcal N_\delta \ge \sqrt{\frac{n\delta}{\log n}},
\end{align*}
and similarly,
\[
T_1\le T_3
\quad\Longleftrightarrow\quad
\log\mathcal N_\delta \le \sqrt{\frac{n\delta}{\log n}}.
\]
This proves the result.
\end{proof}

\subsubsection{Approximation error and covering number bound for ReLU networks}
What remains to be shown is a bound on the approximation error 
\(
\omega(\mc F,\mc G):=\sup_{g\in\mc G}\inf_{f\in\mc F}\|f-g\|_\mu^2 
\)
and the covering number $\mathcal{N}_{\delta}$ of the function class $\mathcal{F}$.
For this we recall that  $\mathcal{G}=\mathcal{G}(\ell, \mathcal{P}, \mathbf{D})$ defined as in \eqref{comp_class_main} and $\mathcal{F}=\mathcal{F}(M, d_{\max}, B)$ is the class of fully connected neural networks with $M$ hidden layers, at most $d_{\max}$ neurons per layer and weight bounds $B$ defined as in \eqref{eqn:network_class_setup}. Note that, compared to \cite{fan_dqn}, who analyze sparsely connected ReLU networks, we consider a class of simple fully connected networks without any sparsity constraint with flexible width and depth. We use the following approximation result:
\begin{proposition}
\label{prop_approx_reformulated}
Let $\mathcal{G}:=\mathcal{G}(\ell, \mathcal{P}, \mathbf{D})$ be defined as in \eqref{comp_class_main} with 
$$(d^*, s^*) \in \arg\max_{(d, s) \in \mathcal{P}} d/s.$$
Then for all sufficiently large $N, L \in \mathbb{N}$ and constant $c=c(\ell, \mathbf{D}) >0$, there exists a neural network
\begin{align*}
    \phi \in \mathcal{F}\!\left(c\ell\big(L + L \sqrt{\log L / \log N}\big),\; c\max_{i} D_iN,\;N\right)=: \mathcal{F}
\end{align*}
such that
\begin{align*}
    \omega(\mathcal{F}, \mathcal{G}) \lesssim (NL)^{-4s^*/d^*}.
\end{align*}
\end{proposition}
\begin{proof}
    See Theorem 4.4 in \cite{nagler2026optimalneuralnetworkapproximation}
\end{proof}
The corresponding complexity of this network class in terms of its covering number can be bounded as follows: 
\begin{proposition}
\label{prop:optimality-bounded-reformulated}
Let $\delta=(NL)^{-4s^*/d^*}$. For the network in Proposition \ref{prop_approx_reformulated} it holds
\begin{align*}
    \log \mathcal{N}_{\delta} 
    \lesssim (NL)^2 \log(NL).
\end{align*}
\end{proposition}
\begin{proof}
    See Proposition 3.3 in \cite{nagler2026optimalneuralnetworkapproximation}
\end{proof}

Our final covergence rate of our ERM based on a class of fully connected ReLU networks under exponentially $\tau$ mixing training data is the following.

\begin{theorem}[Convergence rate for ReLU network ERMs]
\label{thm:relu_erm_final_rate}
Suppose Assumption~\ref{ass:general_dep_reg_clean} holds. Let
$f_0 \in \mathcal{G}(\ell, \mathcal{P}, \mathbf{D})$ and set
\[
    (d^*,s^*)\in \arg\max_{(d,s)\in\mathcal P} d/s .
\]
Let $\hat{f}$ be an empirical risk minimizer over the class
\[
\mathcal{F}\!\left(c \ell\big(L + L\sqrt{\log L/\log N}\big),\;
    c\max_i D_iN,\;
    N
\right),
\]
with sufficiently large constant $c=c(\ell, \mathbf{D})>0$
defined as in \eqref{eqn:general_erm}. Choose the network size such that
\[
    NL
    \asymp
    \left(\frac{n}{\log^3 n}\right)^{
        \frac{d^*}{4(d^*+s^*)}
    }.
\]
Then, for all sufficiently large \(n\),
\[
    \mathbb E\|\hat f-f_0\|_\mu^2
    \lesssim
    \left(\frac{\log^3 n}{n}\right)^{
        \frac{s^*}{d^*+s^*}
    }.
\]
\end{theorem}

\begin{proof}
By Corollary~\ref{cor:dominant_term_simplification_clean}, for every
\(0<\delta\le 1\),
\[
    \mathbb E\|\hat f-f_0\|_\mu^2
    \le
    C_1
    \max\left\{
        \frac{\log n}{n}\log^2\mathcal N_\delta,\,
        \delta
    \right\}
    +
    C_2\omega(\mathcal F,\mathcal G).
\]
Choose
\[
    \delta=(NL)^{-4s^*/d^*}.
\]
By Proposition~\ref{prop:optimality-bounded-reformulated},
\[
    \log \mathcal N_\delta
    \lesssim
    (NL)^2\log(NL).
\]
Therefore,
\[
    \frac{\log n}{n}\log^2\mathcal N_\delta
    \lesssim
    \frac{\log n}{n}(NL)^4\log^2(NL).
\]
Moreover, by Proposition~\ref{prop_approx_reformulated},
\[
    \omega(\mathcal F,\mathcal G)
    \lesssim
    (NL)^{-4s^*/d^*}.
\]
Substituting these bounds gives
\[
\begin{aligned}
    \mathbb E\|\hat f-f_0\|_\mu^2
    &\lesssim
    \max\left\{
        \frac{\log n}{n}(NL)^4\log^2(NL),\,
        (NL)^{-4s^*/d^*}
    \right\}
    +
    (NL)^{-4s^*/d^*} \\
    &\lesssim
    \max\left\{
        \frac{\log n}{n}(NL)^4\log^2(NL),\,
        (NL)^{-4s^*/d^*}
    \right\}.
\end{aligned}
\]
Now choose
\[
    NL
    \asymp
    \left(\frac{n}{\log^3 n}\right)^{
        \frac{d^*}{4(d^*+s^*)}
    }.
\]
For this choice, \(\log(NL)\asymp \log n\), and the two leading terms are balanced:
\[
    \frac{\log n}{n}(NL)^4\log^2(NL)
    \asymp
    (NL)^{-4s^*/d^*}
    \asymp
    \left(\frac{\log^3 n}{n}\right)^{
        \frac{s^*}{d^*+s^*}
    }.
\]
Consequently,
\[
    \mathbb E\|\hat f-f_0\|_\mu^2
    \lesssim
    \left(\frac{\log^3 n}{n}\right)^{
        \frac{s^*}{d^*+s^*}
    },
\]
which proves the claim.
\end{proof}

\subsection{Application to the Q-learning framework and proof of the main result}

In this subsection, we translate Theorem \ref{thm:relu_erm_final_rate} into the DQN one-step result by identifying the regression target with the Bellman update. For a fixed previous network $Q$, recall that the Bellman response is
\begin{align*}
    Y_i=R_i+\gamma \max_{a \in \mathcal{A}} Q(S_i', a)
\end{align*}
with covariate $X_i=(S_i, A_i)$. Then the regression function is
\begin{align*}
    f_0(s,a)=\mathbb{E}[Y_i|S_i=s, A_i=a] = (TQ)(s,a).
\end{align*}
Therefore we apply Theorem \ref{thm:relu_erm_final_rate} with $f_0=TQ$ and $X_i=(S_i, A_i)$.

\begin{theorem}[One-step Bellman regression under dependent replay data]
\label{Bel_reg}
Suppose Assumption~\ref{ass:general_dep_reg_clean} holds for the state-action process 
\[
    X_t=(S_t,A_t),
\]
and let \(Q\in\mathcal F\) be fixed. Define
\[
    Y_t
    =
    R_t+\gamma\max_{a\in\mathcal A}Q(S_{t+1},a),
\]
and assume that the Bellman target satisfies
\[
    \mc{T}Q \in \mathcal{G}(\ell,\mathcal P,\mathbf D).
\]
Let \(\widehat Q\) be an empirical risk minimizer over
\[
\mathcal{F}\!\left(
c\ell\big(L+L\sqrt{\log L/\log N}\big),
c\max_i D_iN,
N
\right)
\]
defined by
\[
    \widehat Q
    \in
    \arg\min_{f\in\mathcal F}
    \frac1n\sum_{t=1}^n
    \left(
        f(S_t,A_t)-Y_t
    \right)^2 .
\]
Choose
\[
    NL
    \asymp
    \left(\frac{n}{\log^3 n}\right)^{
        \frac{d^*}{4(d^*+s^*)}
    },
    \qquad
    (d^*,s^*)\in\arg\max_{(d,s)\in\mathcal P} d/s .
\]
Then, for all sufficiently large \(n\),
\[
    \mathbb E\|\widehat Q-\mc{T}Q\|_{\sigma}^2
    \lesssim
    |\mathcal A|
    \left(\frac{\log^3 n}{n}\right)^{
        \frac{s^*}{d^*+s^*}
    }.
\]
Equivalently,
\[
    \mathbb E\|\widehat Q-\mc{T}Q\|_{\sigma}
    \lesssim
    |\mathcal A|^{1/2}
    \left(\frac{\log^3 n}{n}\right)^{
        \frac{s^*}{2(d^*+s^*)}
    }.
\]
\end{theorem}

\begin{proof}
    By assumption, \(TQ\in\mathcal G(\ell,\mathcal P,\mathbf D)\). Moreover, since
the state-action process is strictly stationary and satisfies the mixing conditions
in Assumption~\ref{ass:general_dep_reg_clean}, the induced regression sample
\[
    \bigl((S_t,A_t),Y_t\bigr)_{t=1}^n
\]
satisfies the conditions of Theorem~\ref{thm:relu_erm_final_rate}. It remains only to track the action dependence. Use
the action-wise representation
\[
    Q(s,a)=f_a(s),\qquad f_a\in\mathcal F_0 .
\]
By assumption,
\[
    \mathcal F
    =
    \left\{
        Q:\mathcal S\times\mathcal \mathcal{A}\to\mathbb R
        :
        Q(s,a)=f_a(s),\ f_a\in\mathcal F_0
    \right\}.
\]
Hence
\[
    \log \mathcal N
    \left(
        \varepsilon,\mathcal F,\|\cdot\|_\infty
    \right)
    \le
    |\mathcal{A}|\log \mathcal N
    \left(
        \varepsilon,\mathcal F_0,\|\cdot\|_\infty
    \right).
\]
Applying Theorem~\ref{thm:relu_erm_final_rate} with this product entropy bound
therefore yields
\[
    \mathbb E\|\widehat Q-\mc{T}Q\|^2_{\sigma}
    \lesssim
    |\mathcal A|
    \left(\frac{\log^3 n}{n}\right)^{
        \frac{s^*}{d^*+s^*}
    }.
\]
Taking square roots gives
\[
    \mathbb E\|\widehat Q-\mathcal{T}Q\|_{\sigma}
    \lesssim
    |\mathcal A|^{1/2}
    \left(\frac{\log^3 n}{n}\right)^{
        \frac{s^*}{2(d^*+s^*)}
    }.
\]
\end{proof}

Finally, we apply the error-propagation bound from Lemma~\ref{lemma:error_propagation} together with the bound from Theorem~\ref{Bel_reg} to obtain
\[
\begin{aligned}
    \mathbb{E}\|Q^*-Q^{\pi_K}\|_{1,\mu}
    &\le
    C\frac{2\phi_{\sigma}}{(1-\gamma)^2} |\mc{A}|
    \max_{k\in[K]} \mathbb{E}
    \|\widehat Q_k-\mathcal T\widehat Q_{k-1}\|_{\sigma}
    +
    \frac{4\gamma^{K+1}}{(1-\gamma)^2}R_{\max} \\
    &\le
    C'
    \frac{\phi_{\mu,\sigma}}{(1-\gamma)^2} |\mc{A}|^{1 / 2}
    n^{-\alpha^*/2} (\log^3 n)^{\alpha^*/2}
    +
    \frac{4\gamma^{K+1}}{(1-\gamma)^2}R_{\max}.
\end{aligned}
\]
This completes the proof of Theorem~\ref{thm:main_result}.
\subsection{Proof of Corollary~\ref{cor:dqn_sample_complexity}}
\begin{proof}[Proof of Corollary~\ref{cor:dqn_sample_complexity}]
Treating \(\phi_{\mu, \sigma}\) and \(\gamma\) as constants and ignoring the polylogarithmic term, Theorem~\ref{thm:main_result} implies
    \begin{align*}
        \norm{Q^* - Q^{\pi_K}}_{1, \mu} \le C' n^{-\nicefrac{\alpha^*}{2}},
    \end{align*}
    where the lower-order algorithmic error term is absorbed into the constant \(C'\) along with the constants \(C, \phi_{\mu, \sigma}\) and \((1 - \gamma)^{-2}\) and the polylogarithmic term.
    Setting \(\norm{Q^* - Q^{\pi_K}}_{1, \mu} \le \varepsilon\) and solving for \(n\) completes the proof.
\end{proof}

\subsection{Auxilary results}

\begin{lemma}[Sub-Gaussian random variables have bounded variance]
\label{lemma:sub_gauss_bounded_variance_clean}
Let \(X \sim \mc{S}\mc{G}(\kappa)\) be a centered sub-Gaussian random variable with parameter $\kappa > 0$. Then 
\begin{align*}
    \mathrm{Var}(X) = \mbb{E}[X^2] \le \kappa^2.
\end{align*}
\end{lemma}
\begin{proof}
    By the moment generating function (mgf) characterization of sub-Gaussian random variables (see e.g., \citet{li_subgaussian_properties}), we have that for any \(\lambda \in \mbb{R}\)
\begin{align*}
    \mbb{E}[\exp(\lambda X)] \le \exp \left(
    \frac12 \kappa^2 \lambda^2
    \right).
\end{align*}
Differentiating the RHS of the mgf above twice with respect to \(\lambda\) and setting \(\lambda = 0\) yields
\begin{align*}
    \mbb{E}[X^2] \le \kappa^2.
\end{align*}
\end{proof}

\begin{lemma}[Effective sample size for exponential \(\tau\)-mixing sequences]
\label{lemma:effective_sample_size_exp_decay_clean}
Let \(\{X_i\}_{i=1}^n\) be an absolutely bounded sequence \(\abs{X_i} \le M\) with \(\tau\)-mixing coefficients satisfying
\begin{align*}
\tau(k) \le C_0 \exp(-c k), \quad k \ge 1,
\end{align*}
for some constants \(C_0, c > 0\). Define the effective sample size
\begin{align*}
n_\mathrm{eff} := \max\left\{1 \le m \le n : 
\tau\Big(\Big\lfloor \frac{n}{m} \Big\rfloor\Big) 
\le \frac{C_1}{m} \lor \frac{C_2}{\sqrt{m}} 
\right\},
\end{align*}
for fixed constants \(C_1, C_2>0\), which corresponds to the number of approximately independent blocks \(B = \floor{\frac{n}{m}}\). Then there exist constants 
\(0 < c_- \le c_+ < \infty\), depending only on \(C_0, C_1, C_2, c\), 
and \(n_0 \in \mathbb{N}\), such that for all \(n \ge n_0\),
\begin{align*}
c_- \frac{n}{\log n} 
\le n_\mathrm{eff} 
\le c_+ \frac{n}{\log n}.
\end{align*}
\end{lemma}

\begin{proof}

We first establish an upper bound.

Let \(1 \le m \le n\) and set 
\begin{align*}
    B = \left\lfloor \frac{n}{m} \right\rfloor,
\end{align*}
so that
\begin{align*}
\tau(B)
\le \frac{C_1}{m} 
\quad \text{or} \quad 
\tau(B)
\le \frac{C_2}{\sqrt{m}}.
\end{align*}

Since \(\lfloor x \rfloor \ge x - 1\) and by the exponential decay, we have
\begin{align*}
    \tau(B) \le C_0 \exp(-cB) \le C_0 \exp \left(-c  \cdot \frac{n}{m-1} \right) = C_0 \exp(c) \exp \left(\frac{-cn}{m} \right)
\end{align*}

\medskip
\noindent
\textit{Case I:}
Suppose
\begin{align*}
C_0 \exp(c) \exp(-c n/m) \le \frac{C_1}{m}.
\end{align*}
Taking logarithms yields
\begin{align*}
&- c \frac{n}{m} + \log(C_0 \exp(c))
\le \log C_1 - \log m \\
& \iff c n \ge m(\log m + K),
\end{align*}
where \(K := \log(C_0 \exp(c) / C_1)\).
Note that 
\begin{align*}
    \forall \alpha > 0 \;\exists\, n_0 \in \mbb{N}_+ \; s.t. \;\alpha \frac{n}{\log n} \le n \;\forall n \ge n_0.
\end{align*}
We now set
\begin{align*}
m = \alpha \frac{n}{\log n}, \quad \alpha > 0,
\end{align*}
and note that \(m \lesssim n\).
Then
\begin{align*}
\log m
= \log n - \log \log n + \log \alpha.
\end{align*}
Substituting into the inequality gives
\begin{align*}
c n 
\ge 
\alpha \frac{n}{\log n}
\left(
\log n - \log \log n + \log \alpha + K
\right).
\end{align*}
Dividing by \(n\),
\begin{align*}
c 
\ge 
\alpha
\left(
1 - \frac{\log \log n}{\log n}
+ \frac{\log \alpha + K}{\log n}
\right).
\end{align*}
Now assume that \(\alpha \le \sqrt{n} \exp(-K)\), then for sufficiently large \(n\),
\begin{align*}
1 - \frac{\log \log n}{\log n}
+ \frac{\log \alpha + K}{\log n}
\ge \frac12.
\end{align*}
Hence, for large \(n\),
\begin{align*}
c \ge \frac{\alpha}{2},
\end{align*}
which implies \(\alpha \le 2c\). Therefore,
\begin{align*}
m \le 2c \frac{n}{\log n}
\end{align*}
for all sufficiently large \(n\).

\medskip
\noindent
\textit{Case II:}
Suppose instead
\begin{align*}
C_0 \exp(c) \exp(-cn/m) \le \frac{C_2}{\sqrt{m}}.
\end{align*}
Taking logarithms gives
\begin{align*}
c n \ge \frac{m}{2} \log m + K' m,
\end{align*}
where \(K' := \log(C_0 \exp(c) / C_2)\).

Substituting \(m = \alpha n / \log n\) and proceeding exactly as in Case I above yields, for sufficiently large \(n\),
\begin{align*}
c \ge \frac{\alpha}{4},
\end{align*}
and hence \(\alpha \le 4c\). Therefore,
\begin{align*}
m \le 4c \frac{n}{\log n}
\end{align*}
for all sufficiently large \(n\).

\medskip

Combining the two branches yields
\begin{align*}
n_\mathrm{eff} \le c_+ \frac{n}{\log n}
\end{align*}
for some constant \(c_+ > 0\).

We now establish a lower bound.
From Lemma~1 of \citet{liu_generalisation_tau_mixing_2025}, we have that
\begin{align*}
    n_\mathrm{eff} \ge \min \left\{\frac{n}{2} \cdot c \cdot \left[1 \vee \log \left(\frac{C_0 \cdot c \cdot n}{M}\right)\right]^{-1}, n\right\}.
\end{align*}
Note that, \(\forall \alpha > 0 \; \exists\, n_0 \in \mbb{N}_+ \; s.t. \; \forall \,n \ge n_0\)
\begin{align*}
    \log(\alpha n) \ge 1.
\end{align*}
Therefore, we assume \(n\) is sufficiently large such that
\begin{align*}
    1 \vee \log \left(\frac{C_0 \cdot c \cdot n}{M}\right) = \log \left(\frac{C_0 \cdot c \cdot n}{M}\right).
\end{align*}
Now consider,
\begin{align*}
    n_\mathrm{eff} &\ge \frac{cn}{2\log \left( \frac{C_0 cn}{M}\right)} = \frac{cn}{2\left[\log n + K\right]},
\end{align*}
where we define \(K := \log \left(\frac{C_0 c}{M}\right)\). Now, notice that \(\forall K > 0 \; \exists \,C' > 0 \; s.t.\; C' \log n \ge \log n + K\). We therefore have that
\begin{align*}
    n_\mathrm{eff} &\ge \frac{cn}{2\left[\log n + K\right]} \ge  \frac{cn}{2C' \log n}.
\end{align*}
Taking \(c_{-} = \frac{c}{2C'}\) we conclude that, for all sufficiently large \(n\),
\begin{align*}
n_\mathrm{eff} \ge c_- \frac{n}{\log n},  
\end{align*}
thereby establishing the lower bound.

\end{proof}

\begin{corollary}[Comparability of effective sample sizes for two sequences]
\label{cor:effective_sample_size_comparison}
Let \(\{X_i\}\) and \(\{Y_i\}\) be absolutely bounded sequences with exponential \(\tau\)-mixing coefficients
\begin{align*}
\tau_X(k) \le C_1 e^{-c k}, \quad \tau_Y(k) \le C_2 e^{-c k}, \quad k \ge 1.
\end{align*}
Let \(n_\mathrm{eff}^{(X)}\) and \(n_\mathrm{eff}^{(Y)}\) denote their respective effective sample sizes. Then there exist constants \(0 < c_-, c_+ < \infty\) depending only on \(C_1, C_2, c\) such that
\begin{align*}
n_\mathrm{eff}^{(X)} \asymp n_\mathrm{eff}^{(Y)}.
\end{align*}
\end{corollary}

\begin{proof}
By Lemma~\ref{lemma:effective_sample_size_exp_decay_clean}, \(n_\mathrm{eff}^{(X)} \asymp n/\log n\) and \(n_\mathrm{eff}^{(Y)} \asymp n/\log n\). Hence,
\begin{align*}
   n_\mathrm{eff}^{(X)} \asymp n_\mathrm{eff}^{(Y)}.
\end{align*}
\end{proof}

\begin{lemma}
\label{lemma:lipschitz_holder_combo_lipschitz_clean}
Let \(\mathcal{X} \subseteq \mathbb{R}^d\) be compact. Suppose 
\(f: \mathcal{X} \to \mathbb{R}\) is \(H_f\)-Lipschitz with respect to \(\|\cdot\|_1\) and satisfies \(|f(x)| \le M\) for all \(x \in \mathcal{X}\), and \(g: \mathcal{X} \to \mathbb{R}\) is \(\beta_g\)-H\"older with constant \(H_g\) with respect to \(\|\cdot\|_1\) and \(|g(x)| \le M\) for all \(x \in \mathcal{X}\).  
Then the function \(h(x) = \big(f(x) - g(x)\big)^2\) is Lipschitz on \(\mathcal{X}\).
\end{lemma}

\begin{proof}
    Let \(x, x' \in \mc{X}\). Using the definition of \(h\) and the fact that \((a^2 - b^2) = (a-b)(a+b)\), we have that
    \begin{align*}
        |h(x) - h(x')| &= \left|\left[f(x) - g(x) \right]^2 - \left[f(x') - g(x') \right]^2\right| \\
        &= \left| \{f(x) - g(x) + f(x') - g(x')\}\cdot \{ f(x) - g(x) - [f(x') - g(x')]\} \right|\\
        &\le \{\abs{f(x) - g(x)} + \abs{f(x') - g(x')}\} \cdot \{\abs{\left[f(x) - f(x') \right] - \left[g(x) - g(x') \right]}\}
    \end{align*}
where the inequality follows from the fact that \(\abs{a \cdot b} = \abs{a} \cdot \abs{b}\) and since \(\abs{a + b} \le \abs{a} + \abs{b}\) for any two real numbers \(a,b \in \mbb{R}\).

The first term in the expression above can be bounded as follows
    \begin{align*}
        \abs{f(x) - g(x)} + \abs{f(x') - g(x')} \le \abs{f(x)} + \abs{g(x)} + \abs{f(x')} + \abs{g(x')} \le 4M.
    \end{align*}
    Therefore, we have that
    \begin{align*}
        \abs{h(x) - h(x')} &\le 4M \cdot \abs{\left[f(x) - f(x')\right] - \left[ g(x) - g(x')\right]} \\
        &\le 4M \cdot \left\{\abs{f(x) - f(x')} + \abs{g(x) - g(x')}\right\} \quad (\text{since } \abs{a - b} \le \abs{a} + \abs{b} ) \\
        &\le 4M \cdot \left\{H_f \left\lVert x - x' \right\rVert_1 + H_g \left \lVert x - x' \right \rVert_1^{\beta_g} \right\}
    \end{align*}
Where the last inequality follows since \(f\) is $H_f$-Lipschitz and \(g\) is \(\beta_g\)-H\"{o}lder with constant \(H_g\).

Define the diameter of \(\mc{X}\) as:
    \begin{align*}
        D_\mc{X} = \underset{x, x' \in \mc{X}}{\sup} \bignorm{x - x'}_1.
    \end{align*}
    Note that \(D_\mc{X} < \infty\) since \(\mc{X}\) is compact. Now let \(a \in [0, D_\mc{X}]\) and let \(k > 0\). Clearly, \(a^k \le \max\{1, D_\mc{X}^{k -1}\} \cdot a\). Now we use this fact with \(a = \bignorm{x - x'}_1\) and let \(k = \beta_g\) to obtain
    \begin{align*}
        &\bignorm{x - x'}_1^{\beta_g} \le \max\{1, D_\mc{X}^{\beta_g - 1}\} \bignorm{x - x'}_1 \\
        &\implies \abs{h(x) - h(x')} \le 4M \cdot \left(H_f + H_g \cdot \max\left\{1, D_{\mc{X}}^{\beta_g - 1}\right\}\right) \bignorm{x - x'}_1
    \end{align*}
    Therefore, \(h\) is \(4\,M\cdot\left(H_f + H_g \cdot \max\left\{1, D_{\mc{X}}^{\beta_g - 1}\right\}\right)\)-Lipschitz. 
\end{proof}

\begin{lemma}
\label{lemma:tau_mixing_pairs_clean}
Let \(\{X_i\}_{i\in\mathbb Z}\) and \(\{\tilde X_i\}_{i\in\mathbb Z}\) be two independent \(\tau\)–mixing sequences defined on a common domain \(\mc{X} \subseteq \mbb{R}^d\) with coefficients \(\tau_X(k)\) and \(\tau_{\tilde X}(k)\), respectively.  
Assume that \(X_i\) and \(\tilde X_i\) are integrable.  
Define the bivariate process
\begin{align*}
    Z_i=(X_i,\tilde X_i).
\end{align*}
Then \(\{Z_i\}_{i \in \mbb{Z}}\) is \(\tau\)–mixing and its coefficients satisfy
\begin{align*}
    \tau_Z(k)\le \tau_X(k)+\tau_{\tilde X}(k),\qquad k\ge 1.
\end{align*}
\end{lemma}

\begin{proof}
Let \(p\in\mathbb Z\) and set
\begin{equation*}
        \begin{aligned}
            &\mathcal A_X^p=\sigma(X_t:t\le p),\quad 
            \mathcal A_{\tilde X}^p=\sigma(\tilde X_t:t\le p) \\
            &\mathcal A^p=\sigma(Z_t:t\le p)
            =\mathcal A_X^p\lor\mathcal A_{\tilde X}^p = \sigma(\mc{A}_X^p \cup \mc{A}_{\tilde{X}}^p).          
        \end{aligned}
\end{equation*}

Since the two processes are independent, the \(\sigma\)–algebras \(\mathcal A_X^p\) and \(\mathcal A_{\tilde X}^p\) are independent.

We recall the coupling characterization of the \(\tau\)–coefficient
\cite{dedecker_coupling_2004}: for any integrable random variable \(Y\),
\begin{align*}
    \tau(\mathcal G,Y)
    =\inf\Bigl\{\mathbb E \left[\norm{Y - Y^*}_1 \right] : 
    Y^*\overset{d}=Y,\; Y^*\ \text{independent of }\mathcal G\Bigr\},
\end{align*}
where \(\mc{G}\) is a \(\sigma\)-algebra and the infimum is taken over all random variables \(Y^*\) with the same marginal distribution as \(Y\) and independent of \(\mc{G}\).

Apply this with \(\mathcal G=\mathcal A_X^p\) and \(Y=X_{p+k}\).  
There exists a coupled random variable \(X_{p+k}^*\), independent of \(\mathcal A_X^p\) such that \citep[Lemma~3]{dedecker_coupling_2004}
\begin{align*}
    \mathbb E \left[ \norm{X_{p+k}-X_{p+k}^*}_1\right]\le\tau_X(k).
\end{align*}
Likewise, there exists \(\tilde X_{p+k}^*\), independent of \(\mathcal A_{\tilde X}\), such that
\begin{align*}
    \mathbb E \left[ \norm{\tilde X_{p+k}-\tilde X_{p+k}^*}_1\right]\le\tau_{\tilde X}(k).
\end{align*}

By enlarging the underlying probability space if necessary, \footnote{We do not consider measurability issues as this is besides the scope of this paper. We simply note that these issues can be addressed by enlarging the underlying probability space and refer readers to e.g., \citet{yu_rates} for more information.} we may construct
\(X_{p+k}^*\) and \(\tilde X_{p+k}^*\) jointly so that they are independent of each
other.  
Since \(\mathcal A_X^p\) and \(\mathcal A_{\tilde X}^p\) are independent, this implies
\begin{align*}
    (X_{p+k}^*,\tilde X_{p+k}^*)\ \text{is independent of}\ 
    \mathcal A_X^p\lor\mathcal A_{\tilde X}^p
    =\mathcal A^p.
\end{align*}
Because the two original processes are independent, we have that their joint law factors into the product of their marginal laws, \(\mc{L}((X_t \tilde{X}_t)) = \mc{L}(X_t) \times \mc{L}(\tilde{X}_t) = \mc{L}(X_t^*) \times \mc{L}(\tilde{X}_t^*)\) by construction. Therefore, we also have that
\begin{align*}
    (X_{p+k}^*,\tilde X_{p+k}^*)\overset{d}=(X_{p+k},\tilde X_{p+k}).
\end{align*}

Define
\begin{align*}
    Z_{p+k}^*=(X_{p+k}^*,\tilde X_{p+k}^*).
\end{align*}
Using the \(\ell^1\)–norm on the product space,
\begin{align*}
    \|(x,\tilde x)\|_1=\|x\|_1+\|\tilde x\|_1,
\end{align*}
and linearity of expectation, we obtain
\begin{align*}
    \mathbb E \left[ \norm{Z_{p+k}-Z_{p+k}^*}_1 \right]
    =\mathbb E\left[ \norm{X_{p+k}-X_{p+k}^*}_1\right]
    +\mathbb E\left[\norm{\tilde X_{p+k}-\tilde X_{p+k}^*}_1\right]
    \le \tau_X(k)+\tau_{\tilde X}(k).
\end{align*}
By the coupling characterization of \(\tau\),
\begin{align*}
    \tau(\mathcal A^p,Z_{p+k})
    \le \tau_X(k)+\tau_{\tilde X}(k).
\end{align*}
Taking the supremum over all \(p\) yields
\begin{align*}
    \tau_Z(k)\le \tau_X(k)+\tau_{\tilde X}(k),
\end{align*}
which completes the proof.
\end{proof}

\begin{lemma}[Composition of Lipschitz Functions]
\label{lemma:lipschitz_compositions}
Let \(f: \mbb{R}^{d_0} \rightarrow \mbb{R}^{d_1}\) be \(H_f\)-Lipschitz and let \(g: \mbb{R}^{d_1} \rightarrow \mbb{R}^{d_2}\) be \(H_g\)-Lipschitz. Then \(h = g \circ f\) is \((H_f \cdot H_g)\)-Lipschitz.
\end{lemma}
\begin{proof}
    Let \(p \in [1, \infty]\), and let \(x, x' \in \mbb{R}^{d_0}\). 
    \begin{align*}
        \norm{h(x) - h(x')}_p &= \norm{g[f(x)] - g[f(x')]}_p \le H_g \cdot \norm{f(x) - f(x')}_p \le H_g \cdot H_f \cdot \norm{x - x'}_p,
    \end{align*}
    which completes the proof.
\end{proof}

\begin{lemma}
\label{lemma:tau_mixing_lipschitz_clean}
Let \((\Omega,\mc{X},\mu)\) be a probability space, let \(X \in L_1(\mc{X})\) be an integrable random variable defined on \((\Omega,\mc{X},\mu)\), and let \(\mc{A}\subset\mc{X}\) be a \(\sigma\)-algebra. Let \(f:\mc{X}\to\mbb{R}\) be an \(H\)-Lipschitz function.
Then
\begin{align*}
\tau(\mathcal{A}, f(X)) \le H\,\tau(\mathcal{A}, X).
\end{align*}
\end{lemma}

\begin{proof}
From the definition of the \(\tau\)-mixing coefficients in~\eqref{eqn:tau_general_setup}
\begin{align*}
\tau(\mathcal{A}, X)
&=
\left\|
\sup_{h\in\Lambda_1}
\left|
\int h(x)\,\mu_{X\mid\mathcal{A}}(dx)
-
\int h(x)\,\mu_X(dx)
\right|
\right\|_{L^1},
\end{align*}
where \(\Lambda_1\) denotes the collection of all \(1\)-Lipschitz functions
\(h:\mathbb{R}\to\mathbb{R}\).

We first rewrite \(\tau(\mathcal{A}, f(X))\) in terms of the distribution of \(X\).
By definition,
\begin{align*}
\tau(\mathcal{A}, f(X))
&=
\left\|
\sup_{h\in\Lambda_1}
\left|
\int h(y)\,\mu_{f(X)\mid\mathcal{A}}(dy)
-
\int h(y)\,\mu_{f(X)}(dy)
\right|
\right\|_{L^1}.
\end{align*}
Since \(\mu_{f(X)\mid\mathcal{A}}\) and \(\mu_{f(X)}\) are the push-forward
measures of \(\mu_{X\mid\mathcal{A}}\) and \(\mu_X\) under \(f\), respectively,
we have for every measurable $h$,
\begin{align*}
\int h(y)\,\mu_{f(X)\mid\mathcal{A}}(dy)
&=
\int h \circ f(x)\,\mu_{X\mid\mathcal{A}}(dx), \\
\int h(y)\,\mu_{f(X)}(dy)
&=
\int h \circ f(x)\,\mu_X(dx).
\end{align*}
Therefore,
\begin{align*}
\tau(\mathcal{A}, f(X))
&=
\left\|
\sup_{h\in\Lambda_1}
\left|
\int h\circ f(x)\,\mu_{X\mid\mathcal{A}}(dx)
-
\int h\circ f(x)\,\mu_X(dx)
\right|
\right\|_{L^1}.
\end{align*}

Now let \(h\in\Lambda_1\). Since \(f\) is \(H\)-Lipschitz and \(h\) is
\(1\)-Lipschitz, the composition \(h\circ f\) is \(H\)-Lipschitz by Lemma~\ref{lemma:lipschitz_compositions}:
for all \(x,y\in\mathbb{R}\),
\begin{align*}
|h(f(x)) - h(f(y))| \le |f(x)-f(y)| \le H|x-y|.
\end{align*}
Hence the function
\begin{align*}
g := \frac{1}{H} \, h\circ f
\end{align*}
belongs to $\Lambda_1$. Moreover,
\begin{align*}
h\circ f = H g.
\end{align*}
Thus,
\begin{align*}
\left|
\int h\circ f \, d\mu_{X\mid\mathcal{A}}
-
\int h\circ f \, d\mu_X
\right|
&=
H
\left|
\int g \, d\mu_{X\mid\mathcal{A}}
-
\int g \, d\mu_X
\right|.
\end{align*}
Taking the supremum over \(h\in\Lambda_1\) and using that
\(g\in\Lambda_1\), we obtain
\begin{align*}
\sup_{h\in\Lambda_1}
\left|
\int h\circ f \, d\mu_{X\mid\mathcal{A}}
-
\int h\circ f \, d\mu_X
\right|
&\le
H
\sup_{g\in\Lambda_1}
\left|
\int g \, d\mu_{X\mid\mathcal{A}}
-
\int g \, d\mu_X
\right|.
\end{align*}
Finally, taking \(L^1\)-norms on both sides we obtain
\begin{align*}
\tau(\mathcal{A}, f(X))
&\le
H
\left\|
\sup_{g\in\Lambda_1}
\left|
\int g \, d\mu_{X\mid\mathcal{A}}
-
\int g \, d\mu_X
\right|
\right\|_{1} \\
&=
H\,\tau(\mathcal{A}, X),
\end{align*}
which completes the proof.
\end{proof}

\begin{corollary}
\label{cor:tau_mixing_lipschitz_clean}
Let $\{X_t\}_{t\ge1}$ be a $\tau$-mixing sequence with coefficients $\tau_X(k)$ defined on a domain $\mc{X} \subset \mbb{R}$, and let $f:\mc{X}\to\mbb{R}$ be $H$-Lipschitz. Then the transformed sequence $\{f(X_t)\}_{t\ge1}$ is $\tau$-mixing with coefficients satisfying
\begin{align*}
\tau_f(k) \le H\,\tau_X(k)
\qquad \text{for all } k\ge 1.
\end{align*}
\end{corollary}
\begin{proof}
    Fix a time \(p \in \mbb{Z}\) and let \(\mc{A}_p = \sigma(X_t: t \le p)\) be the \(\sigma\)-algebra generated by \(\{X_t\}_{t=1}^p\).
    By Lemma~\ref{lemma:tau_mixing_lipschitz_clean} we have that
    \begin{align*}
        \tau(\mc{A}_p, f(X_{p+k})) \le H\tau(\mc{A}_p, X_{p+k}),
    \end{align*}
    where \(k \in \mbb{Z}\).
    Taking the supremum over \(p\) as in~\eqref{eqn:tau_lag_setup} we find that
    \begin{align*}
        \tau_{f}(k) \le H\tau_X(k),
    \end{align*}
    where we denote by \(\tau_f(k)\) the \(\tau\)-mixing coefficients of \(\{f(X_t)\}_{t\ge1}\) at a lag \(k \in \mbb{Z}\).
\end{proof}

\section{Dependence analysis}
\label{app:dependence_analysis}
In this section, we characterize the induced dependence structure of replay buffers and minibatches. This serves as theoretical justification for our assumption of exponentially \(\tau\)-mixing minibatches.

Since replay buffers and minibatches are finite sequences, we introduce truncated mixing coefficients.
Let
\begin{align*}
    W_{1:N} := (W_1,\dots,W_N)
\end{align*}
be a finite sequence of random variables of length \(N\), and define
\begin{align*}
    \mc F_t^W := \sigma(W_1,\dots,W_t), \qquad 1 \le t \le N.
\end{align*}
For \(1 \le k \le N-1\), define the truncated mixing coefficients
\begin{align}
\label{eqn:truncated_tau_mixing_coefficient}
    \tau_W^N(k)
    :=
    \sup_{1 \le t \le N-k}
    \tau(\mc F_t^W, W_{t+k}),
\\
\label{eqn:truncated_beta_mixing_coefficient}
    \beta_W^N(k)
    :=
    \sup_{1 \le t \le N-k}
    \beta(\mc F_t^W, W_{t+k}).
\end{align}
where \(\beta(\cdot, \cdot)\) and \(\tau(\cdot, \cdot)\) are defined in~\eqref{eqn:beta_mixing_coefficients} and~\eqref{eqn:tau_general_setup}.
To streamline the presentation, let \(\zeta\) denote either \(\tau\) or \(\beta\). Then
\begin{align}
\label{eqn:truncated_zeta_mixing_coefficient}
    \zeta_W^N(k)
    :=
    \sup_{1 \le t \le N-k}
    \zeta(\mc F_t^W, W_{t+k}).
\end{align}
We say that \(W_{1:N}\) has exponentially decaying truncated \(\zeta\)-mixing coefficients if there exist constants \(c_0,c_1 > 0\) such that
\begin{align*}
    \zeta_W^N(k) \le c_0 e^{-c_1 k},
    \qquad 1 \le k \le N-1.
\end{align*}
We also note that, when \(W_{1:N}\) is a contiguous subsequence of an infinite process, these truncated coefficients are the finite-sample analogues of the usual one-sided mixing coefficients.

We first state an obvious result that is crucial in all proofs in this section: that \(\zeta\)-mixing coefficients are monotone with respect to the \(\sigma\)-algebra.
\begin{lemma}[Monotonicity in the \(\sigma\)-algebra]
\label{lemma:zeta_monotone_sigma_algebra}
Let \(X\) be a random variable, and let \(\mathcal A\) and \(\mathcal B\) be \(\sigma\)-algebras such that
\begin{align*}
    \mathcal A \subseteq \mathcal B.
\end{align*}
Then, for \(\zeta \in \{\tau,\beta\}\),
\begin{align*}
    \zeta(\mathcal A, X) \le \zeta(\mathcal B, X).
\end{align*}
\end{lemma}
\begin{proof}
The claim is immediate from the definitions of \(\tau\) and \(\beta\): enlarging the \(\sigma\)-algebra can only increase the amount of dependence measured between the past and \(X\).
\end{proof}

\subsection{Dependence strucure of the replay buffer}

We first show that a replay buffer inherits exponential mixing from the underlying process.

\begin{lemma}[Exponentially \(\zeta\)-mixing replay buffer]
\label{lemma:exp_tau_mixing_buffer}
Let \(t_0 \ge 1\), and suppose the process \((X_t)_{t\ge 1}\) is exponentially \(\zeta\)-mixing in the sense that
\begin{align*}
    \zeta_X(k)
    :=
    \sup_{t \ge 1} \zeta(\mc A_t, X_{t+k})
    \le c_0 e^{-c_1 k},
    \qquad k \ge 1,
\end{align*}
where \(\mc A_t := \sigma(X_s: s \le t)\). Define the replay buffer of size \(M \in \mathbb{N}\) as 
\begin{align}
\label{eqn:buffer}
    \mc M = (Z_1,\dots,Z_M),
    \qquad
    Z_m := X_{t_0+m-1}, \quad m=1,2,\dots,M.
\end{align}
Then \(\mc M\) is also exponentially \(\zeta\)-mixing, with
\begin{align*}
    \zeta_{\mc M}^M(k)
    \le
    \zeta_X(k)
    \le
    c_0 e^{-c_1 k},
    \qquad 1 \le k \le M-1.
\end{align*}
\end{lemma}

\begin{proof}
For \(1 \le t \le M\), let
\begin{align*}
    \mc F_t^{\mc M}
    :=
    \sigma(Z_1,\dots,Z_t).
\end{align*}
Since \((Z_m)_{m=1}^M\) is a contiguous subsequence of \((X_t)\), we have
\begin{align*}
    \mc F_t^{\mc M}
    =
    \sigma(X_{t_0},\dots,X_{t_0+t-1})
    \subseteq
    \sigma(X_s : s \le t_0+t-1)
    =
    \mc A_{t_0+t-1}.
\end{align*}
Therefore, for every \(1 \le t \le M-k\),
\begin{align*}
    \zeta(\mc F_t^{\mc M}, Z_{t+k})
    &=
    \zeta(\mc F_t^{\mc M}, X_{t_0+t+k-1})
    \\
    &\le
    \zeta(\mc A_{t_0+t-1}, X_{t_0+t+k-1})
    \\
    &\le
    \zeta_X(k),
\end{align*}
where the first inequality follows from Lemma~\ref{lemma:zeta_monotone_sigma_algebra}. Taking the supremum over \(1 \le t \le M-k\) gives
\begin{align*}
    \zeta_{\mc M}^M(k) \le \zeta_X(k),
    \qquad 1 \le k \le M-1.
\end{align*}
The exponential bound follows immediately.
\end{proof}

This shows that a replay buffer does not create additional long-range dependence: it inherits the same exponential dependence decay as the underlying process, up to truncation at length \(M\).

\subsection{From \texorpdfstring{\(\beta\)}{beta}-mixing to \texorpdfstring{\(\tau\)}{tau}-mixing}

In this section, we show that if the underlying process is exponentially \(\beta\)-mixing, then it is also exponentially \(\tau\)-mixing.

\begin{lemma}[Bounded-metric \(\beta\)-to-\(\tau\) inequality]
\label{lem:beta_to_tau}
Let \((\mathcal X,d_{\mathcal X})\) be a bounded Polish metric space with
diameter
\[
    D_{\mathcal X}
    :=
    \sup_{x,x'\in\mathcal X} d_{\mathcal X}(x,x')
    <
    \infty .
\]
Let \((X_t)_{t\ge 1}\) be a \(\mathcal X\)-valued process.
Then, for every \(k\ge1\),
\[
    \tau_X(k)
    \le
    C_{\beta\tau}D_{\mathcal X}\,\beta_X(k),
\]
where \(C_{\beta\tau}\) is a universal constant depending only on the
normalization of \(\beta\) and total variation.
\end{lemma}

\begin{proof}
Fix \(t\) and set
\[
    \mathcal A_t:=\sigma(X_s:s\le t),
    \qquad
    \mathcal B_{t,k}:=\sigma(X_s:s\ge t+k).
\]
Let \(Z\) be an \(\mathcal X\)-valued random variable measurable with respect
to \(\mathcal B_{t,k}\). For every realization of \(\mathcal A_t\), the
quantity
\[
    \sup_{h\in\Lambda_1(\mathbb R^d)}
    \left|
    \int h\,d\mu_{Z\mid\mathcal A_t}
    -
    \int h\,d\mu_Z
    \right|
\]
is the \(1\)-Wasserstein distance between
\(\mu_{Z\mid\mathcal A_t}\) and \(\mu_Z\). Since both measures are supported
on \(\mathcal X\), the bounded-metric inequality gives
\[
    W_1(\mu_{Z\mid\mathcal A_t},\mu_Z)
    \le
    D_\mc{X}
    \|\mu_{Z\mid\mathcal A_t}-\mu_Z\|_{\mathrm{TV}}.
\]
With the total-variation convention
\[
    \|\mu-\nu\|_{\mathrm{TV}}
    =
    2\sup_{B}|\mu(B)-\nu(B)|,
\]
we obtain
\[
\begin{aligned}
    \tau_Z(k)
    &\le
    2\,D_\mc{X}\,
    \mathbb E\left[
    \sup_B
    \left|
        \mu_{Z\mid\mathcal A_t}(B)-\mu_Z(B)
    \right|
    \right]  \\
    &\le
    2\,D_\mc{X}\,
    \mathbb E\left[
    \sup_{A\in\mathcal B_{t,k}}
    \left|
        \mathbb P(A\mid\mathcal A_t)-\mathbb P(A)
    \right|
    \right].
\end{aligned}
\]
Taking the supremum over \(t\) gives
\[
    \tau_X(k)
    \le
    C_{\beta \tau}\,D_\mc{X}\,\beta_X(k).
\]
With our total variation convention we have \(B_{\beta \tau} = 2\).
\end{proof}

\begin{lemma}[Exponential \(\tau\)-mixing from exponential \(\beta\)-mixing]
\label{lemma:tau_mixing_from_beta_mixing}
Let \((\mathcal Z,d_{\mathcal Z})\) be a bounded Polish metric space with
diameter
\[
    D_{\mathcal Z}
    :=
    \sup_{z,z'\in\mathcal Z} d_{\mathcal Z}(z,z')
    <
    \infty .
\]
Let \((Z_t)_{t\ge 1}\) be an exponentially \(\beta\)-mixing \(\mathcal Z\)-valued process satisfying
\[
    \beta_Z(k) =
    \sup_{t\ge 1}
    \beta\bigl(\sigma(Z_s:s\le t),Z_{t+k}\bigr)
    \le
    C_\beta e^{-c_\beta k},
    \qquad k\ge 1 .
\]
Then \((Z_t)_{t\ge 1}\) is exponentially \(\tau\)-mixing. More
precisely, for every \(k\ge 1\),
\[
    \tau_Z(k)
    :=
    \sup_{t\ge 1}
    \tau\bigl(\sigma(Z_s:s\le t),Z_{t+k}\bigr)
    \le
    C_{\beta\tau}D_{\mathcal Z} C_\beta e^{-c_\beta k},
\]
where \(C_{\beta\tau}\) is a universal constant depending only on the
normalization of \(\beta\) and total variation. In particular, the
exponential rate is unchanged, and only the prefactor changes.
\end{lemma}

\begin{proof}
By the bounded-metric comparison between \(\beta\)- and \(\tau\)-mixing
coefficients in Lemma~\ref{lem:beta_to_tau}, for any sub-\(\sigma\)-algebra \(\mathcal A\) and any
\(\mathcal Z\)-valued random variable \(V\),
\[
    \tau(\mathcal A,V)
    \le
    C_{\beta\tau}D_{\mathcal Z}\,
    \beta(\mathcal A,V).
\]
Applying this inequality with
\[
    \mathcal A=\sigma(Z_s:s\le t),
    \qquad
    V=Z_{t+k},
\]
gives, for every \(t\ge 1\) and \(k\ge 1\),
\[
    \tau\bigl(\sigma(Z_s:s\le t),Z_{t+k}\bigr)
    \le
    C_{\beta\tau}D_{\mathcal Z}\,
    \beta\bigl(\sigma(Z_s:s\le t),Z_{t+k}\bigr).
\]
Taking the supremum over \(t\ge 1\) yields
\[
    \tau_Z(k)
    \le
    C_{\beta\tau}D_{\mathcal Z}\,\beta_Z(k).
\]
Using the assumed exponential \(\beta\)-mixing bound,
\[
    \beta_Z(k)\le C_\beta e^{-c_\beta k},
\]
we obtain
\[
    \tau_Z(k)
    \le
    C_{\beta\tau}D_{\mathcal Z} C_\beta e^{-c_\beta k},
    \qquad k\ge 1 .
\]
This proves the claim.
\end{proof}

\subsection{Dependence structure of minibatches}
We now analyze the dependence structure of minibatches under different sampling schemes. From Lemma~\ref{lemma:tau_mixing_from_beta_mixing} we know that if the underlying sequence is exponentially \(\beta\)-mixing, then it is also exponentially \(\tau\)-mixing with new constants depending on the bounded metric comparison.

Therefore, to streamline the presentation, we proceed under the assumption that the underlying process is exponentially \(\tau\)-mixing with 
\begin{align*}
    \tau_X(k) \le c_0 \exp(-c_1 k), \quad k\ge 1.
\end{align*}
In the case where the underlying trajectory is exponentially \(\beta\)-mixing as in Lemma~\ref{lemma:tau_mixing_from_beta_mixing}, the above holds with
\begin{align*}
    c_0 = C_{\beta \tau} D_\mc{Z} C_\beta, \qquad c_1 = c_\beta.
\end{align*}

\subsubsection{Uniform sampling}

We first consider the standard sampling scheme used in DQN: uniform sampling with replacement.

At training time, we sample a minibatch from the replay buffer. Let
\begin{align*}
    I_1,\dots,I_n \overset{\mathrm{i.i.d.}}{\sim} \mathrm{Unif}\{1,\dots,M\},
\end{align*}
independently of \((Z_1,\dots,Z_M)\), and let
\begin{align*}
    I_{(1)} \le \cdots \le I_{(n)}
\end{align*}
be the corresponding order statistics. For any realized ordered index set
\begin{align*}
    1 \le i_{(1)} \le \cdots \le i_{(n)} \le M,
\end{align*}
we define the ordered minibatch by
\begin{align*}
    \mc Y = (Y_1,\dots,Y_n),
    \qquad
    Y_r := Z_{i_{(r)}},
    \quad r=1,\dots,n.
\end{align*}
With
\begin{align*}
    \mc Y_r := \sigma(Y_1,\dots,Y_r),
\end{align*}
we define, for \(1 \le k \le n-1\),
\begin{align*}
    \tau_{\mc Y}^n(k)
    :=
    \sup_{1 \le r \le n-k}
    \tau(\mc Y_r, Y_{r+k}).
\end{align*}

The first result explains why sampling with replacement is problematic for \(\tau\)-mixing: repeated indices occur with positive probability, so the lag in minibatch order need not correspond to any positive lag in the replay buffer.

\begin{proposition}[Within-minibatch dependence bound for a realized sampled minibatch]
\label{prop:minibatch_tau_bound} 
With the sampling introduced above
define
\begin{align*}
    \mathcal Y^{(i)} = (Y_1^{(i)},\dots,Y_n^{(i)}),
    \qquad
    Y_r^{(i)} := Z_{i_{(r)}},
    \quad r=1,\dots,n.
\end{align*}
Write
\begin{align*}
    \mathcal Y_r^{(i)} := \sigma(Y_1^{(i)},\dots,Y_r^{(i)}).
\end{align*}
and define the lag-zero truncated coefficient
\begin{align*}
    \bar\tau_{\mathcal M}^M(k)
    :=
    \sup_{1 \le t \le M-k}
    \tau(\sigma(Z_1,\dots,Z_t), Z_{t+k}),
    \qquad 0 \le k \le M-1.
\end{align*}
Then, for every \(1 \le r < s \le n\),
\begin{align*}
    \tau(\mathcal Y_r^{(i)}, Y_s^{(i)})
    \le
    \bar\tau_{\mathcal M}^M\bigl(i_{(s)} - i_{(r)}\bigr).
\end{align*}
Consequently, for \(1 \le \ell \le n-1\),
\begin{align*}
    \tau_{\mathcal Y^{(i)}}^n(\ell)
    \le
    \sup_{1 \le r \le n-\ell}
    \bar\tau_{\mathcal M}^M\bigl(i_{(r+\ell)} - i_{(r)}\bigr).
\end{align*}
If, in addition,
\begin{align*}
    \tau_{\mathcal M}^M(k) \le c_0 e^{-c_1 k},
    \qquad 1 \le k \le M-1,
\end{align*}
then
\begin{align*}
    \tau_{\mathcal Y^{(i)}}^n(\ell)
    \le
    \max\Biggl\{
        \bar\tau_{\mathcal M}^M(0),
        \;
        c_0
        \sup_{1 \le r \le n-\ell}
        \exp\bigl(-c_1(i_{(r+\ell)} - i_{(r)})\bigr)
    \Biggr\}.
\end{align*}
\end{proposition}

\begin{proof}
Fix \(1 \le r < s \le n\). By construction,
\begin{align*}
    \mathcal Y_r^{(i)}
    =
    \sigma(Z_{i_{(1)}},\dots,Z_{i_{(r)}}).
\end{align*}
Since
\begin{align*}
    i_{(1)} \le \cdots \le i_{(r)},
\end{align*}
we have
\begin{align*}
    \mathcal Y_r^{(i)}
    \subseteq
    \sigma(Z_1,\dots,Z_{i_{(r)}}).
\end{align*}
Therefore, by Lemma~\ref{lemma:zeta_monotone_sigma_algebra},
\begin{align*}
    \tau(\mathcal Y_r^{(i)}, Y_s^{(i)})
    &=
    \tau(\mathcal Y_r^{(i)}, Z_{i_{(s)}})
    \\
    &\le
    \tau\bigl(\sigma(Z_1,\dots,Z_{i_{(r)}}), Z_{i_{(s)}}\bigr)
    \\
    &\le
    \bar\tau_{\mathcal M}^M\bigl(i_{(s)} - i_{(r)}\bigr).
\end{align*}
This proves the first claim.

Taking the supremum over \(1 \le r \le n-\ell\) yields
\begin{align*}
    \tau_{\mathcal Y^{(i)}}^n(\ell)
    \le
    \sup_{1 \le r \le n-\ell}
    \bar\tau_{\mathcal M}^M\bigl(i_{(r+\ell)} - i_{(r)}\bigr).
\end{align*}
Finally, if
\begin{align*}
    \tau_{\mathcal M}^M(k) \le c_0 e^{-c_1 k},
    \qquad 1 \le k \le M-1,
\end{align*}
then for each \(r\),
\begin{align*}
    \bar\tau_{\mathcal M}^M\bigl(i_{(r+\ell)} - i_{(r)}\bigr)
    \le
    \max\Bigl\{
        \bar\tau_{\mathcal M}^M(0),
        \;
        c_0 e^{-c_1 (i_{(r+\ell)} - i_{(r)})}
    \Bigr\},
\end{align*}
which gives the stated bound.
\end{proof}

This shows that uniform sampling with replacement, which is standard in replay-based RL algorithms, is problematic since it does not guarantee exponentially mixing minibatches. In the next corollary, we show that this problem can be overcome by sampling uniformly \textbf{without} replacement.

\begin{remark}
\label{remark:uniform_sampling_without_replacement}
    If the sampling is \textbf{without} replacement, then the conditions of Proposition~\ref{prop:mixing_underlying_process_induces_mixing_minibatches} are satisfied. Specifically, the indices satisfy \(I_1 < I_2 < \dots < I_n\), and hence Proposition~\ref{prop:mixing_underlying_process_induces_mixing_minibatches} implies that the sampled minibatches are exponentially \(\tau\)-mixing.
\end{remark}

\subsection{Contiguous block sampler}
\label{subsec:contiguous_block_sampling}

We introduce a new sampler inspired by the blocking ideas of \citet{yu_rates} and show in the following proposition and corollary that if the underlying sequence is exponentially \(\tau\)-mixing, this sampler is guaranteed to produce exponentially \(\tau\)-mixing minibatches.

Given the buffer indices \(\{1, \dots, M\}\), this sampler constructs a number of contiguous blocks of length \(b\), each separated by a gap length \(a\), and then concatenates the blocks to form the minibatch.

\begin{proposition}[Contiguous block sampling]
\label{thm:contiguous_block_sampling}
    Let \((X_t)_{t \ge 1}\) be an exponentially \(\tau\)-mixing sequence with
    \begin{align*}
        \tau_X(k) \le c_0 \exp(-c_1 k), \qquad \forall k \ge 1,
    \end{align*}
    where
    \begin{align*}
        \tau_X(k) := \sup_{t \ge 1} \tau(\mc A_t, X_{t+k}),
        \qquad
        \mc A_t := \sigma(X_s : s \le t).
    \end{align*}
    Fix a starting point \(t_0 \ge 1\), sample size \(n \ge 3\), block length \(1 \le b \le n\), and gap length \(0 \le a \le b-2\). Write
    \begin{align*}
        q := \lfloor n/b \rfloor,
        \qquad
        0 \le r < b,
    \end{align*}
    to denote the number of full blocks of length \(b\). Then \(n = qb + r\).
    For \(1 \le i \le q\), define the starting index of the \(i\)-th full block by
    \begin{align*}
        s_i := t_0 + (i-1)(a+b),
    \end{align*}
    and define the \(i\)-th full block by
    \begin{align*}
        B_i := (X_{s_i}, X_{s_i+1}, \dots, X_{s_i+b-1}).
    \end{align*}
    If \(r \ge 1\), define the starting index of the remainder block by
    \begin{align*}
        s_{q+1} := t_0 + q(a+b),
    \end{align*}
    and define the remainder block by
    \begin{align*}
        B_{q+1} := (X_{s_{q+1}}, X_{s_{q+1}+1}, \dots, X_{s_{q+1}+r-1}).
    \end{align*}
    Let \(\mc Y = (Y_1,\dots,Y_n)\) be the concatenation of the blocks
    \begin{align*}
        B_1,B_2,\dots,B_q
    \end{align*}
    and, if \(r \ge 1\), the remainder block \(B_{q+1}\). Then
    \begin{align*}
        \tau_\mc Y^n(k) \le \tau_X(k) \le c_0 \exp(-c_1 k),
        \qquad 1 \le k \le n-1.
    \end{align*}
\end{proposition}

\begin{proof}
For \(1 \le j \le n\), condition on the sampled indices and define
\begin{align*}
    \mc Y_j := \sigma(Y_1,\dots,Y_j).
\end{align*}
Let
\begin{align*}
    u_1 < u_2 < \cdots < u_n
\end{align*}
denote the ordered time indices corresponding to the concatenated sample \(\mc Y\), so that
\begin{align*}
    Y_j = X_{u_j},
    \qquad 1 \le j \le n.
\end{align*}
Since the blocks are contiguous and are concatenated in temporal order, the sequence \(u_1,\dots,u_n\) is strictly increasing. Moreover, consecutive sampled indices satisfy
\begin{align*}
    u_{j+1} - u_j =
    \begin{cases}
        1, & \text{if } Y_j \text{ and } Y_{j+1} \text{ belong to the same block},\\
        a+1, & \text{if } Y_j \text{ and } Y_{j+1} \text{ belong to consecutive blocks}.
    \end{cases}
\end{align*}
In particular,
\begin{align*}
    u_{j+1} - u_j \ge 1,
    \qquad 1 \le j \le n-1.
\end{align*}
Hence, for every \(1 \le j \le n-k\),
\begin{align*}
    u_{j+k} - u_j \ge k.
\end{align*}
This shows that contiguous block sampling preserves the order of the underlying process and guarantees that minibatch lag \(k\) corresponds to at least \(k\) time steps in the original sequence. Proposition~\ref{prop:mixing_underlying_process_induces_mixing_minibatches} then implies that the sampled minibatches are exponentially \(\tau\)-mixing.

\end{proof}

The following corollary shows that this also holds when sampling from a finite contiguous subsequence of the original process, which corresponds to sampling from a replay buffer.
\begin{corollary}[Contiguous block sampling from finite subsequence]
\label{cor:contiguous_block_sampling_buffer}
Consider the same setting as in Proposition~\ref{thm:contiguous_block_sampling}, except suppose now that the minibatches are drawn from the finite replay buffer \(\mc M = (Z_1,\dots,Z_M)\) as in~\eqref{eqn:buffer}. Write
\begin{align*}
    q := \lfloor n/b \rfloor,
    \qquad
    0 \le r < b
\end{align*}
to denote the number of full blocks of length \(b\). Then \(n = qb + r\).
For \(1 \le i \le q\), define the starting index of the \(i\)-th full block in the replay buffer by
\begin{align*}
    s_i := 1 + (i-1)(a+b),
\end{align*}
and define the \(i\)-th full block by
\begin{align*}
    B_i := (Z_{s_i}, Z_{s_i+1}, \dots, Z_{s_i+b-1}).
\end{align*}
Assume that all selected block indices lie inside the replay buffer. Equivalently, assume
\begin{align*}
    (q-1)(a+b) + b \le M
\end{align*}
when \(r=0\), and
\begin{align*}
    q(a+b) + r \le M
\end{align*}
when \(r \ge 1\).

As in Proposition~\ref{thm:contiguous_block_sampling} let \(\mc Y = (Y_1,\dots,Y_n)\) be the concatenation of the blocks.
and, if \(r \ge 1\), the remainder block \(B_{q+1}\). Then
\begin{align*}
    \tau_{\mc Y}^n(k)
    \le
    c_0 \exp(-c_1 k),
    \qquad 1 \le k \le n-1.
\end{align*}
\end{corollary}

\begin{proof}
The proof proceeds exactly as in the proof of Proposition~\ref{thm:contiguous_block_sampling} to show that the sampler is \textit{order-preserving} and \textit{duplicate-free}. Therefore, Proposition~\ref{prop:mixing_underlying_process_induces_mixing_minibatches} implies that the minibatches are exponentially \(\tau\)-mixing and the coefficients satisfy
\begin{align*}
    \tau_{\mc Y}^n(k)
    \le
    c_0 e^{-c_1 k},
    \qquad 1 \le k \le n-1.
\end{align*}
\end{proof}

\begin{remark}[Circular wraparound]
The same conclusion remains valid for a circular wraparound version of the sampler, in which blocks that extend past the end of the replay buffer continue from the beginning of the buffer, provided that the sampled buffer indices are all distinct and the resulting sampled observations are reordered in increasing buffer index order before the truncated mixing coefficients are evaluated. Under this ordering, the same proof applies, since the sampled indices \(u_1 < \cdots < u_n\) again satisfy
\begin{align*}
    \mc Y_j \subseteq \sigma(Z_1,\dots,Z_{u_j})
    \qquad\text{and}\qquad
    u_{j+k} - u_j \ge k.
\end{align*}
Therefore, the same exponential bound
\begin{align*}
    \tau_{\mc Y}^n(k) \le c_0 e^{-c_1 k},
    \qquad 1 \le k \le n-1,
\end{align*}
continues to hold.
\end{remark}

\begin{remark}[Decay control] 
    Our contiguous block sampler provides some control over the decay in dependence of the subsampled sequence; larger gap lengths \(a\) and smaller block sizes \(b\) results in faster decay.
\end{remark}

\section{Experimental Details.}
\label{app:experimental_setting}

We study temporal dependence in two objects produced by DQN training: (i) state--action trajectories generated by a fixed trained policy, and (ii) minibatches sampled from the replay buffer during the course of the DQN algorithm. In both cases, our goal is not to estimate the exact population \(\tau\)-mixing coefficient directly, but rather a finite-sample, finite-memory proxy that is computable from observed samples and is closely aligned with the definition of \(\tau\)-dependence through the \(1\)-Wasserstein distance. Here, we describe the estimator used in our experiments, and provide details about reproducibility.

\subsection{Estimator details}
We describe the construction of our estimator used in the experiments.

\subsubsection{Relationship to \(1\)-Wasserstein distance}

By the Kantorovich-Rubinstein theorem, the \(1\)-Wasserstein distance admits the dual representation (see e.g., Theorem~2 of \citet{kuhn_wasserstein_duality})
\begin{align*}
    W_1(\mu, \nu) = \sup \left\{\int_{\mbb{R}^m}\phi(\xi) \mu(d\xi) - \int_{\mbb{R}^m} \phi(\xi')\nu(d\xi')\right\},
\end{align*}
where the supremum is taken over all \(1\)-Lipschitz functions. Comparing this to the definition of the \(\tau\)-mixing coefficient in ~\eqref{eqn:tau_general_setup} we immediately see that
\begin{align*}
    \tau(\mc A, X) = \norm{W_1(\mu_{X | \mc A}, \mu_X)}_1 = \mbb{E}\left[W_1(\mu_{X | \mc A}, \mu_X)\right].
\end{align*}
That is, the \(\tau\)-mixing coefficient \(\tau(\mc A, X)\) is precisely the expected \(1\)-Wasserstein distance between the marginal law \(\mu_X\) and the conditional law \(\mu_{X | \mc A}\). 

In this way, we can also relate the lag-\(k\) \(\tau\)-mixing coefficient to the expected \(1\)-Wasserstein distance. Let \((X_t)_{t \ge 1}\) be a \(\mbb{R}^d\)-valued stochastic process and let 
\begin{align*}
    \mathcal{F}_t = \sigma(X_1,\dots,X_t)
\end{align*}
be the \(\sigma\)-algebra generated by the past up to time \(t\). Then we have that
\begin{align}
\label{eqn:tau_k_wasserstein}
    \tau_X(k) = \underset{t \ge 1}{\sup} \;\mbb{E}\left[W_1\left(\mc{L}(X_{t+k} \mid \mc F_t), \mc{L}(X_{t+k})\right)\right],
\end{align}
where \(\mc{L}(X_{t+k})\) denotes the marginal law of \(X\) and \(\mc{L}(X_{t+k}\mid \mc{F}_t)\) denotes a conditional law given \(\mc F_t\).

\subsubsection{A finite-window \(\tau\)-mixing proxy}
\label{subsec:tau-proxy}

Consider the relationship in ~\eqref{eqn:tau_k_wasserstein}. 
Since conditioning on the full past is computationally infeasible in our setting, our estimation procedure replaces \(\mc F_t\) by a finite history window of length \(m\). This yields the truncated target
\begin{equation}
\label{eqn:tau_truncated_window}
    \tau_{X}^m(k) := \mathbb{E}\!\left[ W_1\!\left(
    \mathcal{L}(X_{t+m-1+k}\mid \sigma(X_t,\dots,X_{t+m-1})),
    \mathcal{L}(X_{t+m-1+k})
    \right)
    \right].
\end{equation}
Notice that, since 
\begin{align*}
    \sigma(X_t, \dots, X_{t+m-1}) \subseteq \sigma(X_1, \dots, X_{t+m-1}),\qquad \forall \,t \ge 1,
\end{align*}
it follows from Lemma~\ref{lemma:zeta_monotone_sigma_algebra} that, for any \(m \ge 0\)
\begin{align}
    \tau_X^m(k) \le \tau_X(k), \quad k \ge 1.
\end{align}
The estimator in our implementation should therefore be interpreted as an estimator for \(\tau_X^m(k)\), rather than as a direct estimator of the full \(\tau_X(k)\). 

\subsubsection{Construction of the estimator}
\label{subsec:estimator}

Suppose we observe a sample
\begin{align*}
    x_{1:T}=(x_1,\dots,x_T), \qquad x_t\in\mathbb{R}^d.
\end{align*}
Fix a lag \(k\geq 1\) and history length \(m\geq 1\). Define
\begin{align*}
    N_k := T-m-k+1.
\end{align*}
If \(N_k\leq 0\), then there are not enough observations to form a length-\(m\) conditioning window and a \(k\)-step-ahead target, so no estimate is produced for that lag; in our implementation, we set the estimate to zero in this case.

For each \(i=1,\dots,N_k\), define the history vector
\begin{equation}
\label{eqn:history_vector}
    Z_i:= \begin{bmatrix}
        x_i & x_{i+1} & \dots & x_{i+m-1}
    \end{bmatrix}
    \in \mathbb{R}^{md},
\end{equation}
and the corresponding future target
\begin{equation}
\label{eqn:future_target}
    Y_i := x_{i+m-1+k} \in \mathbb{R}^{d}.
\end{equation}
Thus, \((Z_i,Y_i)\) is the sample used to approximate the map
\begin{align*}
    z \mapsto \mathcal{L}(X_{t+m-1+k}\mid X_t,\dots,X_{t+m-1}=z).
\end{align*}

\paragraph{Global empirical law.}
The unconditional law of the future variable is approximated by the empirical measure
\begin{equation}
\label{eqn:global_law}
    \widehat{\mc{L}}_k(\cdot) := \frac{1}{N_k}\sum_{i=1}^{N_k} \delta_{Y_i}.
\end{equation}

\paragraph{Local conditional empirical law via \(k\)-nearest neighbors.}
To estimate the conditional law given a history \(Z_i\), we use a \(k\)-nearest-neighbor procedure in the space of history windows. Let \(r\) denote the number of nearest neighbors used in the estimation. Using Euclidean distance on \(\mathbb{R}^{md}\), let
\begin{align*}
    \mathcal{N}_i \subset \{1,\dots,N_k\}
\end{align*}
be the set of indices corresponding to the \(r\) nearest neighbors of \(Z_i\). 

The resulting local empirical measure is
\begin{equation}
    \label{eqn:local_conditional_law}
    \widehat{\mc L}_k(\cdot | Z_i) := \frac{1}{r}\sum_{j \in \mc{N}_i} \delta_{Y_j}
\end{equation}

\paragraph{Optimal transport discrepancy.}
For each \(i\), the estimator compares \(\widehat{\mc{L}}_k(\cdot)\) and \(\widehat{\mc{L}}_k(\cdot \mid Z_i)\) using the \(1\)-Wasserstein distance with Euclidean ground metric on the target space:
\begin{equation}
\label{eqn:local_wasserstein}
    \widehat{W}_{i,k} :=
    W_1 \left(\widehat{\mc{L}}_k(\cdot \mid Z_i),\widehat{\mc{L}}_k(\cdot) \right).
\end{equation}
Because both measures are discrete and uniformly weighted, this can be written as the optimal transport problem
\begin{equation}
\label{eqn:emd}
    \widehat{W}_{i,k}
    =
    \min_{\pi \in \Pi(a_r,a_{N_k})}
    \sum_{\alpha=1}^{r}\sum_{\beta=1}^{N_k}
    \pi_{\alpha\beta}
    \,
    \bigl\|
    Y_{j_\alpha}-Y_\beta
    \bigr\|_2,
\end{equation}
where \(\mathcal{N}_i=\{j_1,\dots,j_r\}\),
\begin{align*}
    a_r=\left(\frac1r,\dots,\frac1r\right)\in\mathbb{R}^r,
    \qquad a_{N_k}=\left(\frac1{N_k},\dots,\frac1{N_k}\right)\in\mathbb{R}^{N_k},
\end{align*}
and \(\Pi(a_r,a_{N_k})\) denotes the set of nonnegative couplings with these marginals. In the default configuration of our experiments, this quantity is computed using Earth Mover's Distance via \texttt{ot.emd2} from the python POT library \citep{pot_paper, pot_library}. 

\paragraph{Observed dependence score.}

Finally, we replace the supremum over \(t\) with a finite-window maximum to obtain our final estimate
\begin{equation}
\label{eqn:tau_obs}
    \widehat{\tau}_{X, \mathrm{obs}}^m(k)
    :=
    \max_{1 \le i \le N_k}
    \widehat{W}_{i,k}.
\end{equation}
This is the raw empirical dependence score at lag \(k\).

\subsubsection{Permutation centering}
\label{subsec:perm-centering}

Even under weak or no dependence, the quantity in \eqref{eqn:tau_obs} need not be exactly zero in finite samples, because it compares a local empirical distribution based on \(r\) points to a global empirical distribution based on \(N_k\) points. To reduce this finite-sample bias, we subtract a permutation baseline.

For each permutation replicate \(b=1,\dots,B\), draw a random permutation
\begin{align*}
    \sigma_b \in \mathfrak{S}_{N_k}
\end{align*}
and define permuted targets
\begin{align*}
    Y^{(b)}_j := Y_{\sigma_b(j)}, \qquad j=1,\dots,N_k.
\end{align*}
The neighbor sets are \emph{not} recomputed after permutation; instead, the same neighborhoods determined by the history vectors \((Z_i)\) are reused. 

Let \(\widetilde{\mathcal{N}}_i\) denote the index set used in the permutation step. Then the permuted local law is
\begin{equation}
\label{eqn:perm_laws}
    \widehat{\mc{L}}_k^{(b)}(\cdot \mid Z_i):=
    \frac{1}{|\widetilde{\mathcal{N}}_i|}
    \sum_{j\in \widetilde{\mathcal{N}}_i}
    \delta_{Y^{(b)}_j}, \qquad \widehat{\mc{L}}^{(b)}_k(\cdot) := \frac{1}{N_k}\sum_{j=1}^{N_k}\delta_{Y^{(b)}_j}.
\end{equation}
The corresponding transport discrepancy is
\begin{equation}
\label{eqn:perm_local}
    \widehat{W}^{(b)}_{i,k} := W_1\!\left(\widehat{\mc{L}}^{(b)}_{k}(\cdot \mid Z_i),\widehat{\mc{L}}^{(b)}_k(\cdot) \right),
\end{equation}
and the permutation baseline is
\begin{equation}
\label{eqn:perm_baseline}
    \widehat{\tau}_{X,\mathrm{perm}}^m(k)
    :=
    \frac{1}{B}
    \sum_{b=1}^{B}
    \left(
    \max_{1 \le i \le N_k}
    \widehat{W}^{(b)}_{i,k}
    \right).
\end{equation}

The final estimator used in our experiments is
\begin{equation}
\label{eqn:final_estimator}
    \widehat{\tau}^m_X(k) := \max\!\left\{ \widehat{\tau}_{X, \mathrm{obs}}^m(k) -\widehat{\tau}_{X, \mathrm{perm}}^m(k), \, 0 \right\}, \quad 1 \le k \le K,
\end{equation}
where \(K\) is the maximum lag to test for, and is specified by the user.
Thus, \(\widehat{\tau}(k)\) is a nonnegative, bias-corrected dependence score: it is large when nearby histories in \(Z\)-space induce future targets whose empirical law differs substantially from the unconditional law, and it is close to zero when this effect disappears after permutation centering.

Finally, we remark that when
\begin{align*}
    N_k - r \le 0,
\end{align*}
then all the \(N_k\) points are included in the \(r\) nearest neighbors, and hence
\begin{align*}
    \widehat{\mc L}_k(\cdot | Z_i) = \widehat{\mc L}_k(\cdot),
\end{align*}
and consequently,
\begin{align}
\label{eqn:tau_obs_cutoff}
    \widehat{\tau}_{X, \mathrm{obs}}^m(k) = 0, \qquad k \ge T - m + 1 - r.
\end{align}
In our experiments, we use the scikit-learn library \citep{sklearn, sklearn_api} to implement the KNN estimation with the parameter \(\texttt{leave\_one\_out=True}\). This means that the point itself is not considered when determining its neighboring set, resulting in one fewer available observations for estimation. In this case, we have that
\begin{align}
\label{eqn:tau_obs_cutoff_loo}
    \widehat{\tau}_{X, \mathrm{obs}}^m(k) = 0, \qquad k \ge T - m - r.
\end{align}

\subsubsection{Experimental protocol and objects estimated}
\label{subsec:objects-estimated}

We apply the estimator in~\eqref{eqn:final_estimator} to two types of sequences produced by DQN training: trajectories generated by a fixed trained policy, and minibatches sampled from the replay buffer during optimization.

\paragraph{Environments and DQN architecture.}
We consider three Gymnasium environments \citep{gymnasium}: \texttt{CartPole-v1}, \texttt{FrozenLake-v1}, and \texttt{LunarLander-v3}. In each case, the DQN agent uses a fully connected ReLU network with two hidden layers of width \(256\). The network maps states to action values and has architecture
\[
    \mathbf d =
    \begin{bmatrix}
    \mathrm{dim}(\mathcal S) & 256 & 256 & |\mathcal A|
    \end{bmatrix}.
\]
Thus, one forward pass produces \(Q\)-values for all actions. This is the standard DQN parameterization and avoids the alternative state--action architecture
\[
    \mathbf d' =
    \begin{bmatrix}
    \mathrm{dim}(\mathcal S)\times |\mathcal A| & 256 & 256 & 1
    \end{bmatrix},
\]
which would require one forward pass per action.

\paragraph{State--action representation.}
Each observation used by the dependence estimator is a state--action vector. Continuous states are represented by their raw observation vectors, while discrete states are represented by one-hot vectors. The selected action is appended as an additional coordinate. Thus, the sequence analyzed by the estimator has the form
\[
    X_t =
    \begin{bmatrix}
        S_t \\ A_t
    \end{bmatrix}.
\]
The same representation is used for trajectory-level and minibatch-level estimates.

\paragraph{Trajectory dependence.}
To estimate dependence in the underlying RL process, we first train a DQN agent in the given environment. After training, we fix the learned policy and generate rollouts using the greedy action-selection rule, with no exploration. For each rollout, we collect the resulting state--action sequence
\[
    (X_t)_{t=1}^T,
    \qquad
    X_t =
    \begin{bmatrix}
        S_t \\ A_t
    \end{bmatrix},
\]
and apply the estimator in~\eqref{eqn:final_estimator} to obtain
\[
    \widehat{\tau}(k), \qquad 1\leq k\leq K.
\]
For trajectory-level estimates, the lag \(k\) is ordinary environment-time lag along the rollout generated by the fixed greedy policy.

\paragraph{Minibatch dependence.}
To estimate dependence in the data used by the DQN regression step, we log every minibatch sampled from the replay buffer during training, before the corresponding optimizer step is performed. For a minibatch of size \(B_{\mathrm{mb}}\), the logged matrix is
\[
    \mathcal B_u =
    \begin{bmatrix}
        X_1 \\
        X_2 \\
        \vdots \\
        X_{B_{\mathrm{mb}}}
    \end{bmatrix}
    \in \mathbb R^{B_{\mathrm{mb}}\times(d+1)},
\]
where \(u\) indexes the gradient update. The ordering of the rows is the order in which samples are returned by the sampler. We then apply the estimator in~\eqref{eqn:final_estimator} separately to each logged minibatch.

For minibatch-level estimates, the lag \(k\) is therefore lag in the sampled minibatch order. It is not environment-time lag, and it is not lag across gradient updates. These estimates measure the residual dependence among the samples entering a single DQN update.

\subsubsection{Contiguous block replay sampler}
\label{subsec:block-sampler}

The minibatches used during training are sampled using a contiguous block sampler inspired by the blocking technique of \citet{yu_rates}. Rather than drawing all samples independently from the replay buffer, the sampler draws several contiguous blocks of replay-buffer indices. Each block has length \(b\), and successive blocks are separated by a gap of length \(a\). The sampled blocks are then concatenated to form the minibatch.

\subsubsection{Aggregation and visualization}
\label{subsec:aggregation}

Both trajectory-level and minibatch-level experiments produce a collection of estimated curves
\begin{align*}
    \widehat{\tau}^{(1)}(k), \dots, \widehat{\tau}^{(M)}(k),
    \qquad k=1,\dots,K,
\end{align*}
where \(M\) is the number of rollouts or the number of logged minibatches. The plotting script averages these estimates across replications and reports bands equal to \(\pm 2\) standard errors:
\[
\overline{\tau}(k)
=
\frac{1}{M}\sum_{m=1}^M \widehat{\tau}^{(m)}(k),
\qquad
\mathrm{SE}(k)
=
\frac{\mathrm{sd}(\widehat{\tau}^{(1)}(k),\dots,\widehat{\tau}^{(M)}(k))}{\sqrt{M}}.
\]
The randomness in the estimates come from several sources: (i) randomness in the underlying environment, (ii) empirical estimation of conditional and unconditional expectations, and (iii) possible randomness induced by initializations.
The figures therefore summarize how rapidly the empirical dependence score decays with lag for each experimental condition.

\subsubsection{Summary}
In summary, our estimator replaces the inaccessible conditional law in the definition of \(\tau\)-mixing by a \(k\)-nearest-neighbor empirical conditional distribution over length-\(m\) histories, measures its discrepancy from the unconditional empirical law using optimal transport, and subtracts a permutation baseline to reduce finite-sample bias. The resulting statistic is a practical, Wasserstein-based dependence proxy that can be applied uniformly to RL trajectories and replay-buffer minibatches.

\subsection{Reproducibility}
\label{subsec:reproducibility}

\subsubsection{Code to experiments}

All the code together with the data used for conducting our experiments, as well as instructions on how to use it can be found at:
\begin{center}
    \url{https://github.com/leonhalgryn/DQN-mixing-times}
\end{center}

\subsubsection{Compute Resources}
\label{subsubsec:compute_resources}
All experiments were run on a Linux server equipped with two \textit{Intel Xeon Gold 5217} processors (\(16\) physical cores, \(32\) threads total; \(3.0\) GHz base frequency) and \(62\) GB of system memory. We use only CPU for computation with optional parallelization over cores.

We note that, each experiment require significant runtime. On our linux server, this amounts to several days for completing the full experiments pipeline for a single environment; i.e., (i) training the DQN agent and logging the minibatches, (ii) estimating the mixing coefficients for these minibatches, and (iii) simulating trajectories in the environment using the trained policy and estimating the mixing coefficients of these trajectories.

We note that, due to the extensive runtime of the experiments, we were not able to conduct ablation studies, or repeat the experiments over multiple runs. I.e., the experiments shown were for a \textbf{single run}. 

\subsubsection{Parameters and hyperparameters}
We now outline the parameter and hyperparameters settings used in our experiments. Table~\ref{tab:dqn_parameter_descriptions} described the DQN parameters in our implementation, and Table~\ref{tab:dqn_params} records the hyperparmeter setting used in the respective environments. The hyperparameters used in the DQN algorithm were based on those found in open-source implementations online.
We use the same default estimator parameters given in Table~\ref{tab:default_experimental_configurations} for all experiments, and DQN hyperparameters detailed in Table~\ref{tab:dqn_params} for training the DQN algorithm in the respective environments. Table~\ref{tab:default_experimental_configurations} records the parameters used in our estimator \(\widehat{\tau}_X^m(k)\) described previously. Finally, we use the contiguous block sampler in all our experiments with a block and gap length (respectively) of \(a = 2\) and \(b = 42\). We remark that we did not conduct any hyperparameter optimization or experiment with different choices of (hyper-)parameters.

\begin{table}[hb]
\centering
\caption{DQN training parameters used in the experiments.}
\label{tab:dqn_parameter_descriptions}
\resizebox{\columnwidth}{!}{%
\begin{tabular}{lll}
\toprule
Parameter & Meaning & Brief description \\
\midrule
\(\gamma\) & Discount factor & Controls the weight placed on future rewards. \\
Buffer size & Replay-buffer capacity & Maximum number of transitions stored in memory. \\
Max episodes & Training horizon & Number of environment episodes used for training. \\
Minibatch size & Batch size & Number of replay samples used in each gradient update. \\
Initial \(\varepsilon\) & Initial exploration rate & Starting probability of taking a random action. \\
\(\varepsilon_{\min}\) & Minimum exploration rate & Final lower bound for \(\varepsilon\)-greedy exploration. \\
Exploration fraction & Annealing duration & Fraction of training over which \(\varepsilon\) is linearly decayed. \\
Learning rate & Optimizer step size & Step size used by Adam for updating the \(Q\)-network. \\
\(\tau_{{target}}\) & Target-update coefficient & Interpolation parameter soft updates to the target network. \\
Target period (\(T_\mathrm{target}\))& Target-update frequency & Number of environment steps between target-network updates. \\
Start time & Replay warm-up size & Minimum number of stored transitions before training begins. \\
Update frequency & Training frequency & Number of environment steps between optimization phases. \\
Number of updates & Updates per phase & Number of gradient updates performed at each optimization phase. \\
\bottomrule
\end{tabular}%
}
\end{table}

\begin{table}[ht]
\centering
\caption{DQN hyperparameter values used in the main experiments. All experiments use
minibatch size \(128\), discount factor \(\gamma=0.99\), learning rate \(10^{-3}\),
minimum exploration probability \(\varepsilon_{\min}=0.05\), and hidden-layer
widths \((256,256)\). The target-update coefficient is
\(\tau_{\mathrm{target}}=1.0\), corresponding to a hard target-network update
whenever the update is triggered.}
\label{tab:dqn_params}
\resizebox{\columnwidth}{!}{%
\begin{tabular}{lcccccc}
\toprule
Environment & Buffer size & Episodes & Initial \(\varepsilon\) & Expl. frac. & Update freq. & Target period \\
\midrule
\texttt{CartPole-v1}    & \(10^5\)        & \(1000\) & \(0.99\) & \(0.2\) & \(256\) & \(10\) \\
\texttt{FrozenLake-v1}  & \(5\cdot 10^4\) & \(5000\) & \(0.999\) & \(0.3\) & \(128\) & \(1000\) \\
\texttt{LunarLander-v3} & \(10^6\)        & \(2000\) & \(1.0\) & \(0.5\) & \(4\) & \(1000\) \\
\bottomrule
\end{tabular}%
}
\end{table}

\begin{table}[ht]
  \caption{Default configurations used in experiments.}
  \label{tab:default_experimental_configurations}
  \centering
  \begin{tabular}{lll}
    \toprule
        Variable & Description & Value \\
        \hline
        m &  \text{Conditioning depth} & 1 \\
        r & \text{Number of nearest neighbors} & 20\\
        B & \text{Number of permutations} & 1\\
        M & \text{Number of rollouts} & 500\\
    \bottomrule
  \end{tabular}
\end{table}

\section{Additional Theoretical Details.}
\label{app:more_theoretical_details}

\subsection{Ghost sample.}
\label{app:ghost_sample}
We construct the ghost sample \(\{\tilde{X}_i\}_{i \in [n]}\) on an enlarged probability space, such that:
\begin{enumerate}
    \item \(\tilde{X}_i \overset{d}{=} X_i\) for all \(i \in [n]\),
    \item the ghost sequence \(\{\tilde{X}_i\}_{i \in [n]}\) is independent of the training sample \(\{X_i\}_{i \in [n]}\),
    \item \(\{\tilde{X}_i\}_{i \in [n]}\) is stationary and exponentially \(\tau\)-mixing with the same coefficients as \((X_i)_{i \in [n]}\), i.e., \(\tau_{\tilde{X}}(k) = \tau_X(k)\) for all \(k \ge 1\).
\end{enumerate}

This construction allows us to write population norms as expectations over the ghost sample while preserving the dependence structure of the original process. This can be formally justified using the coupling arguments for \(\tau\)-mixing sequences introduced by \citet{dedecker_coupling_2004}.

In classical symmetrization arguments, the ghost sample is typically assumed to be i.i.d.\ and independent of the training data. However, in our setting, assuming an i.i.d.\ ghost sample would imply that it arises from a fundamentally different generative mechanism than the exponentially \(\tau\)-mixing training sample. Consequently, computing expectations with respect to such a ghost sample would not correspond to the same population object as the training sample. Instead, we construct the ghost sample to follow the same generative mechanism as the training data, preserving the \(\tau\)-mixing dependence structure. This ensures that expectations over the ghost sample remain meaningful and directly comparable to the training sample.

\subsection{Effective Sample Size.}
\label{app:effective_sample_size}
The effective sample size \(n_\mathrm{eff}\) quantifies the number of nearly independent observations present in a \(\tau\)-mixing sequence of length \(n\). We provide a general definition below, which is taken from Lemma~1 of \citet{liu_generalisation_tau_mixing_2025}. 
\begin{definition}[Effective sample size for 
\(\tau\)-mixing]
Let \((X_t)_{t \in \mbb{Z}}\) be a strictly stationary \(\tau\)-mixing process satisfying \(\norm{X_t}_2 \le B\), \(\mbb{E}[\norm{X_t}_2^2] = 0\), and \(\mbb{E}[\norm{X_t}_2^2] \le \sigma^2\), then the effective sample size is given by
\begin{align}
\label{eqn:effective_sample_size}
    n_\mathrm{eff} := \max\left\{1 \le m \le n\ \mid \tau \left(\left\lfloor \frac{n}{m}\right\rfloor \right)  \le \frac{B}{m} \vee \frac{\sigma}{\sqrt{m}} \right\}.
\end{align}
Moreover, for an exponentially \(\tau\)-mixing process with \(\tau_X(k) \le c_0 \exp(-c_1 k)\), the effective sample size satisfies
\begin{align*}
    n_\mathrm{eff} \ge \min \left\{\frac{n}{2} \cdot c_1 \cdot \left[1 \vee \log \left(\frac{c_0 \cdot c_1 \cdot n}{B}\right)\right]^{-1}, n\right\}.
\end{align*}
    
\end{definition}
Heuristically, one can partition the sequence into contiguous blocks of length \(m\) such that blocks separated by \(m\) steps are approximately independent; \(n_\mathrm{eff}\) is then roughly the number of such blocks.

This heuristic can be formalized using the block coupling construction of \citet{dedecker_coupling_2004} for \(\tau\)-mixing sequences. 
For a block length \(m\), one can construct a coupled block sequence with the same marginal distributions as the original blocks such that each block of length \(m\) is independent of the others, up to a coupling error bounded by \(\tau(m)\). 
Letting \(B = \lfloor n/m \rfloor\) denote the number of blocks, the total coupling error is at most \(B \cdot \tau(m)\). 
By the definition in~\eqref{eqn:effective_sample_size}, \(m\) is chosen so that this cumulative coupling error is negligible relative to the statistical fluctuations of interest: either of order \(1/B\) for mean-type bounds (\(\tau(m) \lesssim 1/m\)) or of order \(1/\sqrt{B}\) for variance-type bounds (\(\tau(m) \lesssim 1/\sqrt{m}\)).
Hence, the effective sample size can be interpreted as
\begin{align*}
    n_\mathrm{eff} \approx B \asymp \frac{n}{m},
\end{align*}
the number of nearly independent blocks.

Lemma~\ref{lemma:effective_sample_size_exp_decay_clean} shows that under exponential \(\tau\)-mixing, \(\tau(k) \lesssim \exp(-ck)\), the corresponding effective sample size satisfies
\begin{align*}
    n_\mathrm{eff} \asymp \frac{n}{\log n},
\end{align*}
justifying the choice of block length \(m \asymp \log n\).

\end{document}